\documentclass{article}

\usepackage[nonatbib,preprint]{neurips_2026}

\usepackage[utf8]{inputenc} % allow utf-8 input
\usepackage[T1]{fontenc}    % use 8-bit T1 fonts
\usepackage{hyperref}       % hyperlinks
\usepackage{url}            % simple URL typesetting
\usepackage{booktabs}       % professional-quality tables
\usepackage{amsfonts}       % blackboard math symbols
\usepackage{nicefrac}       % compact symbols for 1/2, etc.
\usepackage{microtype}      % microtypography
\usepackage{xcolor}         % colors

\usepackage[numbers,sort&compress]{natbib}
\usepackage{amssymb}
\usepackage{amsmath}
\usepackage{algorithm}
\usepackage{algpseudocode}
\usepackage{caption}
\usepackage{soul}
\usepackage{wrapfig}
\usepackage{makecell}
\usepackage{subcaption}
\usepackage{graphicx}
\usepackage{enumitem}
\usepackage{multirow}
\usepackage{wrapfig}
\usepackage{titlesec}

\titlespacing*{\section}
  {0pt}{*0.4}{*0.4}  % {left}{before-sep}{after-sep}
\titlespacing*{\subsection}
  {0pt}{*0.3}{*0.3}

\DeclareMathOperator*{\argmin}{arg\,min}
\newcommand{\norm}[1]{\left\lVert#1\right\rVert}

\usepackage{amsthm}
\newtheorem{theorem}{Theorem}
\newtheorem{proposition}{Proposition}
\newtheorem{corollary}{Corollary}
\newtheorem{lemma}{Lemma}
\newtheorem{remark}{Remark}

\title{HodgeCover: Higher-Order Topological Coverage Drives Compression of Sparse Mixture-of-Experts}

\author{%
  Tao Zhong$^{1}$,
  Dongzhe Zheng$^{1}$,
  Christine Allen-Blanchette$^{1}$ \\
  $^1$Princeton University\\
  \texttt{\{tzhong, ca15\}@princeton.edu}
}

\begin{document}

\maketitle

\begin{abstract}
Sparse Mixture-of-Experts (MoE) layers route tokens through a handful
of experts, and learning-free compression of these layers reduces
inference cost without retraining. A subtle obstruction blocks every
existing compressor in this family: three experts can each be
pairwise compatible yet form an irreducible cycle when merged
together, so any score that ranks experts on pairwise signals is
structurally blind to which triples are jointly mergeable. We show
the obstruction is a precise mathematical object, the harmonic
kernel of the simplicial Laplacian on a $2$-complex whose vertices
are experts, whose edges carry KL merge barriers, and whose faces
carry triplet barriers; Hodge-decomposing the edge-barrier signal
isolates the kernel exactly. We turn the diagnostic into a selection
objective: HodgeCover greedily covers the harmonic-critical edges
and triplet-critical triangles, and a hybrid variant of HodgeCover pairs it with
off-the-shelf weight pruning on survivors. On three open-weight Sparse MoE
backbones under aggressive expert reduction, HodgeCover matches
state-of-the-art learning-free baselines on the expert-reduction
axis, leads on the aggressive-compression frontier of the hybrid axis, and uniquely balances retained mass
across all four Hodge components. These results show that exposing
the harmonic kernel of a learned MoE structure changes which
compressor wins at the regime that matters most.
\end{abstract}

\begin{figure}[!h]
  \centering
  \vspace{-10pt}
  \includegraphics[width=\linewidth]{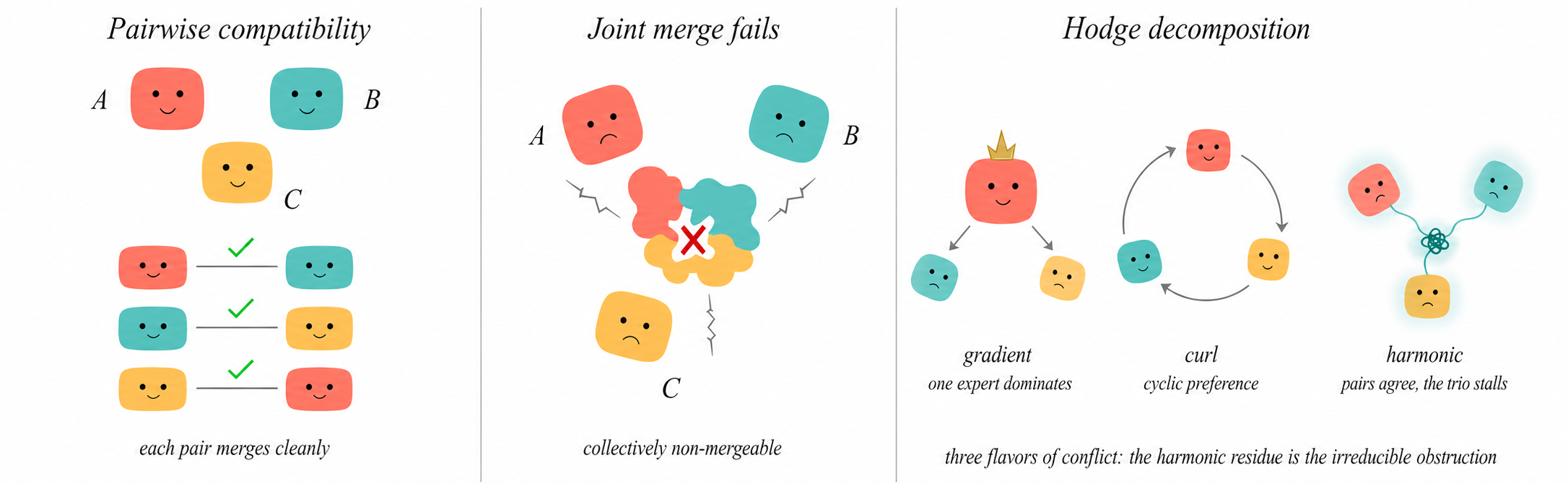}
  \caption{Three experts can be pairwise compatible yet jointly form an
    irreducible mergeability barrier that no pair of pairwise scores predicts.
    This residue is exactly the harmonic component of the Hodge decomposition
    on the layer's simplicial mergeability complex.}
  \vspace{-5pt}
  \label{fig:teaser}
\end{figure}

\section{Introduction}
\label{sec:intro}

In a trained sparse Mixture-of-Experts (MoE) layer, expert $A$ may be pairwise
compatible with expert $B$, expert $B$ with expert $C$, and expert $C$ with
expert $A$, yet collapsing all three into a single expert can still incur an
irreducible barrier that no pair of these compatibilities predicts
(Fig.~\ref{fig:teaser}). Sparsely-gated Mixture-of-Experts~\citep{shazeer2017}
routes each token through only a handful of expert sub-networks, and Switch
Transformer~\citep{fedus2022switch} scaled this construction past the dense
compute frontier; the recipe is now standard in production-grade open
checkpoints such as
Mixtral~\citep{jiang2024mixtral},
OLMoE~\citep{muennighoff2024olmoe},
Qwen 3.5~\citep{qwen3.5}, and
DeepSeek-V3~\citep{deepseekv3-2025},
each of which hosts dozens to hundreds of feed-forward experts per MoE layer
while routing only a small constant subset to any one token. Compressing such
a layer without retraining is therefore a first-order concern:
the parameter mass lives in experts that are inactive at any inference step,
and a learning-free compressor would let practitioners reuse a pretrained MoE
checkpoint on smaller fleets without fine-tuning, knowledge distillation, or
re-pretraining.

The dominant family of such compressors ranks experts on \emph{pairwise}
signals:
REAP~\citep{muzio2025reap} uses output saliency,
REAM~\citep{ream2024} reweights routing through saliency-aware merge scores,
MC-SMoE~\citep{mcsmoe2024} clusters experts via routing-cosine similarity, and
STUN~\citep{xu2025stun} structurally prunes experts before Wanda~\citep{sun2024wanda} unstructured sparsity. Pairwise rankings, however, are
blind to a structural property of mergeability: three experts can be pairwise
mergeable yet \emph{collectively} non-mergeable when their merge-barriers form
an irreducible cycle. Methods that do consume triplet inputs, such as
hypergraph spectral cuts~\citep{zhou2007hypergraph} and triplet-loss penalties~\citep{schroff2015triplet}, aggregate triplets into a sum-of-pairs surrogate
or a binary veto rather than isolating the topological residue. The harmonic
kernel of the simplicial Laplacian~\citep{lim2020hodge,schaub2020simplicial} is the unique linear-algebraic
object that captures this residue, and prior topological deep learning~\citep{bronstein2021geometric} has used it as a diagnostic tool but never as a
parameter-selection objective.

To this end, we introduce \emph{HodgeCover}, a learning-free expert-selection
objective that turns the harmonic kernel into a coverage criterion. From a
sparse MoE layer with $n$ experts we build a $2$-dimensional simplicial
complex $K$ whose vertices are experts, whose edges carry pairwise KL
merge-barriers, and whose $2$-faces carry triplet barriers, and we apply the
simplicial Laplacian~\citep{lim2020hodge}
$L_1=\partial_1^{\!\top}\partial_1+\partial_2\partial_2^{\!\top}$ to Hodge-decompose the edge-barrier signal into gradient,
curl, and harmonic components. HodgeCover then picks $k$ survivor experts by
greedy submodular coverage~\citep{nemhauser1978submodular} of the top-$p$\%
harmonic-critical edges and triplet-critical triangles, weighted by a saliency
score adapted from REAP~\citep{muzio2025reap}; non-survivors are pruned with a
router redirect to the closest survivor. The hybrid \emph{HodgeCover+Wanda}
pairs this expert-axis compressor with unstructured Wanda~\citep{sun2024wanda}
pruning on survivor weights, composing two orthogonal compression axes inside
a single learning-free pipeline.

We summarize our contributions as follows.
\begin{itemize}[leftmargin=*]
  \item We define the simplicial mergeability complex of a sparse MoE layer
  and identify its higher-order mergeability obstruction with the harmonic
  component of its Hodge decomposition, verifying this component is
  non-trivial ($29$\,--\,$62\%$ of per-layer barrier energy) across three MoE
  families (\S\ref{sec:complex}).

  \item We instantiate HodgeCover, a learning-free objective that greedily
  covers the top-$p$\% harmonic-critical edges and triplet-critical triangles,
  and a hybrid HodgeCover+Wanda that pairs it with unstructured Wanda pruning
  on survivors (\S\ref{sec:method}, Fig.~\ref{fig:framework}).

  \item At $66\%$ expert reduction across three MoE scales, HodgeCover+Wanda
  achieves the best perplexity on every model and outperforms STUN+Wanda~\citep{xu2025stun} by up to $+12.6$ pp downstream accuracy, while a matched
  ablation that skips the Hodge step costs $-5.74$ pp on Qwen 3.5-35B
  (\S\ref{sec:exp}, \S\ref{sec:ablations}).
\end{itemize}

\section{Related Work}
\label{sec:related}

\textbf{MoE expert pruning and merging} is a static post-training family that
ranks experts on pairwise signals scored from a small calibration set:
REAP~\citep{muzio2025reap} on output saliency,
REAM~\citep{ream2024} on saliency-aware merge scoring,
MC-SMoE~\citep{mcsmoe2024} and HC-SMoE~\citep{chen2024hcsmoe} on
routing-cosine clustering, and
STUN~\citep{xu2025stun} on a structured-then-unstructured decision composed
with Wanda~\citep{sun2024wanda} and OWL~\citep{yin2023outlier}. Calibration-free pretrained-weight scores~\citep{lu2024notall,aimer2026} and router-gated weight metrics~\citep{moepruner2024} relax the supervision of the same pairwise rule, while
evolutionary search over expert subsets~\citep{liu2024eep,evoesap2026},
prune-then-recompose redistribution~\citep{yang2024moei2}, and unified
slim-trim trimming~\citep{he2024merging} relax the one-shot constraint;
learning-based independent-expert training and merging~\citep{li2022branch,zhong2022meta}
and dense-to-sparse upcycling~\citep{komatsuzaki2023sparse} sit orthogonally
on the capacity-construction axis. Every method aggregates the merge landscape into
edge-level pairwise scores, and none represents the higher-order obstruction
that arises when three experts are pairwise mergeable but not jointly
mergeable. HodgeCover is the first learning-free expert-selection objective
that operates on the harmonic kernel of the layer's simplicial mergeability
complex.

\textbf{Sparsity-based weight pruning} compress LLMs at
the weight-matrix level along three axes. One-shot unstructured pruning ranks weights from
calibration activations (Wanda~\citep{sun2024wanda},
SparseGPT~\citep{frantar2023sparsegpt}) and extends a magnitude- and
movement-pruning lineage~\citep{han2016deepcompression,sanh2020movement};
structured pruning removes coupled channels, heads, or transformer
blocks (LLM-Pruner~\citep{ma2023llmpruner}, ShortGPT~\citep{men2024shortgpt});
post-training quantization compresses bit-width rather than count
(GPTQ~\citep{frantar2023gptq}, AWQ~\citep{lin2024awq},
SmoothQuant~\citep{xiao2023smoothquant},
OmniQuant~\citep{shao2024omniquant}). All these axes operate on weight
matrices and cannot exploit the structural redundancy among MoE experts; on a
sparse MoE checkpoint they leave expert-count compression entirely on the
table. HodgeCover+Wanda composes the two axes inside one learning-free
pipeline: Stage-1 expert reduction on the topological signal, Stage-2
unstructured Wanda on survivor weights.

\textbf{Topological and spectral methods on neural networks} apply
algebraic-topology and spectral machinery along two complementary lines. The architectural line installs Hodge and simplicial
operators as message-passing primitives or feature layers (geometric deep
learning~\citep{bronstein2021geometric}, simplicial neural networks~\citep{ebli2020simplicial,bodnar2021weisfeiler}, topological
GNNs~\citep{horn2022topology}), with foundational Hodge theory traced to~\citet{eckmann1944} and modernized for discrete data~\citep{lim2020hodge,schaub2020simplicial,zheng2026hsd}. The diagnostic line summarises Betti
numbers, persistence diagrams, and Hodge spectra of trained networks for
generalization analysis~\citep{naitzat2020topology,rieck2019neural,grande2024simplicial,zheng2026hsd}; the closest
matrix-level precedents that act on a spectral signal,
spectral sparsification~\citep{spielman2011spectral} and
spectral pruning~\citep{suzuki2020spectral}, operate at the dyadic level and
never lift to higher-order simplicial structure. Across both lines topology
is consumed as input or summarised as output, and almost no prior work
operationalizes it into a parameter-selection objective. We lift the Hodge
decomposition from analysis tool to engineering selection criterion that
drives a practical compressor.

\textbf{Triplet- and hypergraph-aware methods in compression and clustering}
encode higher-order relational structure as hyperedges or triplet
constraints in two distinct families. Hypergraph spectral methods reduce
triplet (or larger) cliques to a clique-expansion Laplacian for cuts
\citep{zhou2007hypergraph,agarwal2006higher},
higher-order Cheeger inequalities~\citep{lee2014multiway,louis2015hypergraph},
motif-based partitioning~\citep{benson2016higher}, and hypergraph convolution
and attention~\citep{feng2019hypergraph,bai2021hypergraph}. Metric learning
encodes triplets as scalar margin constraints in pair-contrastive form
\citep{hadsell2006dimensionality} or triplet form
(FaceNet~\citep{schroff2015triplet}, deep ranking~\citep{wang2014triplet}).
All these methods aggregate higher-order inputs into a sum-of-pairs surrogate
or a binary veto, never separating the gradient, curl, and harmonic
components of the signal. We Hodge-decompose triplet-augmented edge barriers
and show (\S\ref{sec:ablations}) that the harmonic component is the
load-bearing piece: a soft-triplet ablation that skips the decomposition
costs $-5.74$\,pp downstream-task average accuracy on Qwen 3.5-35B at $66\%$ compression.

\section{The Mergeability Complex and Its Hodge Decomposition}
\label{sec:complex}

We now formalize the higher-order obstruction sketched in
Section~\ref{sec:intro}; the full Hodge primer, all proofs, projection
formulas, and extended diagnostics are in App.~\ref{app:hodge}.

\subsection{Problem statement and notation}
\label{sub:problem}

Fix a sparse Mixture-of-Experts (MoE) layer with $n$ experts
$\{f_1, \ldots, f_n\}$ and a learned router $g$, and write $f(\cdot \mid x)$
for its output next-token distribution at input $x$. Given a target survivor
count $k < n$, learning-free expert compression seeks a subset
$S \subset \{1, \ldots, n\}$ with $|S| = k$ and a router-redirect map
$\pi : \{1, \ldots, n\} \setminus S \to S$ such that the post-compression
layer $f_S$ minimizes the calibration-set KL divergence
\begin{equation}
  \label{eq:objective}
  \mathcal{L}(S, \pi)
    \;=\;
    \mathbb{E}_{x \sim \mathcal{D}}\,
    D_{\mathrm{KL}}\!\left(f(\cdot \mid x) \,\big\Vert\, f_S(\cdot \mid x)\right),
\end{equation}
under the constraint that $f_S$ is obtained from a single forward pass over a
small calibration corpus $\mathcal{D}$, with no fine-tuning, knowledge
distillation, or LoRA adaptation. We fix $\mathcal{D}$ to a small held-out
calibration corpus throughout (App.~\ref{app:calibration}).

We treat experts as vertices of an undirected complete graph
$V = \{1, \ldots, n\}$. For an unordered pair
$\{i, j\} \in \binom{V}{2}$, the \emph{pairwise merge barrier}
\begin{equation}
  \label{eq:pair-barrier}
  b_{ij}
    \;=\;
    \mathbb{E}_{x \sim \mathcal{D}}\,
    D_{\mathrm{KL}}\!\left(
      f(\cdot \mid x) \,\big\Vert\, f_{V \setminus \{i,j\}\,\cup\,\{i \oplus j\}}(\cdot \mid x)
    \right)
\end{equation}
measures the distributional cost of replacing $\{i, j\}$ with their
frequency-weighted merge $i \oplus j$, where the frequency weighting is the
empirical token-routing rate of each expert on $\mathcal{D}$ and a guard
falls back to the unweighted average when both rates vanish
(App.~\ref{app:merge}). For an unordered triple
$\{i, j, k\} \in \binom{V}{3}$, the \emph{triplet merge barrier} $b_{ijk}$ is
defined analogously, by replacing $\{i, j, k\}$ with their joint
frequency-weighted merge $i \oplus j \oplus k$. These pairwise and triplet
barriers will be lifted to edge- and triangle-supported signals on the
2-complex defined next.

\subsection{The mergeability complex and the simplicial Laplacian}
\label{sub:complex}

The \emph{mergeability complex} of an MoE layer is the abstract simplicial
2-complex $K = (V, E, T)$ with $E = \binom{V}{2}$ (the complete pairwise edge
set) and $T \subseteq \binom{V}{3}$ a curated triangle set whose
two-stage construction is included in App.~\ref{app:triangle-set}. Pairwise barriers form an
edge-supported signal $b \in C_1(K)$ with coefficient $b_{ij}$ on edge
$\{i, j\}$, and triplet barriers form a triangle-supported signal
$c \in C_2(K)$ with coefficient $b_{ijk}$ on triangle $\{i, j, k\}$. They
enter the analysis as signals on $K$, never as weights inside the operator
$L_1$ defined below. We work with the complete edge set rather than a
thresholded subgraph $\{(i,j) : b_{ij} \le \tau_e\}$, because thresholding
distorts the Hodge spectrum: it changes the rank of $\partial_1$ on
disconnection events and can detach otherwise harmonic cycles
(App.~\ref{app:thresholding}).

Fix any total ordering on $V$ and the induced lexicographic orientation on
edges and triangles. Let $C_q(K) = \mathbb{R}^{F_q(K)}$, where $F_q(K)$ is
the set of oriented $q$-simplices ($q = 0, 1, 2$); throughout this section
$q$ indexes the chain dimension and is unrelated to the percentile
thresholds for top-percentage edges and triangles introduced in
Section~\ref{sec:method}. For oriented edges
$[i, j]$ with $i < j$ and oriented triangles $[i, j, k]$ with $i < j < k$,
the signed boundary operators are
\begin{align}
  \label{eq:bdry1}
  \partial_1 : C_1(K) &\to C_0(K),
    & \partial_1\,[i, j] &= [j] - [i], \\
  \label{eq:bdry2}
  \partial_2 : C_2(K) &\to C_1(K),
    & \partial_2\,[i, j, k] &= [j, k] - [i, k] + [i, j],
\end{align}
satisfying the chain-complex identity $\partial_1 \circ \partial_2 = 0$ that
underlies all of homology theory \citep{hatcher2002, munkres1984}. The
\emph{1-Hodge Laplacian} of $K$ is
\begin{equation}
  \label{eq:laplacian}
  L_1
    \;=\;
    \partial_1^{\!\top}\partial_1 \;+\; \partial_2\,\partial_2^{\!\top},
\end{equation}
a symmetric positive semi-definite operator on $C_1(K)$
\citep{eckmann1944, lim2020hodge, schaub2020simplicial, horak2013spectra}.
We use the unweighted form of $L_1$ throughout: barriers enter only as the
edge-supported signal $b \in C_1(K)$, never as edge weights inside the
operator. The kernel $\ker(L_1)$ depends only on the combinatorial topology
of $K$ (App.~\ref{app:laplacian}), so the unweighted $L_1$ keeps the
harmonic projector independent of how barriers are normalized across
layers and across model families.

\subsection{Hodge decomposition of the merge-barrier signal}
\label{sub:hodge}

\begin{theorem}[Discrete Hodge decomposition; \citep{eckmann1944,lim2020hodge}]
\label{thm:hodge}
Every edge-supported signal $b \in C_1(K)$ admits a unique orthogonal
decomposition
\begin{equation}
  \label{eq:hodge}
  b
    \;=\;
    b_{\mathrm{grad}}
    \;+\; b_{\mathrm{curl}}
    \;+\; b_{\mathrm{harm}},
\end{equation}
with $b_{\mathrm{grad}} \in \mathrm{im}(\partial_1^{\!\top})$,
$b_{\mathrm{curl}} \in \mathrm{im}(\partial_2)$, and
$b_{\mathrm{harm}} \in \ker(L_1)$. The three subspaces are pairwise
orthogonal in $\mathbb{R}^{|E|}$, and
$\dim \ker(L_1) = \beta_1(K)$, the first Betti number of $K$.
\end{theorem}

A self-contained linear-algebraic proof using the chain identity
$\partial_1\partial_2 = 0$ is in App.~\ref{app:hodge-proof}; the explicit
projection formulas
$P_{\mathrm{grad}} = \partial_1^{\!\top}(\partial_1\partial_1^{\!\top})^{+}\partial_1$,
$P_{\mathrm{curl}} = \partial_2(\partial_2^{\!\top}\partial_2)^{+}\partial_2^{\!\top}$,
and $P_{\mathrm{harm}} = I - P_{\mathrm{grad}} - P_{\mathrm{curl}}$, with
$(\cdot)^{+}$ the Moore-Penrose pseudoinverse, are derived in
App.~\ref{app:projection}.

\begin{proposition}[Harmonic energy as irreducible mergeability residual]
\label{prop:harmonic-residual}
Let
$\mathcal{M}_K
  \;:=\; \mathrm{im}(\partial_1^{\!\top})\,\oplus\,\mathrm{im}(\partial_2)
  \,\subseteq\, C_1(K)$
denote the subspace of edge-barrier signals expressible as a vertex
potential lifted to edges plus a triangle-boundary correction. For every
$b \in C_1(K)$,
\begin{equation}
  \label{eq:harmonic-residual}
  \inf_{m \in \mathcal{M}_K}\,\norm{b - m}^2
    \;=\; \norm{b_{\mathrm{harm}}}^2,
\end{equation}
with unique minimizer $m^\star = b_{\mathrm{grad}} + b_{\mathrm{curl}}$.
\end{proposition}

In particular, the harmonic energy fraction $\rho_{\mathrm{harm}}(\ell)$
in Eq.~\ref{eq:harmonic-frac} is the fraction of layer-$\ell$
merge-barrier energy that no vertex-potential or triangle-boundary
explanation can capture. A one-line orthogonal-projection proof is in
App.~\ref{app:harmonic-residual}; a conditional corollary linking this
residual to the calibration KL loss $\mathcal{L}(S, \pi)$ under an
edge-exposure linearization is in App.~\ref{app:kl-corollary}, where it
is used only as interpretation, not as an assumption in the algorithm.

The three components have a transparent interpretation specific to the
mergeability complex. The gradient component $b_{\mathrm{grad}}$ is the lift
to edges of a vertex-supported potential, and answers the question
\emph{"is expert $i$ broadly mergeable?"} as a per-expert score; pairwise
methods~\citep{muzio2025reap,xu2025stun}
implicitly recover only this component. The curl component
$b_{\mathrm{curl}}$ is the lift to edges of a triangle-supported potential,
and captures whether \emph{around triangle $(i,j,k)$ the merge barriers close
up coherently}; hypergraph cuts \citep{zhou2007hypergraph} and triplet
penalties \citep{schroff2015triplet, wang2014triplet} aggregate this
component into a sum-of-pairs surrogate. The harmonic component
$b_{\mathrm{harm}} \in \ker(L_1)$ is the residue that no triangulation of
pairwise scores can flatten: by Proposition~\ref{prop:harmonic-residual}
it is the unique edge-barrier component irreducible under any
vertex-potential plus triangle-boundary explanation, and the ambient
harmonic subspace $\ker(L_1)$ has dimension $\beta_1(K)$
\citep{eckmann1944, hatcher2002}. None of the prior pairwise or triplet
pruners isolates it.

\subsection{Per-layer diagnostic across three MoE families}
\label{sub:diagnostic}

\begin{figure}[!t]
  \centering
  \includegraphics[width=\linewidth]{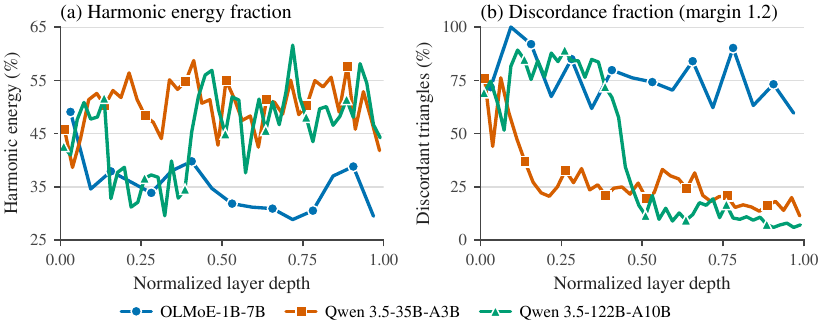}
  \caption{Harmonic energy fraction $\rho_{\mathrm{harm}}(\ell)$
  (Eq.~\ref{eq:harmonic-frac}) and discordance fraction $\delta(\ell)$
  (Eq.~\ref{eq:discordance}) at every layer of three
  production MoE families. Both signals stay non-trivial throughout depth.}
  \vspace{-10pt}
  \label{fig:diagnostic}
\end{figure}

Two complementary per-layer diagnostics test whether the harmonic
component is non-trivial in production sparse MoEs. Let
$b^{(\ell)} \in \mathbb{R}^{|E|}$ denote the edge-supported pairwise-KL
merge-barrier signal at layer $\ell$. The first is the
\emph{harmonic energy fraction}
\begin{equation}
  \label{eq:harmonic-frac}
  \rho_{\mathrm{harm}}(\ell)
    \;=\;
    \frac{\norm{P_{\mathrm{harm}}\,b^{(\ell)}}^2}
         {\norm{b^{(\ell)}}^2}
  \qquad (b^{(\ell)} \neq 0),
\end{equation}
a global property of the simplicial Laplacian. The second is the
combinatorial \emph{discordance fraction}
\begin{equation}
  \label{eq:discordance}
  \delta(\ell)
    \;=\;
    \frac{1}{|T|}\,\Big|\big\{\{i,j,k\} \in T :
      b_{ijk}^{(\ell)} > 1.2 \cdot \max\!\big(b_{ij}^{(\ell)}, b_{ik}^{(\ell)}, b_{jk}^{(\ell)}\big)
    \big\}\Big|,
\end{equation}
the fraction of sampled triangles whose joint merge barrier exceeds the
worst pairwise barrier by a $20\%$ margin. Both live in $[0, 1]$ and are
scale-invariant per layer; $\rho_{\mathrm{harm}}$ is a spectral property
tied to $\ker(L_1)$, while $\delta$ is a combinatorial stress test that
does not require diagonalizing $L_1$.

Figure~\ref{fig:diagnostic} reports both diagnostics on every layer of
OLMoE-1B-7B \citep{muennighoff2024olmoe}, Qwen3.5-35B-A3B and
Qwen3.5-122B-A10B \citep{qwen3.5}. The harmonic energy fraction stays
non-trivial at every layer of every model, with the two Qwen variants in a
slightly higher and tighter band than OLMoE; the discordance fraction is
also non-trivial throughout, with OLMoE the most discordant model and the
two Qwen models tapering in deeper layers. The two diagnostics agree in
direction and disagree in shape, giving independent evidence of an
irreducible higher-order obstruction; per-layer ranges, gradient/curl
companion curves, and the Euler-Poincar\'{e} count for the first Betti
number $\beta_1(K)$ (which is structurally pinned by $n$ and $|T|$ and
thus demoted from per-layer analysis) are reported in
App.~\ref{app:extended-diagnostic} and~\ref{app:laplacian}.

By Proposition~\ref{prop:harmonic-residual}, a non-trivial
$\rho_{\mathrm{harm}}(\ell)$ measures the fraction of layer-$\ell$
merge-barrier energy that no vertex-potential or triangle-boundary
model can explain, so Figure~\ref{fig:diagnostic} quantifies the
lower-order-irreducible barrier energy that HodgeCover explicitly
targets. We \emph{designate} the top-$p\%$ edges of $K$
ranked by
$|b_{\mathrm{harm}, e}|$, together with the top-$q_T\%$ triangles ranked
by $|b_{ijk}|$, as coverage constraints for survivor selection in the
next section, and call the resulting selector HodgeCover
(Section~\ref{sec:method}).

\section{Method}
\label{sec:method}

\begin{figure*}[!t]
  \centering
  \includegraphics[width=\linewidth]{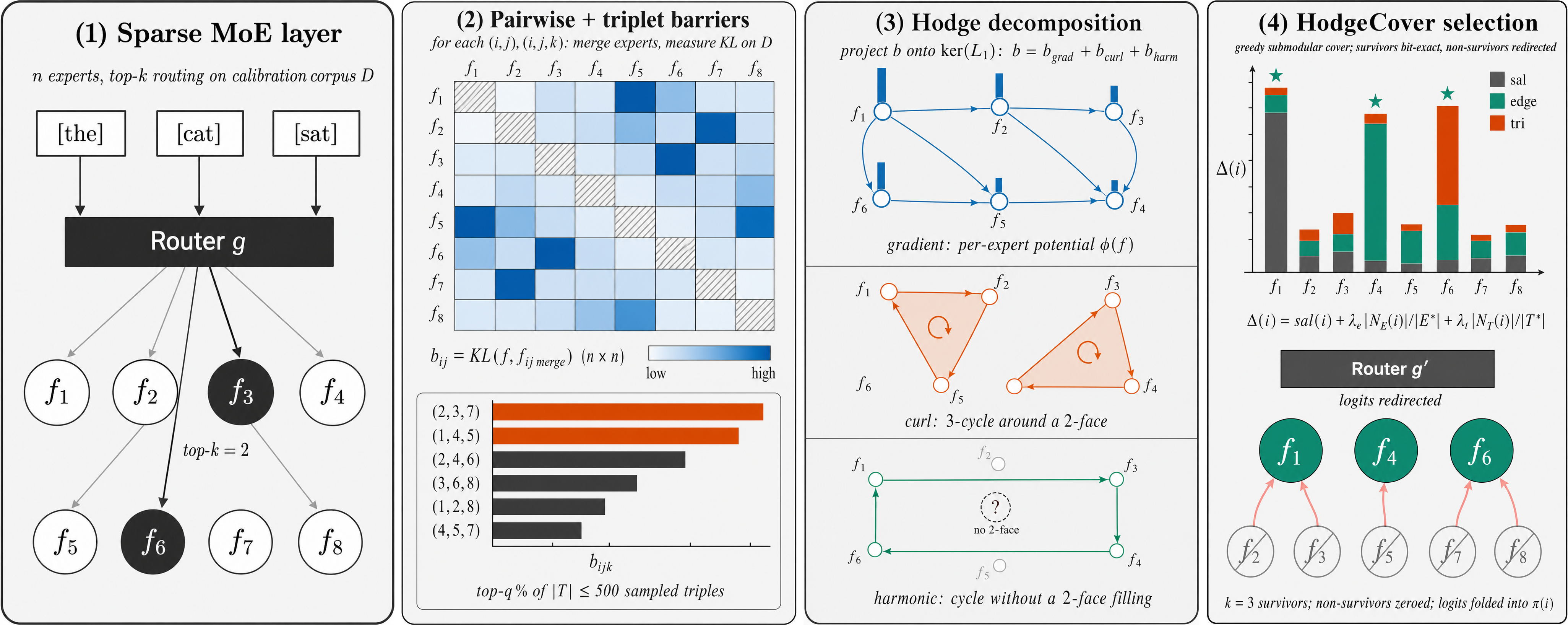}
  \caption{The four stages of HodgeCover: token-level routing, pairwise
  + triplet KL barriers on the calibration corpus, Hodge decomposition
  of the edge-supported barrier signal, and submodular survivor
  selection with a Hodge-weighted router redirect.}
  \vspace{-10pt}
  \label{fig:framework}
\end{figure*}

We turn the diagnostic of Section~\ref{sec:complex} into a learning-free
expert-selection algorithm, \emph{HodgeCover}
(\S\ref{sub:hodgecover}), and a hybrid \emph{HodgeCover+Wanda}
that composes it with off-the-shelf weight pruning
(\S\ref{sub:hodgewanda}).

\subsection{HodgeCover}
\label{sub:hodgecover}

Let $k < n$ be the target survivor count from the layer's $n$ experts.
HodgeCover selects a survivor set $S^\star \subseteq V$ of size $k$ in
three steps.

\textbf{Critical Simplices.} From the Hodge decomposition
$b = b_{\mathrm{grad}} + b_{\mathrm{curl}} + b_{\mathrm{harm}}$
(Theorem~\ref{thm:hodge}) we extract the
\emph{harmonic-critical edge set} $E^\star \subseteq E$ as the top-$p\%$
edges ranked by $|b_{\mathrm{harm},e}|$, and the
\emph{triplet-critical triangle set} $T^\star \subseteq T$ as the
top-$q_T\%$ triangles ranked by raw $|b_{ijk}|$ rather than by a $2$-harmonic projection. The $2$-harmonic kernel of the triangle-supported signal
$c \in C_2(K)$ requires assembling the next Laplacian $L_2$ and adds
no information here, since on the curated $T$ the
top-$q_T\%$ triplet-barrier triangles already mark the regions where
joint merging amplifies pairwise cost
(Section~\ref{sub:diagnostic}, App.~\ref{app:triangle-set}). Both $p$ and $q_T$ are scalar
percentile thresholds reported in App.~\ref{app:hyperparameters}.

\textbf{Coverage Objective.} Let
$N_E(i) = \{e \in E^\star : i \in e\}$ and
$N_T(i) = \{\sigma \in T^\star : i \in \sigma\}$ be the per-expert
incidence sets, and write
$C_E(S) = \bigcup_{i \in S} N_E(i)$,
$C_T(S) = \bigcup_{i \in S} N_T(i)$ for what a candidate
$S \subseteq V$ covers. HodgeCover defines the selection objective
\begin{equation}
  \label{eq:hc-objective}
  \Phi(S)
    \;:=\;
    \underbrace{\sum_{i \in S}\mathrm{sal}(i)}_{\text{saliency}}
    \;+\;
    \lambda_e\;
    \underbrace{\frac{|C_E(S)|}{|E^\star|}}_{\text{harmonic-edge coverage}}
    \;+\;
    \lambda_t\;
    \underbrace{\frac{|C_T(S)|}{|T^\star|}}_{\text{triplet-triangle coverage}},
\end{equation}
with $\lambda_e, \lambda_t \ge 0$ and saliencies $\mathrm{sal}(i) \in [0, 1]$
given by the REAP output-saliency score \citep{muzio2025reap}
normalized per layer (with the convention $|C_E(S)|/|E^\star| = 0$ when
$E^\star = \emptyset$, and likewise for $T^\star$). The objective $\Phi$
is non-negative monotone submodular
(Proposition~\ref{prop:topological-coverage}); the standard greedy
procedure adds, at each step, the expert maximizing the marginal score
\begin{equation}
  \label{eq:greedy-marginal}
  \Delta(i \mid S)
    \;=\;
    \mathrm{sal}(i)
    \;+\;
    \lambda_e\,\frac{|N_E(i) \setminus C_E(S)|}{|E^\star|}
    \;+\;
    \lambda_t\,\frac{|N_T(i) \setminus C_T(S)|}{|T^\star|},
\end{equation}
and terminates after $k$ additions. The denominators
$|E^\star|, |T^\star|$ keep all three terms in $[0, 1]$ across layers
and models, so $\lambda_e, \lambda_t$ are layer-agnostic. The
harmonic-edge term targets the irreducible residual identified by
Proposition~\ref{prop:harmonic-residual}; the triplet-triangle term
protects high-cost joint merges that enter through the curated
triangle set $T^\star$.

\begin{proposition}[Topological coverage of HodgeCover]
\label{prop:topological-coverage}
The HodgeCover selection objective $\Phi$ in
Eq.~\ref{eq:hc-objective} is a non-negative monotone submodular set
function in $S \subseteq V$ subject to the cardinality constraint
$|S| = k$, and the greedy survivor set $S^\star$ produced by
iterating Eq.~\ref{eq:greedy-marginal} from the empty set
($\mathrm{SE} = \emptyset$ in Algorithm~\ref{alg:hodgecover}, the
setting used for every reported run) satisfies
\begin{equation}
  \label{eq:topological-coverage}
  \Phi(S^\star)
    \;\ge\;
    \big(1 - 1/e\big)\,\max_{S \subseteq V,\,|S| = k}\,\Phi(S).
\end{equation}
\end{proposition}

A self-contained proof and the formal definition of $\Phi$'s two
coverage components as set-cover instances are in
App.~\ref{app:topological-coverage-proof}.

\begin{remark}[Inexpressibility for per-expert rankings]
\label{rem:inexpressibility}
The harmonic-critical incidence is set-valued in $S$: when critical
edges are shared between experts, $|C_E(S)|$ cannot be expressed as
a top-$k$ rule on any individual per-expert score
$h : V \to \mathbb{R}$. Rankings such as REAP \citep{muzio2025reap},
REAM centroid selection \citep{ream2024}, and MC-SMoE
\citep{mcsmoe2024} all factor through such a vertex score, so they
can match $\Phi$ only in the trivially modular case where
$\{N_E(i)\}_{i \in V}$ are pairwise disjoint;
App.~\ref{app:topological-coverage-proof} gives a constructive
counterexample on $K_4$.
\end{remark}

\textbf{Router Redirect for non-Survivors.} Let
$N^\star = V \setminus S^\star$ be the dropped experts. We use the
REAP redirect surgery \citep{muzio2025reap}: expert weights of each $i \in N^\star$ are erased and its router logit is
folded into a single survivor $\pi(i) \in S^\star$, so the gating
distribution renormalizes onto $S^\star$ at no extra forward pass.
Where REAP picks $\pi(i)$ to be the cosine-nearest survivor of expert
$i$ in router-key space, we pick it to be the nearest survivor under a
\emph{Hodge-weighted barrier} that gives extra weight to harmonic edges:
\begin{equation}
  \label{eq:hc-redirect}
  \pi(i)
    \;=\;
    \argmin_{j \in S^\star}\;
    b_{ij}
    \cdot
    \Big(1 + \alpha\,
      \tfrac{|b_{\mathrm{harm},\{i,j\}}|}{\|b\|}\Big),
    \qquad i \in N^\star,
\end{equation}
where the strength $\alpha > 0$ controls how much edges that carry
their own harmonic mass are penalized: routing non-survivor mass
through such an edge would re-introduce the very obstruction that
the survivor set was selected to cover. We adopt the convention that
the harmonic factor is $1$ when $\|b\| = 0$ (equivalently, when
$b_{\mathrm{harm}} = 0$). Crucially, \emph{survivors are
never perturbed}: their gate vectors and expert weights are kept
bit-exact, so the post-compression layer is the survivor sub-MoE with
a renormalized router. These three steps run independently per MoE layer; the
cross-layer budget allocator that turns a global compression rate into
the per-layer counts $\{k_\ell\}_\ell$ is given in
App.~\ref{app:algorithm}.

\label{sub:costs}
The full HodgeCover pipeline is single-pass and learning-free. Plan-time is dominated by the $O\big((n^2+|T|)\,|\mathcal{D}|\,d^2\big)$
barrier sweep ($d$ = expert hidden dim, $|\mathcal{D}|$ = calibration
tokens). Algorithm pseudocode, hyperparameter defaults,
the cross-layer allocator, and a baseline-axis
comparison grid (Table~\ref{tab:baseline-axis}) versus other
learning-free MoE compressors benchmarked in \S\ref{sec:exp} are
in App.~\ref{app:algorithm} and App.~\ref{app:hyperparameters}.

\subsection{HodgeCover+Wanda}
\label{sub:hodgewanda}

HodgeCover targets the expert-count axis; we obtain a hybrid
compressor, \emph{HodgeCover+Wanda}, by composing it with an
off-the-shelf one-shot weight pruner. Stage~1 runs HodgeCover at a
fixed expert-count drop rate $r_1$, yielding the survivor set
$S^\star$ and the redirected router. Stage~2 applies unstructured
Wanda \citep{sun2024wanda} on the survivor weights at the residual
sparsity needed to reach the total target rate, reusing the same
calibration corpus $\mathcal{D}$ at no additional forward-pass.
Stage~1 is the contribution; Stage~2 is plug-and-play (any
calibration-aware unstructured weight pruner could replace Wanda).
The pseudocode, the residual-sparsity protocol that selects
$r_2$ given a Stage-1 rate $r_1$, and the per-cell numerical values
of $r_1$ and $r_2$ used are in App.~\ref{app:hodgecover-pseudocode} and~\ref{app:wanda-stage}.

\section{Experiments}
\label{sec:exp}

\subsection{Setup}
\label{sub:exp-setup}

We benchmark three open-weight sparse MoE backbones spanning two
expert-count regimes and two scales:
OLMoE-1B-7B \citep{muennighoff2024olmoe} (16 MoE layers, 64 experts),
Qwen 3.5-35B-A3B and Qwen 3.5-122B-A10B \citep{qwen3.5}
(40 / 48 MoE layers, 256 experts).
Each backbone is compressed at $33\%$ and $66\%$ of expert count.
Each baseline follows its source paper's released implementation up to
the calibration-footing reductions detailed in
App.~\ref{app:baseline-impl}: REAP \citep{muzio2025reap}, REAM
\citep{ream2024}, MC-SMoE~\citep{mcsmoe2024} (w/o KD reduction), and STUN~\citep{xu2025stun}. HodgeCover (ours) uses the same
uniform per-layer convention as REAP and REAM
(App.~\ref{app:method-detail}). Calibration uses $2{,}048$ C4-train
tokens \citep{raffel2020c4}; evaluation reports WikiText-103
perplexity \citep{merity2017wikitext}, C4 perplexity, and nine
downstream tasks via the LM Evaluation Harness \citep{gao2024lmeval}:
ARC-c and ARC-e \citep{clark2018arc}, BoolQ \citep{clark2019boolq},
HellaSwag \citep{zellers2019hellaswag}, MMLU 5-shot
\citep{hendrycks2021mmlu}, PIQA \citep{bisk2020piqa}, TruthfulQA-MC2
\citep{lin2022truthfulqa}, WinoGrande \citep{sakaguchi2021winogrande},
and GSM8K 8-shot \citep{cobbe2021gsm8k}. We summarize the nine
downstream-task accuracies by their unweighted arithmetic mean, the
\emph{downstream-task average} (DS-Avg).
Beyond the four published baselines we also report four ablations of
HodgeCover that share its calibration corpus and per-layer expert
budget but replace the topological objective with a no-triangle,
soft-triplet-penalty, hard-triplet-veto, or topology-blind alternative
\citep{zhou2007hypergraph}; precise constructions, hardware, harness
configuration, and per-method allocator descriptions are in
App.~\ref{app:method-detail} and App.~\ref{app:expsetup}.

\subsection{Main results}
\label{sub:exp-main}

% Merged main results across 3 backbones (Model + Rate columns)
\begin{table*}[!t]
  \centering
  \footnotesize
  \setlength{\tabcolsep}{3pt}
  \renewcommand{\arraystretch}{1.0}
  \caption{Per-task downstream accuracy (\%) plus WikiText-103 / C4 perplexity, with the rightmost column reporting the unweighted nine-task DS-Avg. \textbf{Bold} = best in column within the same axis; \underline{underline} = second-best in the expert-count axis.}
  \label{tab:main-results}
  \resizebox{\linewidth}{!}{%
  \begin{tabular}{ll|l|cc|ccccccccc|c}
    \toprule
    Model & Rate & Method & Wiki & C4 & ARC-c & ARC-e & BoolQ & HellaS & MMLU & PIQA & TQA & WinoG & GSM8K & DS-Avg \\
    \midrule
    \multirow{17}{*}{\makecell[l]{OLMoE\\1B-7B}} & -- & \emph{Uncompressed} & \emph{8.17} & \emph{13.64} & \emph{48.9} & \emph{76.2} & \emph{75.0} & \emph{77.0} & \emph{52.8} & \emph{80.7} & \emph{35.6} & \emph{68.7} & \emph{9.8} & \emph{58.3} \\
    \cmidrule(lr){2-15}
     & \multirow{8}{*}{$33\%$} & \multicolumn{12}{l}{\textit{Expert-count reduction methods}} \\
     &  & \quad REAP \citep{muzio2025reap} & \textbf{19.28} & \textbf{19.62} & \underline{28.7} & \underline{42.5} & \textbf{65.3} & \underline{63.8} & \underline{27.4} & \underline{66.8} & 35.5 & \underline{64.8} & \underline{1.7} & \underline{44.1} \\
     &  & \quad REAM \citep{ream2024} & 33.00 & 26.64 & 27.1 & 39.3 & \underline{63.3} & 53.0 & \textbf{27.5} & 63.4 & \underline{41.5} & 60.9 & 1.3 & 41.9 \\
     &  & \quad MC-SMoE \citep{mcsmoe2024} & 39.56 & 39.87 & 26.3 & 41.8 & 62.3 & 43.3 & 27.2 & 65.1 & 39.2 & 50.3 & 1.2 & 39.6 \\
     &  & \quad \textbf{HodgeCover (ours)} & \underline{23.50} & \underline{22.91} & \textbf{32.8} & \textbf{50.3} & 62.8 & \textbf{67.0} & 25.6 & \textbf{73.0} & \textbf{41.7} & \textbf{64.9} & \textbf{2.2} & \textbf{46.7} \\
     &  & \multicolumn{12}{l}{\textit{Hybrid methods}} \\
     &  & \quad STUN+Wanda \citep{xu2025stun,sun2024wanda} & \textbf{11.85} & 18.43 & \textbf{41.5} & 67.2 & 66.4 & 69.2 & 34.6 & 77.3 & 33.3 & 64.8 & \textbf{2.6} & 50.8 \\
     &  & \quad \textbf{HodgeCover+Wanda (ours)} & 13.92 & \textbf{16.69} & 40.9 & \textbf{67.3} & \textbf{73.0} & \textbf{73.6} & \textbf{38.7} & \textbf{78.1} & \textbf{35.1} & \textbf{67.8} & 2.4 & \textbf{53.0} \\
    \cmidrule(lr){2-15}
     & \multirow{8}{*}{$66\%$} & \multicolumn{12}{l}{\textit{Expert-count reduction methods}} \\
     &  & \quad REAP \citep{muzio2025reap} & \textbf{864.7} & \textbf{180.8} & \textbf{24.5} & \textbf{30.9} & \textbf{48.0} & \textbf{32.3} & 25.1 & \textbf{54.2} & \textbf{50.7} & \textbf{53.2} & \underline{0.0} & \textbf{35.4} \\
     &  & \quad REAM \citep{ream2024} & 4,131.2 & 1,026.3 & \underline{24.3} & 30.4 & 42.8 & 28.8 & \underline{25.3} & 53.2 & \underline{50.2} & 50.9 & 0.0 & \underline{34.0} \\
     &  & \quad MC-SMoE \citep{mcsmoe2024} & 1,443.7 & 835.2 & 23.6 & \underline{30.8} & \underline{43.3} & 27.9 & 23.5 & \underline{54.1} & 50.0 & 50.8 & 0.0 & 33.8 \\
     &  & \quad \textbf{HodgeCover (ours)} & \underline{1,134.9} & \underline{535.5} & 24.2 & 29.5 & 38.6 & \underline{30.0} & \textbf{25.4} & 52.7 & 47.7 & \underline{51.2} & \textbf{0.1} & 33.3 \\
     &  & \multicolumn{12}{l}{\textit{Hybrid methods}} \\
     &  & \quad STUN+Wanda \citep{xu2025stun,sun2024wanda} & 22.11 & 31.01 & 32.9 & \textbf{58.0} & 62.5 & 55.3 & 26.9 & 71.7 & \textbf{36.0} & 61.1 & 1.8 & 45.1 \\
     &  & \quad \textbf{HodgeCover+Wanda (ours)} & \textbf{18.32} & \textbf{21.18} & \textbf{33.7} & 56.6 & \textbf{72.6} & \textbf{64.4} & \textbf{35.0} & \textbf{72.6} & 35.2 & \textbf{64.5} & \textbf{2.5} & \textbf{48.6} \\
    \midrule
    \multirow{17}{*}{\makecell[l]{Qwen 3.5\\35B-A3B}} & -- & \emph{Uncompressed} & \emph{7.25} & \emph{13.19} & \emph{62.1} & \emph{79.5} & \emph{87.9} & \emph{82.5} & \emph{84.9} & \emph{83.0} & \emph{53.5} & \emph{74.7} & \emph{86.9} & \emph{77.2} \\
    \cmidrule(lr){2-15}
     & \multirow{8}{*}{$33\%$} & \multicolumn{12}{l}{\textit{Expert-count reduction methods}} \\
     &  & \quad REAP \citep{muzio2025reap} & \textbf{9.65} & \textbf{13.71} & \textbf{57.1} & \textbf{76.0} & 90.0 & 81.5 & 79.8 & \underline{82.0} & 53.2 & 75.1 & \textbf{87.6} & \underline{75.8} \\
     &  & \quad REAM \citep{ream2024} & \underline{9.91} & 14.14 & 54.4 & 73.1 & \underline{90.5} & 80.6 & 78.8 & 81.2 & 53.8 & 74.0 & \underline{86.8} & 74.8 \\
     &  & \quad MC-SMoE \citep{mcsmoe2024} & 10.90 & 14.97 & 45.1 & 64.6 & 89.2 & 78.3 & 62.5 & 81.0 & 51.3 & \textbf{76.1} & 83.9 & 70.2 \\
     &  & \quad \textbf{HodgeCover (ours)} & 9.97 & \underline{13.75} & \underline{56.3} & \underline{73.9} & \textbf{91.3} & \textbf{82.5} & \textbf{80.7} & \textbf{82.9} & \textbf{54.7} & \underline{76.3} & 85.1 & \textbf{75.9} \\
     &  & \multicolumn{12}{l}{\textit{Hybrid methods}} \\
     &  & \quad STUN+Wanda \citep{xu2025stun,sun2024wanda} & \textbf{8.86} & 15.61 & 50.9 & 74.4 & \textbf{89.2} & 72.1 & 76.7 & 76.1 & 51.0 & 70.6 & 33.2 & 66.0 \\
     &  & \quad \textbf{HodgeCover+Wanda (ours)} & 9.18 & \textbf{13.42} & \textbf{55.9} & \textbf{76.0} & 89.0 & \textbf{81.8} & \textbf{82.6} & \textbf{82.3} & \textbf{53.7} & \textbf{74.2} & \textbf{89.6} & \textbf{76.1} \\
    \cmidrule(lr){2-15}
     & \multirow{8}{*}{$66\%$} & \multicolumn{12}{l}{\textit{Expert-count reduction methods}} \\
     &  & \quad REAP \citep{muzio2025reap} & \textbf{14.97} & \underline{19.29} & \underline{46.0} & \textbf{69.8} & \textbf{88.2} & \underline{71.4} & \textbf{58.7} & \underline{77.3} & \underline{48.2} & \underline{72.5} & \underline{65.3} & \underline{66.4} \\
     &  & \quad REAM \citep{ream2024} & 16.80 & 22.32 & 42.7 & 64.4 & 87.2 & 67.9 & 51.7 & 74.1 & 46.9 & 70.9 & 39.0 & 60.5 \\
     &  & \quad MC-SMoE \citep{mcsmoe2024} & 20.64 & 24.25 & 29.0 & 47.5 & 85.1 & 54.9 & 32.7 & 69.4 & 47.2 & 67.6 & 34.6 & 52.0 \\
     &  & \quad \textbf{HodgeCover (ours)} & \underline{15.13} & \textbf{18.86} & \textbf{46.4} & \underline{67.5} & \underline{88.1} & \textbf{72.5} & \underline{54.5} & \textbf{78.1} & \textbf{51.7} & \textbf{73.7} & \textbf{67.5} & \textbf{66.7} \\
     &  & \multicolumn{12}{l}{\textit{Hybrid methods}} \\
     &  & \quad STUN+Wanda \citep{xu2025stun,sun2024wanda} & 11.77 & 20.05 & 49.6 & 76.6 & 89.5 & 63.8 & 69.2 & 72.9 & 47.1 & 67.5 & 22.1 & 62.0 \\
     &  & \quad \textbf{HodgeCover+Wanda (ours)} & \textbf{10.25} & \textbf{15.22} & \textbf{56.5} & \textbf{78.5} & \textbf{90.8} & \textbf{76.8} & \textbf{78.0} & \textbf{80.8} & \textbf{50.3} & \textbf{74.1} & \textbf{85.5} & \textbf{74.6} \\
    \midrule
    \multirow{17}{*}{\makecell[l]{Qwen 3.5\\122B-A10B}} & -- & \emph{Uncompressed} & \emph{4.50} & \emph{12.33} & \emph{63.6} & \emph{80.5} & \emph{86.3} & \emph{85.9} & \emph{88.1} & \emph{83.4} & \emph{51.8} & \emph{74.6} & \emph{88.3} & \emph{78.0} \\
    \cmidrule(lr){2-15}
     & \multirow{8}{*}{$33\%$} & \multicolumn{12}{l}{\textit{Expert-count reduction methods}} \\
     &  & \quad REAP \citep{muzio2025reap} & \underline{7.21} & \textbf{12.78} & \textbf{63.3} & \underline{81.6} & 71.8 & \underline{84.8} & \textbf{84.7} & \textbf{83.3} & \underline{53.2} & \underline{77.1} & 81.7 & \textbf{75.7} \\
     &  & \quad REAM \citep{ream2024} & 7.52 & 12.98 & \underline{61.9} & \textbf{81.7} & \underline{74.6} & 83.5 & 83.5 & \underline{83.1} & 52.9 & 75.6 & \underline{82.6} & 75.5 \\
     &  & \quad MC-SMoE \citep{mcsmoe2024} & 9.44 & 13.96 & 49.2 & 72.4 & \textbf{80.5} & 81.8 & 71.6 & 82.1 & \textbf{54.0} & 75.6 & 82.4 & 72.2 \\
     &  & \quad \textbf{HodgeCover (ours)} & \textbf{7.21} & \underline{12.81} & 61.7 & 80.1 & 74.4 & \textbf{84.9} & \underline{83.9} & 83.1 & 51.6 & \textbf{77.3} & \textbf{84.5} & \underline{75.7} \\
     &  & \multicolumn{12}{l}{\textit{Hybrid methods}} \\
     &  & \quad STUN+Wanda \citep{xu2025stun,sun2024wanda} & 6.35 & 14.48 & \textbf{62.5} & \textbf{84.6} & \textbf{88.8} & 78.8 & 83.1 & 79.8 & \textbf{52.6} & 74.6 & 74.8 & 75.5 \\
     &  & \quad \textbf{HodgeCover+Wanda (ours)} & \textbf{5.99} & \textbf{12.49} & 62.2 & 79.9 & 84.2 & \textbf{85.6} & \textbf{85.5} & \textbf{83.4} & 51.2 & \textbf{77.5} & \textbf{90.3} & \textbf{77.8} \\
    \cmidrule(lr){2-15}
     & \multirow{8}{*}{$66\%$} & \multicolumn{12}{l}{\textit{Expert-count reduction methods}} \\
     &  & \quad REAP \citep{muzio2025reap} & \textbf{12.24} & \underline{17.28} & \underline{50.2} & \textbf{71.3} & 77.9 & \underline{75.1} & \textbf{70.0} & \underline{79.8} & \underline{50.5} & \underline{75.0} & \textbf{80.2} & \underline{70.0} \\
     &  & \quad REAM \citep{ream2024} & 13.94 & 19.67 & 46.8 & 67.6 & \underline{83.7} & 70.5 & 65.4 & 75.5 & 48.8 & 73.6 & 30.4 & 62.5 \\
     &  & \quad MC-SMoE \citep{mcsmoe2024} & 18.21 & 22.82 & 33.6 & 50.5 & \textbf{86.0} & 62.9 & 34.6 & 72.4 & 44.0 & 72.5 & 34.0 & 54.5 \\
     &  & \quad \textbf{HodgeCover (ours)} & \underline{12.46} & \textbf{17.21} & \textbf{52.2} & \underline{71.1} & 75.3 & \textbf{77.0} & \underline{69.0} & \textbf{80.1} & \textbf{51.5} & \textbf{76.4} & \underline{78.1} & \textbf{70.1} \\
     &  & \multicolumn{12}{l}{\textit{Hybrid methods}} \\
     &  & \quad STUN+Wanda \citep{xu2025stun,sun2024wanda} & 8.77 & 17.50 & 59.6 & \textbf{84.8} & \textbf{82.1} & 71.0 & 78.3 & 76.9 & 49.6 & 73.0 & 62.3 & 70.8 \\
     &  & \quad \textbf{HodgeCover+Wanda (ours)} & \textbf{7.42} & \textbf{13.86} & \textbf{61.2} & 82.3 & 79.3 & \textbf{80.1} & \textbf{83.2} & \textbf{81.9} & \textbf{51.1} & \textbf{76.4} & \textbf{87.9} & \textbf{75.9} \\
    \bottomrule
  \end{tabular}
  }% end resizebox
  \vspace{-12.5pt}
\end{table*}

HodgeCover+Wanda has the best C4 perplexity on every
$($model, rate$)$ cell in Table~\ref{tab:main-results} and the best
WikiText perplexity on every $66\%$ cell, with WikiText margins of
$1.4$--$3.8$ points over STUN+Wanda at $66\%$ on
the three backbones; at $33\%$, STUN+Wanda is ahead on WikiText for
the two smaller backbones (OLMoE and Qwen 3.5-35B-A3B) at
the cost of $1.7$--$2.2$ C4 perplexity points. On the four Qwen $($model, rate$)$ cells the downstream
margin is large and structural: at $66\%$ on Qwen 3.5-35B,
HodgeCover+Wanda gains $+12.6$ pp DS-Avg over STUN+Wanda
($74.6\%$ vs.\ $62.0\%$), driven primarily by GSM8K
($85.5\%$ vs.\ $22.1\%$) with a smaller $8.9$ pp lead on MMLU
($78.0\%$ vs.\ $69.2\%$); the same pattern holds at $66\%$ on
Qwen 3.5-122B-A10B. On the expert-count axis, HodgeCover
matches REAP within $\pm 0.3$ pp DS-Avg on all four Qwen cells and leads on
three of four, while leading REAM and MC-SMoE by $6$--$16$ pp at $66\%$. On
OLMoE every $66\%$ expert-count cell loses $23$--$25$ pp DS-Avg and
crosses C4 perplexity 100, but HodgeCover+Wanda still recovers a
$21.2$ C4 perplexity and $48.6\%$ DS-Avg, $+3.5$ pp over STUN+W. The matched-control hybrids REAP+Wanda and REAM+Wanda
(computed by us, not reported in the source papers) trail
HodgeCover+W at $66\%$ by $0.3$--$0.8$ pp DS-Avg on the two Qwen
scales and by $5$--$8$ pp DS-Avg on OLMoE
(App.~\ref{app:full-hybrid}).

\subsection{Mechanism: how baselines deviate from HodgeCover}
\label{sub:exp-mechanism}

\begin{figure*}[!tbp]
  \centering
  \includegraphics[width=\linewidth]{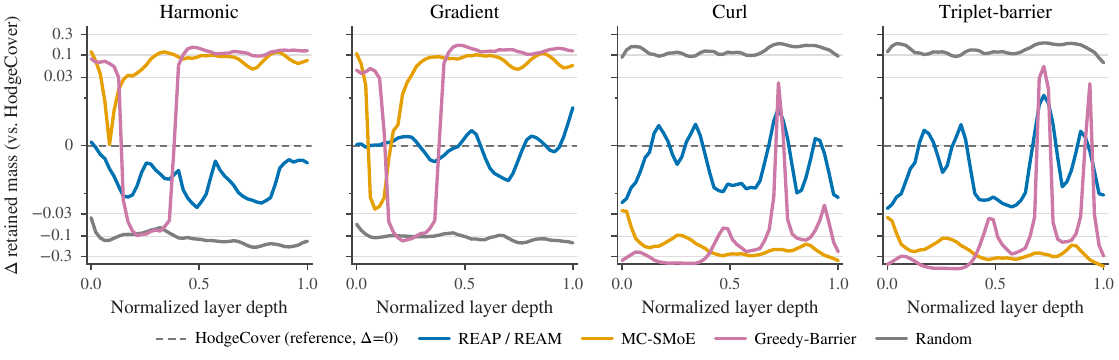}
  \caption{Per-layer deviation from HodgeCover on Qwen 3.5-122B-A10B
  at $66\%$ across the four Hodge components. Greedy-Barrier and
  MC-SMoE pay for higher harmonic and gradient retention with sharply
  lower curl and triplet-barrier mass; Random sweeps in the opposite
  direction; REAP / REAM (whose pure-axis survivor selection is
  identical) deviate weakly on every component.}
  \label{fig:mechanism}
  \vspace{-12.5pt}
\end{figure*}

For each Hodge component (harmonic, gradient, curl) we evaluate the
fraction of the original uncompressed component's $\ell^1$ mass that
survives the survivor set; for the triplet-barrier component we
evaluate the analogous fraction over triangle-supported coefficients.
The exact definition is given in Eq.~\ref{eq:retained-mass}
(App.~\ref{app:mechanism-metric}). We then report each baseline's
deviation from HodgeCover, which is the reference and so has
deviation zero by construction.
Figure~\ref{fig:mechanism} resolves the baselines into two
trade-off classes. Greedy-Barrier and MC-SMoE retain $0.05$--$0.07$
more harmonic and gradient mass at the cost of $0.10$--$0.25$ less
curl and triplet-barrier mass; Random sweeps the opposite trade,
gaining curl and triplet-barrier mass while losing harmonic and
gradient. The HodgeCover coverage objective $\Phi$ does not pick a
side: by greedy-covering the top-$p\%$ harmonic-critical edges and
top-$q_T\%$ triplet-critical triangles jointly, it stays close to the
backbone's pre-compression decomposition along all four components,
which the downstream tables reward by the $6$--$16$ pp DS-Avg gap of
HodgeCover over Greedy-Barrier and MC-SMoE on the two Qwen scales.
REAP and REAM share identical pure-axis survivor selection; their deviation from HodgeCover is
weak across all four components, consistent with their narrow
DS-Avg gap to HodgeCover in Table~\ref{tab:main-results}. Matched
plots on the two remaining backbones and macro numerics are in
App.~\ref{app:mechanism}; the full ablation construction is in
App.~\ref{app:method-detail}.

\subsection{Ablation Study}
\label{sec:ablations}

% Qwen 3.5-35B ablations table — narrow wraptable + resizebox
\begin{wraptable}{r}{0.36\linewidth}
  \centering
  \footnotesize
  \setlength{\tabcolsep}{3pt}
  \vspace{-11pt}
  \caption{Ablations on Qwen 3.5-35B-A3B. \textbf{Bold} = best in column.}
  \label{tab:ablations}
  \resizebox{\linewidth}{!}{%
  \begin{tabular}{l|cc|c}
    \toprule
    Variant & Wiki & C4 & DS-Avg \\
    \midrule
    \multicolumn{4}{l}{\emph{$33\%$ rate}} \\
    HodgeCover (ours) & \textbf{9.97} & \textbf{13.75} & \textbf{75.9} \\
    Hodge No-Triangle & 12.48 & 17.10 & 67.7 \\
    Triplet-Hypergraph & 14.29 & 16.76 & 64.5 \\
    Triplet-Penalty & 10.22 & 13.82 & 72.3 \\
    Greedy-Barrier & 10.22 & 13.79 & 72.7 \\
    \midrule
    \multicolumn{4}{l}{\emph{$66\%$ rate}} \\
    HodgeCover (ours) & \textbf{15.13} & 18.86 & \textbf{66.7} \\
    Hodge No-Triangle & 21.34 & 28.34 & 55.2 \\
    Triplet-Hypergraph & 98.08 & 154.9 & 36.2 \\
    Triplet-Penalty & 16.14 & 18.77 & 60.9 \\
    Greedy-Barrier & 15.30 & \textbf{18.13} & 60.7 \\
    \bottomrule
  \end{tabular}
  }% end resizebox
\end{wraptable}

To isolate the contribution of the Hodge decomposition we compare
HodgeCover against the four ablations of
Section~\ref{sub:exp-setup} on Qwen 3.5-35B-A3B
(Table~\ref{tab:ablations}). At $66\%$ HodgeCover gains $+5.7$ pp
DS-Avg over Triplet-Penalty at virtually identical C4 perplexity
($18.86$ vs.\ $18.77$): the gain is the result of routing the same
triplet inputs through the Hodge kernel rather than a sum penalty.
Removing triangles entirely costs $-11.5$ pp DS-Avg, and replacing the
soft objective with a binary triangle veto collapses by $-30.5$ pp,
identifying both the triangle term and its soft formulation as
necessary. The full per-task breakdown across all three backbones is
in Appendix~\ref{app:ablations-full}; Appendix~\ref{app:systems}
reports plan-time, inference throughput, and routing health: HodgeCover+Wanda matches REAP+Wanda's throughput to within $6\%$
at $66\%$, and the simplicial complex is computed once and cached
across compression rates.

\section{Conclusion}
\label{sec:conclusion}

Pairwise expert-merge scores are structurally blind to a higher-order
property of MoE mergeability: three experts can be pairwise
compatible yet collectively form an irreducible cycle. The simplicial
mergeability complex makes this cycle visible, and its Hodge
decomposition isolates it as a harmonic component that no aggregation
of pairwise scores can express. Covering the harmonic-critical edges
and triplet-critical triangles with a learning-free objective drives
the best aggressive-compression frontier across three production MoE
scales, evidence that the harmonic kernel of a learned structure can
carry content that pairwise scoring cannot reach.

\textbf{Limitations and Broader Impact.} HodgeCover stays inside the
learning-free family: recovering the last few perplexity points
typically requires a fine-tuning or knowledge-distillation step that
is orthogonal to our pipeline. Plan-time on a $35$B-parameter
MoE is dominated by barrier computation; this is a
one-shot offline cost amortized across compression rates because the
simplicial complex is cached. Our evaluation is restricted to
language MoE; multimodal and reinforcement-learning
post-trained checkpoints are deferred, although the construction is
modality-agnostic. As with other compression methods, lowering the
inference cost of large pretrained MoE models without retraining
both democratizes access and accelerates deployment of insufficiently
aligned checkpoints; the practitioner remains responsible for
the audit posture of any pretrained MoE checkpoint
they deploy.

% \begin{ack}
% Use unnumbered first level headings for the acknowledgments. All acknowledgments

% Do {\bf not} include this section in the anonymized submission, only in the final paper. You can use the \texttt{ack} environment provided in the style file to automatically hide this section in the anonymized submission.
% \end{ack}

{\small
\bibliographystyle{plainnat}
\bibliography{main}

@article{shazeer2017,
  title={Outrageously large neural networks: The sparsely-gated mixture-of-experts layer},
  author={Shazeer, Noam and Mirhoseini, Azalia and Maziarz, Krzysztof and Davis, Andy and Le, Quoc and Hinton, Geoffrey and Dean, Jeff},
  journal={arXiv preprint arXiv:1701.06538},
  year={2017}
}

@article{fedus2022switch,
  title={Switch transformers: Scaling to trillion parameter models with simple and efficient sparsity},
  author={Fedus, William and Zoph, Barret and Shazeer, Noam},
  journal={Journal of Machine Learning Research},
  volume={23},
  number={120},
  pages={1--39},
  year={2022}
}

@article{jiang2024mixtral,
  title={Mixtral of experts},
  author={Jiang, Albert Q and Sablayrolles, Alexandre and Roux, Antoine and Mensch, Arthur and Savary, Blanche and Bamford, Chris and Chaplot, Devendra Singh and Casas, Diego de las and Hanna, Emma Bou and Bressand, Florian and others},
  journal={arXiv preprint arXiv:2401.04088},
  year={2024}
}

@article{muennighoff2024olmoe,
  title={Olmoe: Open mixture-of-experts language models},
  author={Muennighoff, Niklas and Soldaini, Luca and Groeneveld, Dirk and Lo, Kyle and Morrison, Jacob and Min, Sewon and Shi, Weijia and Walsh, Pete and Tafjord, Oyvind and Lambert, Nathan and others},
  journal={arXiv preprint arXiv:2409.02060},
  year={2024}
}

@misc{qwen3.5,
  title        = {{Qwen3.5}: Towards Native Multimodal Agents},
  author       = {{Qwen Team}},
  year         = {2026},
  month        = feb,
  howpublished = {\url{https://qwen.ai/blog?id=qwen3.5}},
  note         = {Alibaba Cloud / Qwen team blog post}
}

@article{deepseekv3-2025,
  title={Deepseek-v3 technical report},
  author={Liu, Aixin and Feng, Bei and Xue, Bing and Wang, Bingxuan and Wu, Bochao and Lu, Chengda and Zhao, Chenggang and Deng, Chengqi and Zhang, Chenyu and Ruan, Chong and others},
  journal={arXiv preprint arXiv:2412.19437},
  year={2024}
}

@article{muzio2025reap,
  title={Reap the experts: Why pruning prevails for one-shot moe compression},
  author={Lasby, Mike and Lazarevich, Ivan and Sinnadurai, Nish and Lie, Sean and Ioannou, Yani and Thangarasa, Vithursan},
  journal={arXiv preprint arXiv:2510.13999},
  year={2025}
}

@article{ream2024,
  title={REAM: Merging Improves Pruning of Experts in LLMs},
  author={Jha, Saurav and Hashemzadeh, Maryam and Pasand, Ali Saheb and Parviz, Ali and Lee, Min-Joong and Knyazev, Boris},
  journal={arXiv preprint arXiv:2604.04356},
  year={2026}
}

@article{mcsmoe2024,
  title={Merge, then compress: Demystify efficient smoe with hints from its routing policy},
  author={Li, Pingzhi and Zhang, Zhenyu and Yadav, Prateek and Sung, Yi-Lin and Cheng, Yu and Bansal, Mohit and Chen, Tianlong},
  journal={arXiv preprint arXiv:2310.01334},
  year={2023}
}

@inproceedings{xu2025stun,
  title={Stun: Structured-then-unstructured pruning for scalable moe pruning},
  author={Lee, Jaeseong and Hwang, Seung-won and Qiao, Aurick and Campos, Daniel F and Yao, Zhewei and He, Yuxiong},
  booktitle={Proceedings of the 63rd Annual Meeting of the Association for Computational Linguistics (Volume 1: Long Papers)},
  pages={13660--13676},
  year={2025}
}

@article{sun2024wanda,
  title={A simple and effective pruning approach for large language models},
  author={Sun, Mingjie and Liu, Zhuang and Bair, Anna and Kolter, J Zico},
  journal={arXiv preprint arXiv:2306.11695},
  year={2023}
}

@article{zhou2007hypergraph,
  title={Learning with hypergraphs: Clustering, classification, and embedding},
  author={Zhou, Dengyong and Huang, Jiayuan and Sch{\"o}lkopf, Bernhard},
  journal={Advances in neural information processing systems},
  volume={19},
  year={2006}
}

@inproceedings{schroff2015triplet,
  title={Facenet: A unified embedding for face recognition and clustering},
  author={Schroff, Florian and Kalenichenko, Dmitry and Philbin, James},
  booktitle={Proceedings of the IEEE conference on computer vision and pattern recognition},
  pages={815--823},
  year={2015}
}

@article{lim2020hodge,
  title={Hodge Laplacians on graphs},
  author={Lim, Lek-Heng},
  journal={Siam Review},
  volume={62},
  number={3},
  pages={685--715},
  year={2020},
  publisher={SIAM}
}

@article{schaub2020simplicial,
  title={Random walks on simplicial complexes and the normalized Hodge 1-Laplacian},
  author={Schaub, Michael T and Benson, Austin R and Horn, Paul and Lippner, Gabor and Jadbabaie, Ali},
  journal={SIAM Review},
  volume={62},
  number={2},
  pages={353--391},
  year={2020},
  publisher={SIAM}
}

@article{bronstein2021geometric,
  title={Geometric deep learning: Grids, groups, graphs, geodesics, and gauges},
  author={Bronstein, Michael M and Bruna, Joan and Cohen, Taco and Veli{\v{c}}kovi{\'c}, Petar},
  journal={arXiv preprint arXiv:2104.13478},
  year={2021}
}

@article{nemhauser1978submodular,
  title={An analysis of approximations for maximizing submodular set functions—I},
  author={Nemhauser, George L and Wolsey, Laurence A and Fisher, Marshall L},
  journal={Mathematical programming},
  volume={14},
  number={1},
  pages={265--294},
  year={1978},
  publisher={Springer}
}

@inproceedings{lu2024notall,
  title={Not all experts are equal: Efficient expert pruning and skipping for mixture-of-experts large language models},
  author={Lu, Xudong and Liu, Qi and Xu, Yuhui and Zhou, Aojun and Huang, Siyuan and Zhang, Bo and Yan, Junchi and Li, Hongsheng},
  booktitle={Proceedings of the 62nd Annual Meeting of the Association for Computational Linguistics (Volume 1: Long Papers)},
  pages={6159--6172},
  year={2024}
}

@article{liu2024eep,
  title={Efficient expert pruning for sparse mixture-of-experts language models: Enhancing performance and reducing inference costs},
  author={Liu, Enshu and Zhu, Junyi and Lin, Zinan and Ning, Xuefei and Blaschko, Matthew B and Yan, Shengen and Dai, Guohao and Yang, Huazhong and Wang, Yu},
  journal={arXiv preprint arXiv:2407.00945},
  year={2024}
}

@article{moepruner2024,
  title={Moe-pruner: Pruning mixture-of-experts large language model using the hints from its router},
  author={Xie, Yanyue and Zhang, Zhi and Zhou, Ding and Xie, Cong and Song, Ziang and Liu, Xin and Wang, Yanzhi and Lin, Xue and Xu, An},
  journal={arXiv preprint arXiv:2410.12013},
  year={2024}
}

@article{evoesap2026,
  title={EvoESAP: Non-Uniform Expert Pruning for Sparse MoE},
  author={Liu, Zongfang and Tang, Shengkun and Sun, Boyang and Shen, Zhiqiang and Yuan, Xin},
  journal={arXiv preprint arXiv:2603.06003},
  year={2026}
}

@article{aimer2026,
  title={AIMER: Calibration-Free Task-Agnostic MoE Pruning},
  author={Liu, Zongfang and Tang, Shengkun and Shen, Yifan and Wang, Huan and Yuan, Xin},
  journal={arXiv preprint arXiv:2603.18492},
  year={2026}
}

@article{chen2024hcsmoe,
  title={Retraining-free merging of sparse moe via hierarchical clustering},
  author={Chen, I and Liu, Hsu-Shen and Sun, Wei-Fang and Chao, Chen-Hao and Hsu, Yen-Chang and Lee, Chun-Yi and others},
  journal={arXiv preprint arXiv:2410.08589},
  year={2024}
}

@inproceedings{yang2024moei2,
  title={Moe-i2: Compressing mixture of experts models through inter-expert pruning and intra-expert low-rank decomposition},
  author={Yang, Cheng and Sui, Yang and Xiao, Jinqi and Huang, Lingyi and Gong, Yu and Duan, Yuanlin and Jia, Wenqi and Yin, Miao and Cheng, Yu and Yuan, Bo},
  booktitle={Findings of the Association for Computational Linguistics: EMNLP 2024},
  pages={10456--10466},
  year={2024}
}

@article{he2024merging,
  title={Towards efficient mixture of experts: A holistic study of compression techniques},
  author={He, Shwai and Dong, Daize and Ding, Liang and Li, Ang},
  journal={arXiv preprint arXiv:2406.02500},
  year={2024}
}

@article{li2022branch,
  title={Branch-train-merge: Embarrassingly parallel training of expert language models},
  author={Li, Margaret and Gururangan, Suchin and Dettmers, Tim and Lewis, Mike and Althoff, Tim and Smith, Noah A and Zettlemoyer, Luke},
  journal={arXiv preprint arXiv:2208.03306},
  year={2022}
}

@article{komatsuzaki2023sparse,
  title={Sparse upcycling: Training mixture-of-experts from dense checkpoints},
  author={Komatsuzaki, Aran and Puigcerver, Joan and Lee-Thorp, James and Ruiz, Carlos Riquelme and Mustafa, Basil and Ainslie, Joshua and Tay, Yi and Dehghani, Mostafa and Houlsby, Neil},
  journal={arXiv preprint arXiv:2212.05055},
  year={2022}
}

@inproceedings{frantar2023sparsegpt,
  title={Sparsegpt: Massive language models can be accurately pruned in one-shot},
  author={Frantar, Elias and Alistarh, Dan},
  booktitle={International conference on machine learning},
  pages={10323--10337},
  year={2023},
  organization={PMLR}
}

@article{han2016deepcompression,
  title={Deep compression: Compressing deep neural networks with pruning, trained quantization and huffman coding},
  author={Han, Song and Mao, Huizi and Dally, William J},
  journal={arXiv preprint arXiv:1510.00149},
  year={2015}
}

@article{sanh2020movement,
  title={Movement pruning: Adaptive sparsity by fine-tuning},
  author={Sanh, Victor and Wolf, Thomas and Rush, Alexander},
  journal={Advances in neural information processing systems},
  volume={33},
  pages={20378--20389},
  year={2020}
}

@article{ma2023llmpruner,
  title={Llm-pruner: On the structural pruning of large language models},
  author={Ma, Xinyin and Fang, Gongfan and Wang, Xinchao},
  journal={Advances in neural information processing systems},
  volume={36},
  pages={21702--21720},
  year={2023}
}

@inproceedings{men2024shortgpt,
  title={Shortgpt: Layers in large language models are more redundant than you expect},
  author={Men, Xin and Xu, Mingyu and Zhang, Qingyu and Yuan, Qianhao and Wang, Bingning and Lin, Hongyu and Lu, Yaojie and Han, Xianpei and Chen, Weipeng},
  booktitle={Findings of the Association for Computational Linguistics: ACL 2025},
  pages={20192--20204},
  year={2025}
}

@article{lin2024awq,
  title={Awq: Activation-aware weight quantization for on-device llm compression and acceleration},
  author={Lin, Ji and Tang, Jiaming and Tang, Haotian and Yang, Shang and Chen, Wei-Ming and Wang, Wei-Chen and Xiao, Guangxuan and Dang, Xingyu and Gan, Chuang and Han, Song},
  journal={Proceedings of machine learning and systems},
  volume={6},
  pages={87--100},
  year={2024}
}

@article{frantar2023gptq,
  title={Gptq: Accurate post-training quantization for generative pre-trained transformers},
  author={Frantar, Elias and Ashkboos, Saleh and Hoefler, Torsten and Alistarh, Dan},
  journal={arXiv preprint arXiv:2210.17323},
  year={2022}
}

@inproceedings{xiao2023smoothquant,
  title={Smoothquant: Accurate and efficient post-training quantization for large language models},
  author={Xiao, Guangxuan and Lin, Ji and Seznec, Mickael and Wu, Hao and Demouth, Julien and Han, Song},
  booktitle={International conference on machine learning},
  pages={38087--38099},
  year={2023},
  organization={PMLR}
}

@article{shao2024omniquant,
  title={Omniquant: Omnidirectionally calibrated quantization for large language models},
  author={Shao, Wenqi and Chen, Mengzhao and Zhang, Zhaoyang and Xu, Peng and Zhao, Lirui and Li, Zhiqian and Zhang, Kaipeng and Gao, Peng and Qiao, Yu and Luo, Ping},
  journal={arXiv preprint arXiv:2308.13137},
  year={2023}
}

@article{eckmann1944,
  title={Harmonische funktionen und randwertaufgaben in einem komplex},
  author={Eckmann, Beno},
  journal={Commentarii Mathematici Helvetici},
  volume={17},
  number={1},
  pages={240--255},
  year={1944},
  publisher={Springer}
}

@inproceedings{bodnar2021weisfeiler,
  title={Weisfeiler and lehman go topological: Message passing simplicial networks},
  author={Bodnar, Cristian and Frasca, Fabrizio and Wang, Yuguang and Otter, Nina and Montufar, Guido F and Lio, Pietro and Bronstein, Michael},
  booktitle={International conference on machine learning},
  pages={1026--1037},
  year={2021},
  organization={PMLR}
}

@article{horn2022topology,
  title={Topological graph neural networks},
  author={Horn, Max and De Brouwer, Edward and Moor, Michael and Moreau, Yves and Rieck, Bastian and Borgwardt, Karsten},
  journal={arXiv preprint arXiv:2102.07835},
  year={2021}
}

@article{naitzat2020topology,
  title={Topology of deep neural networks},
  author={Naitzat, Gregory and Zhitnikov, Andrey and Lim, Lek-Heng},
  journal={Journal of Machine Learning Research},
  volume={21},
  number={184},
  pages={1--40},
  year={2020}
}

@article{rieck2019neural,
  title={Neural persistence: A complexity measure for deep neural networks using algebraic topology},
  author={Rieck, Bastian and Togninalli, Matteo and Bock, Christian and Moor, Michael and Horn, Max and Gumbsch, Thomas and Borgwardt, Karsten},
  journal={arXiv preprint arXiv:1812.09764},
  year={2018}
}

@article{spielman2011spectral,
  title={Spectral sparsification of graphs},
  author={Spielman, Daniel A and Teng, Shang-Hua},
  journal={SIAM Journal on Computing},
  volume={40},
  number={4},
  pages={981--1025},
  year={2011},
  publisher={SIAM}
}

@inproceedings{grande2024simplicial,
  title={Disentangling the spectral properties of the hodge laplacian: not all small eigenvalues are equal},
  author={Grande, Vincent P and Schaub, Michael T},
  booktitle={ICASSP 2024-2024 IEEE International Conference on Acoustics, Speech and Signal Processing (ICASSP)},
  pages={9896--9900},
  year={2024},
  organization={IEEE}
}

@article{suzuki2020spectral,
  title={Spectral pruning: Compressing deep neural networks via spectral analysis and its generalization error},
  author={Suzuki, Taiji and Abe, Hiroshi and Murata, Tomoya and Horiuchi, Shingo and Ito, Kotaro and Wachi, Tokuma and Hirai, So and Yukishima, Masatoshi and Nishimura, Tomoaki},
  journal={arXiv preprint arXiv:1808.08558},
  year={2018}
}

@article{ebli2020simplicial,
  title={Simplicial neural networks},
  author={Ebli, Stefania and Defferrard, Micha{\"e}l and Spreemann, Gard},
  journal={arXiv preprint arXiv:2010.03633},
  year={2020}
}

@inproceedings{agarwal2006higher,
  title={Higher order learning with graphs},
  author={Agarwal, Sameer and Branson, Kristin and Belongie, Serge},
  booktitle={Proceedings of the 23rd international conference on Machine learning},
  pages={17--24},
  year={2006}
}

@inproceedings{louis2015hypergraph,
  title={Hypergraph markov operators, eigenvalues and approximation algorithms},
  author={Louis, Anand},
  booktitle={Proceedings of the forty-seventh annual ACM symposium on Theory of computing},
  pages={713--722},
  year={2015}
}

@article{lee2014multiway,
  title={Multiway spectral partitioning and higher-order cheeger inequalities},
  author={Lee, James R and Gharan, Shayan Oveis and Trevisan, Luca},
  journal={Journal of the ACM (JACM)},
  volume={61},
  number={6},
  pages={1--30},
  year={2014},
  publisher={ACM New York, NY, USA}
}

@article{benson2016higher,
  title={Higher-order organization of complex networks},
  author={Benson, Austin R and Gleich, David F and Leskovec, Jure},
  journal={Science},
  volume={353},
  number={6295},
  pages={163--166},
  year={2016},
  publisher={American Association for the Advancement of Science}
}

@inproceedings{feng2019hypergraph,
  title={Hypergraph neural networks},
  author={Feng, Yifan and You, Haoxuan and Zhang, Zizhao and Ji, Rongrong and Gao, Yue},
  booktitle={Proceedings of the AAAI conference on artificial intelligence},
  volume={33},
  number={01},
  pages={3558--3565},
  year={2019}
}

@article{bai2021hypergraph,
  title={Hypergraph convolution and hypergraph attention},
  author={Bai, Song and Zhang, Feihu and Torr, Philip HS},
  journal={Pattern Recognition},
  volume={110},
  pages={107637},
  year={2021},
  publisher={Elsevier}
}

@inproceedings{wang2014triplet,
  title={Learning fine-grained image similarity with deep ranking},
  author={Wang, Jiang and Song, Yang and Leung, Thomas and Rosenberg, Chuck and Wang, Jingbin and Philbin, James and Chen, Bo and Wu, Ying},
  booktitle={Proceedings of the IEEE conference on computer vision and pattern recognition},
  pages={1386--1393},
  year={2014}
}

@inproceedings{hadsell2006dimensionality,
  title={Dimensionality reduction by learning an invariant mapping},
  author={Hadsell, Raia and Chopra, Sumit and LeCun, Yann},
  booktitle={2006 IEEE computer society conference on computer vision and pattern recognition (CVPR'06)},
  volume={2},
  pages={1735--1742},
  year={2006},
  organization={IEEE}
}

@book{hatcher2002,
  author    = {Hatcher, Allen},
  title     = {Algebraic Topology},
  publisher = {Cambridge University Press},
  address   = {Cambridge, UK},
  year      = {2002},
  isbn      = {978-0-521-79540-1}
}

@book{munkres1984,
  title={Elements Of Algebraic Topology},
  author={Munkres, J.R.},
  isbn={9780201627282},
  lccn={84006250},
  year={1996},
  publisher={Avalon Publishing}
}

@article{horak2013spectra,
  title={Spectra of combinatorial Laplace operators on simplicial complexes},
  author={Horak, Danijela and Jost, J{\"u}rgen},
  journal={Advances in Mathematics},
  volume={244},
  pages={303--336},
  year={2013},
  publisher={Elsevier}
}

@phdthesis{goldberg2002combinatorial,
  title={Combinatorial Laplacians of simplicial complexes},
  author={Goldberg, Timothy E},
  year={2002},
  school={Bard College}
}

@inproceedings{friedman1998computing,
  title={Computing Betti numbers via combinatorial Laplacians},
  author={Friedman, Joel},
  booktitle={Proceedings of the twenty-eighth annual ACM symposium on Theory of Computing},
  pages={386--391},
  year={1996}
}

@article{raffel2020c4,
  title={Exploring the limits of transfer learning with a unified text-to-text transformer},
  author={Raffel, Colin and Shazeer, Noam and Roberts, Adam and Lee, Katherine and Narang, Sharan and Matena, Michael and Zhou, Yanqi and Li, Wei and Liu, Peter J},
  journal={Journal of machine learning research},
  volume={21},
  number={140},
  pages={1--67},
  year={2020}
}

@article{merity2017wikitext,
  title={Pointer sentinel mixture models},
  author={Merity, Stephen and Xiong, Caiming and Bradbury, James and Socher, Richard},
  journal={arXiv preprint arXiv:1609.07843},
  year={2016}
}

@article{gao2024lmeval,
  title={A framework for few-shot language model evaluation},
  author={Gao, Leo and Tow, Jonathan and Biderman, Stella and Black, Sid and DiPofi, Anthony and Foster, Charles and Golding, Laurence and Hsu, Jeffrey and McDonell, Kyle and Muennighoff, Niklas and others},
  journal={Zenodo},
  year={2021}
}

@article{clark2018arc,
  title={Think you have solved question answering? try arc, the ai2 reasoning challenge},
  author={Clark, Peter and Cowhey, Isaac and Etzioni, Oren and Khot, Tushar and Sabharwal, Ashish and Schoenick, Carissa and Tafjord, Oyvind},
  journal={arXiv preprint arXiv:1803.05457},
  year={2018}
}

@inproceedings{clark2019boolq,
  title={Boolq: Exploring the surprising difficulty of natural yes/no questions},
  author={Clark, Christopher and Lee, Kenton and Chang, Ming-Wei and Kwiatkowski, Tom and Collins, Michael and Toutanova, Kristina},
  booktitle={Proceedings of the 2019 conference of the north American chapter of the association for computational linguistics: Human language technologies, volume 1 (long and short papers)},
  pages={2924--2936},
  year={2019}
}

@inproceedings{zellers2019hellaswag,
  title={Hellaswag: Can a machine really finish your sentence?},
  author={Zellers, Rowan and Holtzman, Ari and Bisk, Yonatan and Farhadi, Ali and Choi, Yejin},
  booktitle={Proceedings of the 57th annual meeting of the association for computational linguistics},
  pages={4791--4800},
  year={2019}
}

@article{hendrycks2021mmlu,
  title={Measuring massive multitask language understanding},
  author={Hendrycks, Dan and Burns, Collin and Basart, Steven and Zou, Andy and Mazeika, Mantas and Song, Dawn and Steinhardt, Jacob},
  journal={arXiv preprint arXiv:2009.03300},
  year={2020}
}

@inproceedings{bisk2020piqa,
  title={Piqa: Reasoning about physical commonsense in natural language},
  author={Bisk, Yonatan and Zellers, Rowan and Gao, Jianfeng and Choi, Yejin and others},
  booktitle={Proceedings of the AAAI conference on artificial intelligence},
  volume={34},
  number={05},
  pages={7432--7439},
  year={2020}
}

@inproceedings{lin2022truthfulqa,
  title={Truthfulqa: Measuring how models mimic human falsehoods},
  author={Lin, Stephanie and Hilton, Jacob and Evans, Owain},
  booktitle={Proceedings of the 60th annual meeting of the association for computational linguistics (volume 1: long papers)},
  pages={3214--3252},
  year={2022}
}

@article{sakaguchi2021winogrande,
  title={Winogrande: An adversarial winograd schema challenge at scale},
  author={Sakaguchi, Keisuke and Bras, Ronan Le and Bhagavatula, Chandra and Choi, Yejin},
  journal={Communications of the ACM},
  volume={64},
  number={9},
  pages={99--106},
  year={2021},
  publisher={ACM New York, NY, USA}
}

@article{cobbe2021gsm8k,
  title={Training verifiers to solve math word problems},
  author={Cobbe, Karl and Kosaraju, Vineet and Bavarian, Mohammad and Chen, Mark and Jun, Heewoo and Kaiser, Lukasz and Plappert, Matthias and Tworek, Jerry and Hilton, Jacob and Nakano, Reiichiro and others},
  journal={arXiv preprint arXiv:2110.14168},
  year={2021}
}

@article{yin2023outlier,
  title={Outlier weighed layerwise sparsity (owl): A missing secret sauce for pruning llms to high sparsity},
  author={Yin, Lu and Wu, You and Zhang, Zhenyu and Hsieh, Cheng-Yu and Wang, Yaqing and Jia, Yiling and Li, Gen and Jaiswal, Ajay and Pechenizkiy, Mykola and Liang, Yi and others},
  journal={arXiv preprint arXiv:2310.05175},
  year={2023}
}

@inproceedings{zheng2026hsd,
  title={Topology-Preserving Neural Operator Learning via Hodge Decomposition},
  author={Zheng, Dongzhe and Zhong, Tao and Allen-Blanchette, Christine},
  booktitle={International conference on machine learning},
  year={2026},
  organization={PMLR}
}

@article{zhong2022meta,
  title={Meta-dmoe: Adapting to domain shift by meta-distillation from mixture-of-experts},
  author={Zhong, Tao and Chi, Zhixiang and Gu, Li and Wang, Yang and Yu, Yuanhao and Tang, Jin},
  journal={Advances in Neural Information Processing Systems},
  volume={35},
  pages={22243--22257},
  year={2022}
}
}

%%%%%%%%%%%%%%%%%%%%%%%%%%%%%%%%%%%%%%%%%%%%%%%%%%%%%%%%%%%%

\appendix

\section{Hodge Decomposition Primer and Implementation Details}
\label{app:hodge}

This appendix carries the full background and proofs deferred from
Section~\ref{sec:complex}. App.~\ref{app:complexes} reviews the abstract
simplicial complex; App.~\ref{app:boundary} defines the boundary
operators and proves the chain identity; App.~\ref{app:laplacian}
constructs the combinatorial $1$-Hodge Laplacian, characterizes its
kernel, and reports the Euler-Poincar\'{e} count for $\beta_1(K)$;
App.~\ref{app:hodge-proof} states and proves the discrete Hodge
decomposition; App.~\ref{app:projection} derives explicit projection
formulas via the Moore-Penrose pseudoinverse;
App.~\ref{app:thresholding} documents the effect of edge thresholding on
the spectrum and motivates working with the complete edge set;
App.~\ref{app:triangle-set} analyzes the sensitivity of
$b_{\mathrm{harm}}$ to the curated triangle set $T$;
App.~\ref{app:merge} defines the frequency-weighted merge operation and
its all-zero-frequency guard; App.~\ref{app:calibration} fixes the
calibration corpus; App.~\ref{app:extended-diagnostic} reports the
per-layer ranges, gradient/curl companion curves, and the representative
discordance margin used in Figure~\ref{fig:diagnostic}.

We follow the textbook treatments of \citet{hatcher2002, munkres1984} for
classical homology, and \citet{eckmann1944, lim2020hodge,
schaub2020simplicial, horak2013spectra,goldberg2002combinatorial,
friedman1998computing,zheng2026hsd} for the discrete Hodge framework.

\subsection{Abstract simplicial complexes}
\label{app:complexes}

An \emph{abstract simplicial complex} on a finite vertex set
$V = \{1, \ldots, n\}$ is a family $K \subseteq 2^V$ of non-empty subsets
that is closed under taking non-empty subsets:
$\sigma \in K$ and $\emptyset \neq \tau \subseteq \sigma$ implies
$\tau \in K$. Elements $\sigma \in K$ with $|\sigma| = q + 1$ are
\emph{$q$-simplices}; the set of $q$-simplices is denoted $F_q(K)$. We work
exclusively with $q \in \{0, 1, 2\}$: vertices, edges, triangles. Throughout
this paper $K = (V, E, T)$ with $V = F_0(K) = \{1, \ldots, n\}$, $E = F_1(K)$
the complete pairwise edge set $\binom{V}{2}$, and $T = F_2(K) \subseteq
\binom{V}{3}$ a curated triangle set whose construction is given in
App.~\ref{app:triangle-set}.

To do linear algebra on $K$ we orient every simplex. Fix any total ordering
$<$ on $V$. For each simplex $\{v_0, \ldots, v_q\} \in F_q(K)$ with
$v_0 < v_1 < \cdots < v_q$, we choose the canonical
lexicographically-ordered representative $[v_0, \ldots, v_q]$ as its
orientation; an odd permutation of these indices represents the opposite
orientation. Equivalently, an \emph{oriented $q$-simplex} is the
equivalence class of an ordered $(q+1)$-tuple of distinct vertices modulo
even permutations.

The space of \emph{$q$-chains} is the real vector space
\begin{equation*}
  C_q(K) \;=\; \mathbb{R}^{F_q(K)},
\end{equation*}
with the standard basis indexed by oriented $q$-simplices. For $q = 0, 1, 2$
the dimensions are $\dim C_0 = |V| = n$, $\dim C_1 = |E|$, and
$\dim C_2 = |T|$. We endow each $C_q(K)$ with the Euclidean inner product in
the lexicographic basis. This is the unweighted choice; weighted variants
are discussed in Remark~\ref{rem:weighted-laplacian}.

\subsection{Boundary operators and the chain identity}
\label{app:boundary}

The boundary of an oriented simplex is the formal alternating sum of its
oriented faces. Concretely,
\begin{align*}
  \partial_1 \colon C_1(K) &\to C_0(K),
    & \partial_1\,[i, j] &= [j] - [i],
    & i < j, \\
  \partial_2 \colon C_2(K) &\to C_1(K),
    & \partial_2\,[i, j, k] &= [j, k] - [i, k] + [i, j],
    & i < j < k.
\end{align*}
Extending linearly, $\partial_1$ is represented in the canonical bases by
the signed vertex-edge incidence matrix $B_1 \in \{-1, 0, +1\}^{|V| \times |E|}$
of the underlying graph; $\partial_2$ is represented by the signed
edge-triangle incidence matrix $B_2 \in \{-1, 0, +1\}^{|E| \times |T|}$.

\begin{lemma}[Chain identity]
\label{lem:chain-identity}
$\partial_1 \circ \partial_2 = 0$ as a linear map $C_2(K) \to C_0(K)$.
Equivalently, $\mathrm{im}(\partial_2) \subseteq \ker(\partial_1)$.
\end{lemma}

\begin{proof}
By linearity it suffices to check the identity on a single basis element.
Fix an oriented triangle $[i, j, k]$ with $i < j < k$. Then
\begin{align*}
  \partial_1\,\partial_2\,[i, j, k]
    &\;=\;\partial_1\,\big([j, k] - [i, k] + [i, j]\big)\\
    &\;=\;\big([k] - [j]\big)\;-\;\big([k] - [i]\big)\;+\;\big([j] - [i]\big)\\
    &\;=\;[k] - [j] - [k] + [i] + [j] - [i]\\
    &\;=\;0,
\end{align*}
where every term cancels with its mate of opposite sign.
\end{proof}

The chain identity is the discrete counterpart of $d \circ d = 0$ for the
exterior derivative, and it is the single algebraic fact that makes the
Hodge decomposition possible.

\begin{remark}[Why this is a 2-complex and not a hypergraph]
\label{rem:not-hypergraph}
A hypergraph treats each triangle $\{i, j, k\}$ as a single hyperedge with
no internal structure. The simplicial complex $K$ retains the boundary
relation between an oriented triangle and its three oriented edges, which is
exactly the structure that distinguishes ``three pairwise compatibilities''
from ``a coherent triangle of pairwise compatibilities.'' This boundary
relation is what makes $\partial_2 \partial_2^\top$ a non-trivial operator
on $C_1$; without it, the curl term $b_{\mathrm{curl}}$ vanishes by
construction and the decomposition collapses to the graph Laplacian's
kernel-orthogonal split.
\end{remark}

\subsection{The combinatorial 1-Hodge Laplacian and its kernel}
\label{app:laplacian}

The \emph{combinatorial 1-Hodge Laplacian} of $K$ is the symmetric
positive-semi-definite operator
\begin{equation}
  \label{eq:laplacian-app}
  L_1 \;=\; \partial_1^\top \partial_1 \;+\; \partial_2 \partial_2^\top
    \;=\; L_1^{\mathrm{down}} + L_1^{\mathrm{up}}
    \;\in\; \mathbb{R}^{|E|\times|E|}.
\end{equation}
Here $L_1^{\mathrm{down}} = \partial_1^\top \partial_1$ couples each edge to
its endpoint-sharing neighbors (a ``lower'' adjacency), and
$L_1^{\mathrm{up}} = \partial_2 \partial_2^\top$ couples each edge to its
co-faces in $T$ (an ``upper'' adjacency). The down-Laplacian
$L_1^{\mathrm{down}}$ is the standard graph Laplacian's ``edge form,'' while
$L_1^{\mathrm{up}}$ is the new piece that depends on the chosen triangle
set $T$.

The 1-Hodge Laplacian satisfies $L_1 = L_1^\top$ and $\langle L_1 b, b\rangle
= \norm{\partial_1 b}^2 + \norm{\partial_2^\top b}^2 \ge 0$, hence is
positive semi-definite. Equation~\ref{eq:laplacian-app} also implies
\begin{equation}
  \label{eq:kernel-app}
  \ker(L_1)
    \;=\;
    \ker(\partial_1) \,\cap\, \ker(\partial_2^\top),
\end{equation}
since $L_1 b = 0$ holds if and only if $\langle L_1 b, b\rangle = 0$, which
forces both squared norms in the line above to vanish. The dimension of
this kernel is the first Betti number of $K$:
\begin{equation}
  \label{eq:betti}
  \dim \ker(L_1) \;=\; \beta_1(K)
    \;=\; \dim \ker(\partial_1) - \dim \mathrm{im}(\partial_2),
\end{equation}
which by the rank-nullity theorem applied to $\partial_1$ and the
Euler-Poincar\'{e} formula for a 2-complex specializes to
\begin{equation}
  \label{eq:euler}
  \beta_1(K)
    \;=\; |E| - |V| + \beta_0\!\left(K^{(1)}\right) - \mathrm{rank}\,\partial_2,
\end{equation}
where $\beta_0(K^{(1)}) = |\pi_0(V, E)|$ is the number of connected
components of the 1-skeleton $K^{(1)} = (V, E)$. For
the complete edge set $E = \binom{V}{2}$ used in this paper, $K^{(1)}$ is
the complete graph $K_n$, so $\beta_0(K^{(1)}) = 1$ and
$\beta_1(K) = \binom{n}{2} - n + 1 - \mathrm{rank}\,\partial_2$.

We instantiate this on the three production MoE families used throughout
the paper. The per-layer candidate triangle set $T \subseteq \binom{V}{3}$
is the one constructed by Stage~A of App.~\ref{app:triangle-set}, capped
at $|T| \le 500$; on every layer of all three models the
median-pairwise-barrier subgraph supplies more than $500$ qualifying
3-cliques, so the cap binds and the kept triples are a uniform-random
subsample with fixed seed $42$.
For OLMoE ($n = 64$, $\binom{64}{2} = 2{,}016$), the empirical
$\beta_1(K)$ takes values in $\{1453, 1454, 1455\}$ across the $16$ layers,
which by Eq.~\ref{eq:euler} corresponds to $\mathrm{rank}\,\partial_2
\in \{498, 499, 500\}$ (i.e., the curl subspace is full-rank or nearly
so on every layer). For Qwen3.5-35B and Qwen3.5-122B
($n = 256$, $\binom{256}{2} = 32{,}640$), $\beta_1(K) = 31{,}885$ at every
layer, so $\mathrm{rank}\,\partial_2 = 500$ and the curl subspace is
exactly full-rank on every layer. In every (model, layer) cell,
$\beta_1(K)$ is structurally pinned by $n$ and $|T|$ via
Eq.~\ref{eq:euler} (with at most a $\pm 2$ fluctuation across OLMoE's
$16$ layers and exact constancy across both Qwen variants' $40$ and $48$
layers), so it functions as a model-level descriptor rather than a
per-layer signal; this is why Section~\ref{sub:diagnostic} demotes
$\beta_1$ from the per-layer analysis and uses
$\rho_{\mathrm{harm}}(\ell)$ instead.

\begin{remark}[Weighted Hodge Laplacian]
\label{rem:weighted-laplacian}
A weighted variant replaces the unweighted boundary operators with
$\widetilde{\partial}_1 = \partial_1\,W_E^{1/2}$ and
$\widetilde{\partial}_2 = W_E^{-1/2}\,\partial_2\,W_T^{1/2}$ for diagonal
weight matrices $W_E$ and $W_T$ \citep{horak2013spectra}, giving a weighted
Laplacian $\widetilde{L}_1 = \widetilde{\partial}_1^\top \widetilde{\partial}_1
+ \widetilde{\partial}_2 \widetilde{\partial}_2^\top$. The kernel of
$\widetilde{L}_1$ is isomorphic to that of the unweighted $L_1$ as long as
all weights are strictly positive (the kernel is a topological invariant
\citep{eckmann1944, lim2020hodge}), but the eigenvectors and the projector
$P_{\mathrm{harm}}$ change. We deliberately use the unweighted form: the
edge weights $b_{ij}$ vary by orders of magnitude across layers and across
model families, and folding them into the operator would entangle
``which edges are harmonic'' (a topological question) with ``how barriers
are normalized'' (a domain-specific question). The signal lives outside the
operator.
\end{remark}

\subsection{The discrete Hodge decomposition}
\label{app:hodge-proof}

We restate Theorem~\ref{thm:hodge} from the main body and give a
self-contained linear-algebraic proof.

\begin{theorem}[Discrete Hodge decomposition]
\label{thm:hodge-app}
Let $K = (V, E, T)$ be a finite simplicial 2-complex with boundary operators
$\partial_1, \partial_2$ as in App.~\ref{app:boundary}. Then $C_1(K)$
admits the orthogonal direct-sum decomposition
\begin{equation}
  \label{eq:hodge-app}
  C_1(K)
    \;=\;
    \mathrm{im}(\partial_1^\top)
      \,\oplus\, \mathrm{im}(\partial_2)
      \,\oplus\, \ker(L_1),
\end{equation}
and consequently every $b \in C_1(K)$ has a unique decomposition
$b = b_{\mathrm{grad}} + b_{\mathrm{curl}} + b_{\mathrm{harm}}$ with
components in the three respective subspaces and pairwise orthogonal in the
canonical inner product.
\end{theorem}

\begin{proof}
We prove three orthogonality statements and a dimension count.

\textbf{Step 1 (orthogonality of $\mathrm{im}(\partial_1^\top)$ and
$\mathrm{im}(\partial_2)$).} Take $u = \partial_1^\top \alpha \in
\mathrm{im}(\partial_1^\top)$ for some $\alpha \in C_0(K)$, and
$v = \partial_2 \beta \in \mathrm{im}(\partial_2)$ for some $\beta \in
C_2(K)$. Then
\begin{equation*}
  \langle u, v\rangle
    \;=\; \langle \partial_1^\top \alpha, \partial_2 \beta\rangle
    \;=\; \langle \alpha, \partial_1\partial_2\beta\rangle
    \;=\; \langle \alpha, 0\rangle
    \;=\; 0,
\end{equation*}
where the third equality is Lemma~\ref{lem:chain-identity}. Hence the two
images are orthogonal.

\textbf{Step 2 (orthogonality of both images to $\ker(L_1)$).} If
$h \in \ker(L_1)$, by Eq.~\ref{eq:kernel-app} we have $\partial_1 h = 0$ and
$\partial_2^\top h = 0$. So for any $\alpha \in C_0$ and any $\beta \in C_2$,
$\langle \partial_1^\top \alpha, h\rangle = \langle\alpha, \partial_1 h\rangle = 0$
and
$\langle \partial_2 \beta, h\rangle = \langle\beta, \partial_2^\top h\rangle = 0$.
Hence $\ker(L_1)$ is orthogonal to both images.

\textbf{Step 3 (the three subspaces span $C_1(K)$).} The orthogonal complement
of $\mathrm{im}(\partial_1^\top)$ in $C_1$ is $\ker(\partial_1)$ (a standard
identity in linear algebra:
$\mathrm{im}(A^\top)^\perp = \ker(A)$ for any matrix $A$). The orthogonal
complement of $\mathrm{im}(\partial_2)$ inside $\ker(\partial_1)$ is then
$\ker(\partial_1) \cap \mathrm{im}(\partial_2)^\perp$, which by another
application of the same identity equals $\ker(\partial_1) \cap
\ker(\partial_2^\top)$. By Eq.~\ref{eq:kernel-app}, this intersection is
exactly $\ker(L_1)$. Therefore
\begin{equation*}
  C_1(K)
    \;=\; \mathrm{im}(\partial_1^\top)
        \,\oplus\, \big(\ker(\partial_1)\big)
    \;=\; \mathrm{im}(\partial_1^\top)
        \,\oplus\, \big(\mathrm{im}(\partial_2) \,\oplus\, \ker(L_1)\big),
\end{equation*}
using the chain identity $\mathrm{im}(\partial_2) \subseteq
\ker(\partial_1)$ from Lemma~\ref{lem:chain-identity}, and the
orthogonal-complement identities above. Combined with Steps 1 and 2, this
gives the orthogonal direct sum~Eq.~\ref{eq:hodge-app}.

The decomposition $b = b_{\mathrm{grad}} + b_{\mathrm{curl}} +
b_{\mathrm{harm}}$ is then unique because each summand is the orthogonal
projection of $b$ onto the corresponding subspace.
\end{proof}

\subsection{Explicit projection formulas}
\label{app:projection}

The orthogonal projectors onto the three subspaces have closed-form
expressions in terms of the boundary operators. Let $A^{+}$ denote the
Moore-Penrose pseudoinverse of $A$.

\begin{proposition}[Hodge projectors]
\label{prop:projectors}
The orthogonal projectors $P_{\mathrm{grad}}$, $P_{\mathrm{curl}}$,
$P_{\mathrm{harm}} : C_1(K) \to C_1(K)$ onto
$\mathrm{im}(\partial_1^\top)$, $\mathrm{im}(\partial_2)$, and $\ker(L_1)$
respectively are
\begin{align}
  \label{eq:proj-grad}
  P_{\mathrm{grad}}
    &\;=\; \partial_1^\top \big(\partial_1\partial_1^\top\big)^{+} \partial_1
    \;=\; \partial_1^\top L_0^{+} \partial_1, \\
  \label{eq:proj-curl}
  P_{\mathrm{curl}}
    &\;=\; \partial_2 \big(\partial_2^\top\partial_2\big)^{+} \partial_2^\top
    \;=\; \partial_2 L_2^{+} \partial_2^\top, \\
  \label{eq:proj-harm}
  P_{\mathrm{harm}}
    &\;=\; I_{|E|} - P_{\mathrm{grad}} - P_{\mathrm{curl}},
\end{align}
where $L_0 = \partial_1\partial_1^\top$ is the standard graph Laplacian on
$V$ and $L_2 = \partial_2^\top \partial_2$ is the analogous operator on
$T$.
\end{proposition}

\begin{proof}
For any real matrix $A$, $A A^{+}$ is the orthogonal projector onto
$\mathrm{im}(A)$ in the Euclidean inner product, and the singular value
decomposition gives the equivalent closed form
$A A^{+} = A(A^\top A)^{+} A^\top$ (a standard consequence of the
Moore-Penrose definition; see e.g.\ \citealp{lim2020hodge}). Apply
this with $A = \partial_1^\top$ to get Eq.~\ref{eq:proj-grad}, and with
$A = \partial_2$ to get Eq.~\ref{eq:proj-curl}. Equation~\ref{eq:proj-harm}
follows from the orthogonal direct sum
$C_1(K) = \mathrm{im}(\partial_1^\top) \oplus \mathrm{im}(\partial_2) \oplus
\ker(L_1)$ and the fact that orthogonal projectors onto orthogonal subspaces
sum to the identity.
\end{proof}

In practice, computing $P_{\mathrm{harm}} b$ does not require forming any of
the three matrices $L_0^{+}$, $L_2^{+}$, $P_{\mathrm{harm}}$. We instead
solve two least-squares systems:
\begin{enumerate}
  \item find $\alpha \in C_0$ minimizing
        $\norm{\partial_1^\top \alpha - b}^2$, then set
        $b_{\mathrm{grad}} = \partial_1^\top \alpha$;
  \item find $\beta \in C_2$ minimizing
        $\norm{\partial_2 \beta - (b - b_{\mathrm{grad}})}^2$, then set
        $b_{\mathrm{curl}} = \partial_2 \beta$;
  \item set $b_{\mathrm{harm}} = b - b_{\mathrm{grad}} - b_{\mathrm{curl}}$.
\end{enumerate}
Each least-squares step is solved by a dense Moore--Penrose
pseudoinverse (\texttt{numpy.linalg.pinv}) of
$L_0 \in \mathbb{R}^{|V| \times |V|}$ for the gradient step and of
$L_2 \in \mathbb{R}^{|T| \times |T|}$ for the curl step
(Eqs.~\ref{eq:proj-grad}--\ref{eq:proj-curl}), giving a per-layer
cost of $O\bigl(|V|^3 + |T|^3 + |E|\,|V|^2\bigr)$. With
$|V| = n \le 256$ and $|T| \le 500$ this projection runs in well
under two seconds per layer at every model scale in this paper, and
is dominated by the pairwise-barrier sweep
(App.~\ref{app:complexity}).

\subsection{Harmonic energy as irreducible mergeability residual}
\label{app:harmonic-residual}

This appendix proves Proposition~\ref{prop:harmonic-residual} from
main-body Section~\ref{sub:hodge}. We first make the lower-order
explainable subspace and the irreducible residual energy explicit, then
give the proof and an interpretation paragraph.

Recall from App.~\ref{app:boundary} the boundary operators
$\partial_1 : C_1(K) \to C_0(K)$ and
$\partial_2 : C_2(K) \to C_1(K)$ satisfying
$\partial_1\partial_2 = 0$ (Lemma~\ref{lem:chain-identity}), and from
Theorem~\ref{thm:hodge-app} the orthogonal direct-sum decomposition
$C_1(K) = \mathrm{im}(\partial_1^\top) \oplus \mathrm{im}(\partial_2)
\oplus \ker(L_1)$ in the canonical inner product. Define the
\emph{lower-order explainable subspace}
\begin{equation}
  \label{eq:M-subspace-app}
  \mathcal{M}_K
    \;:=\;
    \mathrm{im}(\partial_1^\top)\,\oplus\,\mathrm{im}(\partial_2)
    \;\subseteq\;
    C_1(K),
\end{equation}
which by the same theorem is closed in $C_1(K)$. Elements of
$\mathcal{M}_K$ are exactly the edge-barrier signals expressible as
$m = \partial_1^\top \phi + \partial_2 \psi$ for some vertex potential
$\phi \in C_0(K)$ and triangle potential $\psi \in C_2(K)$: the first
term is the lift to edges of a per-vertex \emph{is expert $i$ broadly
mergeable} score, the second term is a triangle-boundary correction
that re-distributes barrier mass coherently around faces of $K$. The
\emph{irreducible residual energy} of $b \in C_1(K)$ relative to
$\mathcal{M}_K$ is
\begin{equation}
  \label{eq:I-residual-app}
  \mathcal{I}_K(b)
    \;:=\;
    \inf_{\phi \in C_0(K),\,\psi \in C_2(K)}
      \norm{b - \partial_1^\top\phi - \partial_2\psi}^2
    \;=\;
    \inf_{m \in \mathcal{M}_K}\,\norm{b - m}^2,
\end{equation}
i.e.\ the smallest squared error of any vertex-potential plus
triangle-boundary explanation of $b$. Proposition~\ref{prop:harmonic-residual}
states $\mathcal{I}_K(b) = \norm{b_{\mathrm{harm}}}^2$ with unique
minimizer $m^\star = b_{\mathrm{grad}} + b_{\mathrm{curl}}$ as an
element of $\mathcal{M}_K$ (the representing potentials
$\phi, \psi$ in Eq.~\ref{eq:I-residual-app} need not be unique because
$\partial_1^\top$ and $\partial_2$ may have non-trivial kernels).

\begin{proof}[Proof of Proposition~\ref{prop:harmonic-residual}]
By Theorem~\ref{thm:hodge-app}, $\mathcal{M}_K$ is a closed subspace of
$C_1(K)$ with orthogonal complement $\mathcal{M}_K^\perp = \ker(L_1)$.
The unique decomposition $b = b_{\mathrm{grad}} + b_{\mathrm{curl}} +
b_{\mathrm{harm}}$ has $b_{\mathrm{grad}} + b_{\mathrm{curl}} \in
\mathcal{M}_K$ and $b_{\mathrm{harm}} \in \mathcal{M}_K^\perp$. For any
$m \in \mathcal{M}_K$, write
\begin{equation*}
  b - m
    \;=\;
    \big(b_{\mathrm{grad}} + b_{\mathrm{curl}} - m\big)
    \;+\; b_{\mathrm{harm}},
\end{equation*}
where the first summand lies in $\mathcal{M}_K$ and the second in
$\mathcal{M}_K^\perp$. By Pythagoras,
\begin{equation*}
  \norm{b - m}^2
    \;=\;
    \norm{b_{\mathrm{grad}} + b_{\mathrm{curl}} - m}^2
    \;+\;
    \norm{b_{\mathrm{harm}}}^2.
\end{equation*}
The first summand is non-negative and equals zero exactly at
$m = b_{\mathrm{grad}} + b_{\mathrm{curl}}$. Therefore
$\inf_{m \in \mathcal{M}_K}\,\norm{b - m}^2 = \norm{b_{\mathrm{harm}}}^2$,
attained uniquely at $m^\star = b_{\mathrm{grad}} + b_{\mathrm{curl}}$.
\end{proof}

\textbf{Interpretation.} The proposition shows that the harmonic component
$b_{\mathrm{harm}}$ is exactly the residue left in $b$ after removing
every barrier pattern expressible as (i) a vertex-level expert
potential, through $\mathrm{im}(\partial_1^\top)$, and (ii) a
triangle-boundary potential, through $\mathrm{im}(\partial_2)$.
Therefore the normalized harmonic energy fraction
\begin{equation*}
  \rho_{\mathrm{harm}}(\ell)
    \;=\;
    \frac{\norm{P_{\mathrm{harm}}\,b^{(\ell)}}^2}
         {\norm{b^{(\ell)}}^2}
\end{equation*}
in Eq.~\ref{eq:harmonic-frac} is the fraction of layer-$\ell$
edge-barrier energy that no such lower-order explanation can capture.

The statement is a property of the edge-barrier cochain $b$ on $K$; it
does not by itself yield an unconditional lower bound on the
calibration KL loss $\mathcal{L}(S, \pi)$ of Eq.~\ref{eq:objective},
which depends on a model of how a compression action exposes edge
barriers. App.~\ref{app:kl-corollary} introduces one such edge-exposure
corollary, used only as interpretation and not as an assumption in the
HodgeCover algorithm of Section~\ref{sec:method}.

\subsection{Conditional link to compression loss under edge-exposure linearization}
\label{app:kl-corollary}

This appendix gives one concrete sense in which the harmonic residual
of Proposition~\ref{prop:harmonic-residual} is the part of the
edge-barrier landscape that a harmonic-blind selector can miss when
predicting the calibration KL loss $\mathcal{L}(S, \pi)$ of
Eq.~\ref{eq:objective}. The construction introduces an explicit
approximation hypothesis (an edge-exposure linearization of
$\mathcal{L}$) and is included only as interpretation; the HodgeCover
algorithm of Section~\ref{sec:method} does not rely on it.

Throughout, write $a := (S, \pi)$ for a compression action and
$\mathcal{L}(a) := \mathcal{L}(S, \pi)$ for its calibration KL loss
(Eq.~\ref{eq:objective}). Set
$b_{\mathrm{lo}} := b_{\mathrm{grad}} + b_{\mathrm{curl}} \in
\mathcal{M}_K$, so the Hodge decomposition reads
$b = b_{\mathrm{lo}} + b_{\mathrm{harm}}$.

\textbf{Edge-exposure linearization.} Suppose there exists a linear
\emph{edge-exposure} vector $w_a \in C_1(K)$ and an action-dependent
baseline $\beta_a \in \mathbb{R}$ such that
\begin{equation}
  \label{eq:edge-exposure-app}
  \mathcal{L}(a)
    \;=\;
    \beta_a \;+\; \langle w_a,\, b\rangle \;+\; \xi_a,
    \qquad |\xi_a| \le \varepsilon,
\end{equation}
for some uniform approximation error $\varepsilon \ge 0$. The bilinear
form $\langle w_a,\, b\rangle$ aggregates how strongly action $a$
\emph{exposes} each pairwise merge, with $\beta_a$ absorbing the
zeroth-order action-dependent offset and $\xi_a$ the linearization
residual. The hypothesis is non-vacuous only after fixing the
baseline $\beta_a$ and the exposure map $a \mapsto w_a$ by a specified
local approximation scheme; the corollary below does not assert that
any particular such scheme is accurate, only what its harmonic-blind
prediction error must look like once one has been fixed. A harmonic-blind surrogate that sees only $b_{\mathrm{lo}}$
would predict
\begin{equation}
  \label{eq:L-lo-app}
  \mathcal{L}_{\mathrm{lo}}(a)
    \;:=\;
    \beta_a \;+\; \langle w_a,\, b_{\mathrm{lo}}\rangle.
\end{equation}
Decompose $w_a = w_{a,\mathrm{lo}} + w_{a,\mathrm{harm}}$ along
$\mathcal{M}_K \oplus \mathcal{M}_K^\perp$, with
$w_{a,\mathrm{harm}} := P_{\mathrm{harm}}\,w_a$.

\begin{corollary}[Harmonic exposure under edge-exposure linearization]
\label{cor:kl-corollary}
Suppose Eq.~\ref{eq:edge-exposure-app} holds. For every action $a$,
\begin{equation}
  \label{eq:kl-corollary-app}
  \mathcal{L}(a) - \mathcal{L}_{\mathrm{lo}}(a)
    \;=\;
    \langle w_{a,\mathrm{harm}},\, b_{\mathrm{harm}}\rangle + \xi_a,
\end{equation}
hence
\begin{equation}
  \label{eq:kl-corollary-bound-app}
  \big|\mathcal{L}(a) - \mathcal{L}_{\mathrm{lo}}(a)\big|
    \;\le\;
    \norm{w_{a,\mathrm{harm}}}\,\norm{b_{\mathrm{harm}}} + \varepsilon.
\end{equation}
For any two actions $a, a'$,
\begin{equation}
  \label{eq:kl-ordering-app}
  \big(\mathcal{L}(a) - \mathcal{L}(a')\big)
    \;-\;
    \big(\mathcal{L}_{\mathrm{lo}}(a) - \mathcal{L}_{\mathrm{lo}}(a')\big)
    \;=\;
    \big\langle P_{\mathrm{harm}}(w_a - w_{a'}),\, b_{\mathrm{harm}}\big\rangle
    + (\xi_a - \xi_{a'}),
\end{equation}
and consequently
\begin{equation}
  \label{eq:kl-ranking-bound-app}
  \Big|\big(\mathcal{L}(a) - \mathcal{L}(a')\big)
    -
    \big(\mathcal{L}_{\mathrm{lo}}(a) - \mathcal{L}_{\mathrm{lo}}(a')\big)\Big|
    \;\le\;
    \norm{P_{\mathrm{harm}}(w_a - w_{a'})}\,\norm{b_{\mathrm{harm}}}
      \;+\; 2\varepsilon.
\end{equation}
\end{corollary}

\begin{proof}
Subtracting Eq.~\ref{eq:L-lo-app} from Eq.~\ref{eq:edge-exposure-app}
gives
$\mathcal{L}(a) - \mathcal{L}_{\mathrm{lo}}(a)
  = \langle w_a,\, b - b_{\mathrm{lo}}\rangle + \xi_a
  = \langle w_a,\, b_{\mathrm{harm}}\rangle + \xi_a$.
Because $b_{\mathrm{harm}} \in \ker(L_1) = \mathcal{M}_K^\perp$ and
$P_{\mathrm{harm}}$ is the orthogonal projector onto $\ker(L_1)$,
\begin{equation*}
  \langle w_a,\, b_{\mathrm{harm}}\rangle
    \;=\;
    \langle P_{\mathrm{harm}}\,w_a,\, b_{\mathrm{harm}}\rangle
    \;+\;
    \langle (I - P_{\mathrm{harm}})\,w_a,\, b_{\mathrm{harm}}\rangle
    \;=\;
    \langle w_{a,\mathrm{harm}},\, b_{\mathrm{harm}}\rangle,
\end{equation*}
where the second term vanishes because
$(I - P_{\mathrm{harm}})\,w_a \in \mathcal{M}_K$ is orthogonal to
$b_{\mathrm{harm}} \in \mathcal{M}_K^\perp$. This proves
Eq.~\ref{eq:kl-corollary-app}; Eq.~\ref{eq:kl-corollary-bound-app} is
the Cauchy-Schwarz inequality applied to that inner product, plus
$|\xi_a| \le \varepsilon$. Equation~\ref{eq:kl-ordering-app} follows by
subtracting two instances of Eq.~\ref{eq:kl-corollary-app} and using
linearity of $P_{\mathrm{harm}}$, and Eq.~\ref{eq:kl-ranking-bound-app}
is the Cauchy-Schwarz inequality applied to the right-hand-side inner
product, plus $|\xi_a - \xi_{a'}| \le 2\varepsilon$.
\end{proof}

\textbf{Interpretation.} Under the edge-exposure linearization
Eq.~\ref{eq:edge-exposure-app}, the harmonic-blind prediction error of
$\mathcal{L}_{\mathrm{lo}}$ and its ranking error between any two
actions are controlled exactly by the harmonic exposure of $b$, with no
contribution from any $b_{\mathrm{lo}}$-aligned exposure direction
(Eqs.~\ref{eq:kl-corollary-app}, \ref{eq:kl-ordering-app}). Whenever
the right-hand side of Eq.~\ref{eq:kl-ranking-bound-app} exceeds
$|\mathcal{L}_{\mathrm{lo}}(a) - \mathcal{L}_{\mathrm{lo}}(a')|$, a
harmonic-blind selector cannot certify the true ordering of
$\mathcal{L}(a)$ and $\mathcal{L}(a')$. We re-emphasize that
HodgeCover does not assume Eq.~\ref{eq:edge-exposure-app} when
planning a compression: it covers top harmonic-critical edges
(Eq.~\ref{eq:hc-objective}) directly. The corollary is included to
make precise the role harmonic-edge coverage plays under any
locally-linear approximation of the calibration KL loss, and is not a
substantive predictive model in its own right.

\subsection{Effect of edge thresholding}
\label{app:thresholding}

A natural simplification is to keep only edges with $b_{ij} \le \tau_e$ for
some threshold $\tau_e$, giving the thresholded graph
$G_\tau = (V, E_\tau)$ with $E_\tau = \{(i, j) : b_{ij} \le \tau_e\}$. We
argue against this choice:
\begin{enumerate}
  \item $\mathrm{rank}\,\partial_1 = |V| - |\pi_0(G_\tau)|$, so the rank of
        $\partial_1$ changes whenever edge deletions disconnect $G_\tau$
        (creating a new component); deleting a non-bridge edge of $G_\tau$
        leaves $\mathrm{rank}\,\partial_1$ unchanged but decreases
        $\dim \ker(\partial_1)$ by one, while deleting a bridge does the
        reverse;
  \item the edge chain space $C_1$ shrinks, and every triangle
        $\{i, j, k\} \in T$ incident on a deleted edge is also removed,
        changing the rank of $\partial_2$ and therefore the curl subspace
        $\mathrm{im}(\partial_2)$;
  \item $\beta_1(K)$ changes by Eq.~\ref{eq:euler}, with the change
        depending non-monotonically on whether the deleted edge is a
        bridge of $G_\tau$, lies on a non-trivial cycle, or borders a
        triangle in $T$.
\end{enumerate}
The harmonic projector $P_{\mathrm{harm}}$ is not continuous in the
combinatorial topology of $K$: a small change in $\tau_e$ can detach a
harmonic cycle that previously contributed to $\ker(L_1)$ and re-attach it
as a boundary. This makes $\rho_{\mathrm{harm}}(\ell)$ depend strongly on
$\tau_e$ rather than on the underlying barrier signal $b$, defeating the
purpose of the diagnostic.

We therefore use $E = \binom{V}{2}$ (the complete edge set), with all
barriers $b_{ij}$ retained as the signal. Numerically, this works for the
MoE sizes considered: $|E| = 2{,}016$ for OLMoE and $|E| = 32{,}640$ for
both Qwen variants; the sparse least-squares solves in
App.~\ref{app:projection} run in well under a second per layer on a
single CPU core.

\subsection{Triangle-set construction and sensitivity}
\label{app:triangle-set}

The triangle set $T$ is constructed in two stages, separating
\emph{candidate sampling} (per layer, before any Hodge analysis) from
the \emph{filtration for the Hodge complex} (per Hodge call, inside
Algorithm~\ref{alg:hodgecover}).

\emph{Stage A (candidate sampling).} On every MoE layer we first compute
all pairwise barriers $\{b_{ij}\}_{\{i,j\} \in \binom{V}{2}}$ on the
calibration corpus $\mathcal{D}$ (App.~\ref{app:calibration}). We set
the candidate threshold $\tau_{\mathrm{cand}}$ to the median of the
upper-triangular entries of $\{b_{ij}\}$ and enumerate every $3$-clique
of the thresholded graph $(V, \{e : b_e \le \tau_{\mathrm{cand}}\})$, i.e.
every unordered triple whose three pairwise barriers all sit at or below
the median. If the number of qualifying triples exceeds the cap
$|T|_{\max} = 500$ we draw a uniform-random subsample of size $|T|_{\max}$
with fixed seed $42$; otherwise we keep every qualifying triple. We then
compute $b_{ijk}$ on $\mathcal{D}$ for each kept triple. The resulting
$T \subseteq \binom{V}{3}$ with $|T| \le 500$ is the layer's candidate
triangle set, fixed for all subsequent Hodge calls. The cap
$|T|_{\max} = 500$ is set at the high end of what one triplet-barrier
sweep (one short forward pass per triple, App.~\ref{app:complexity})
finishes in the same wall-clock budget as the pairwise sweep on the
largest model used here.

\emph{Stage B (filtration for the Hodge complex).} Inside
Algorithm~\ref{alg:hodgecover}, the Hodge complex is built at a single
Betti-maximizing threshold $\tau^\star$. Concretely, we sweep $\tau$ on
an $80$-point uniform grid over $[0,\, 1.1 \cdot \max_e b_e]$, and at
each $\tau$ build $K_\tau = (V, E_\tau, T_\tau)$ with
$E_\tau = \{e : b_e \le \tau\}$ and $T_\tau = \{(i,j,k) \in T :
b_{ijk} \le \tau \text{ and all three edges of } (i,j,k) \text{ lie in
} E_\tau\}$, recording $\beta_1(K_\tau)$ via Eq.~\ref{eq:euler}. We
take $\tau^\star = \arg\max_\tau \beta_1(K_\tau)$, breaking ties by the
larger $|E_\tau|$. On every layer of all three models in this paper the
$|E_\tau|$ tie-break drives $\tau^\star$ to the high end of the grid,
where $E_{\tau^\star} = \binom{V}{2}$ and $T_{\tau^\star} = T$ (the full
Stage-A candidate set); the Hodge complex used at $\tau^\star$ is
therefore exactly the complete-edge / full-candidate-$T$ complex
specified in Section~\ref{sub:complex} and motivated in
App.~\ref{app:thresholding}, and the filtration sweep itself serves
only to certify this Betti-maximizing identity and to expose the
$\tau \to 0$ regime to the sensitivity discussion below. Stage B uses
the candidate set $T$ from Stage A as a fixed input; only the subset
of $T$ that survives the Stage~B filtration ever enters $\partial_2$
and the curl subspace.

\textbf{Sensitivity to $|T|$.} The harmonic component $b_{\mathrm{harm}}$
depends on the candidate set $T$ from Stage A through the curl subspace
$\mathrm{im}(\partial_2)$, which has rank
$\mathrm{rank}\,\partial_2 \le \min(|E|, |T|)$. Three regimes are useful:
\begin{enumerate}
  \item $T = \emptyset$ (no triangles): $\mathrm{im}(\partial_2) = \{0\}$
        and $\ker(L_1) = \ker(\partial_1)$, the cycle space of the graph
        $(V, E)$. The harmonic component then carries the entire
        non-gradient signal, and the decomposition reduces to the standard
        graph-Laplacian split into gradient and rotational components.
  \item $T = \binom{V}{3}$ (all triangles): for a complete simplicial
        2-skeleton on $V$ with $E = \binom{V}{2}$, $\beta_1$ vanishes
        \citep{hatcher2002}, so the harmonic component vanishes
        identically. At Qwen scale this would mean $|T| = \binom{256}{3}
        \approx 2.7 \times 10^6$ triangle-barrier evaluations per layer,
        which is computationally infeasible under the learning-free
        constraint.
  \item $T$ a curated subset of size $|T| = m_T \ll \binom{n}{3}$: the
        harmonic subspace has intermediate dimension, and \emph{which}
        triples enter $T$ is the central modeling choice. Stage~A above
        fixes that choice (3-cliques of the median-pairwise-barrier
        subgraph, uniform-random-subsampled at the cap with fixed seed)
        and App.~\ref{app:laplacian} verifies that $|T| = 500$ holds
        across every layer of all three models for cross-layer and
        cross-model comparability.
\end{enumerate}
The cap $|T|_{\max} = 500$ is a compromise between regimes 1 and 2: it
keeps $|T| / \binom{n}{3}$ small enough to remain learning-free yet
large enough that the kept triples concentrate on the dense, mutually
low-barrier part of the barrier graph where joint merging is non-trivial
to predict from any single pairwise barrier.

\subsection{Frequency-weighted merge and the all-zero-frequency guard}
\label{app:merge}

Let $\mathcal{R}(x') \subseteq \{1, \ldots, n\}$ denote the set of experts
selected by the router on token $x'$ (top-routing in all models studied
here, with the model-specific fan-out taken from each backbone's
released configuration; we reserve $k$ for the survivor count and use
no symbol for the router fan-out). The \emph{empirical token-routing frequency} of expert $i$ on the
calibration corpus is the corpus-level constant
\begin{equation*}
  r_i \;=\; \mathbb{E}_{x' \sim \mathcal{D}}\,
    \mathbf{1}\!\left[\,i \in \mathcal{R}(x')\,\right]
  \;=\; \frac{1}{|\mathcal{D}|}\,
    \big|\big\{x' \in \mathcal{D} : i \in \mathcal{R}(x')\big\}\big|,
\end{equation*}
which depends only on $\mathcal{D}$ and not on the evaluation input
$x$. The frequency-weighted pair merge of two experts $f_i, f_j$ is then
\begin{equation}
  \label{eq:merge-pair}
  (i \oplus j)(x)
    \;=\;
    \begin{cases}
      \dfrac{r_i\,f_i(x) + r_j\,f_j(x)}{r_i + r_j},
        & r_i + r_j > 0, \\[0.8em]
      \dfrac{1}{2}\bigl(f_i(x) + f_j(x)\bigr),
        & r_i + r_j = 0,
    \end{cases}
\end{equation}
an exact convex combination on the support of the routing measure with a
fallback to the unweighted average on the zero-measure exception. The
implementation tests $r_i + r_j < \varepsilon$ with $\varepsilon = 10^{-12}$
in place of the exact-zero predicate to guard against floating-point
underflow. An analogous formula holds for the triplet merge
$(i \oplus j \oplus k)$.

The all-zero-frequency case $r_i + r_j = 0$ arises sporadically on
high-expert-count Qwen layers (where some experts are routed by no
calibration token), and the unguarded ratio is then $0/0$. The fallback in
Eq.~\ref{eq:merge-pair} is required to keep barrier matrices NaN-free at
$n > 128$.

\subsection{Calibration corpus}
\label{app:calibration}

All barrier matrices in this paper are computed on the same calibration
corpus $\mathcal{D}$: $2{,}048$ tokens drawn from the C4 training split
\citep{raffel2020c4} with random seed $42$, fixed across all methods
(HodgeCover, REAP, REAM, MC-SMoE, STUN, Triplet-Penalty,
Triplet-Hypergraph, Greedy-Barrier). This is the same corpus used by the
published REAP and REAM baselines \citep{muzio2025reap, ream2024} for
direct comparability.

\subsection{Per-layer ranges and gradient/curl companion curves}
\label{app:extended-diagnostic}

This section reports the per-layer ranges, gradient/curl companion curves,
and the representative discordance margin referenced in
Figure~\ref{fig:diagnostic}.

The harmonic energy fraction $\rho_{\mathrm{harm}}(\ell)$
(Eq.~\ref{eq:harmonic-frac}) and the discordance fraction $\delta(\ell)$
(Eq.~\ref{eq:discordance}) take the following per-model ranges across
$16 + 40 + 48 = 104$ layers:
\begin{itemize}
  \item OLMoE: $\rho_{\mathrm{harm}} \in [0.289, 0.491]$ (mean $0.352$);
        $\delta \in [0.596, 1.000]$ (mean $0.757$, with $\delta = 1.000$
        on layer $1$, where every sampled triangle is discordant under
        the $1.2\times$ margin).
  \item Qwen3.5-35B: $\rho_{\mathrm{harm}} \in [0.387, 0.587]$
        (mean $0.503$); $\delta \in [0.114, 0.762]$ (mean $0.277$).
  \item Qwen3.5-122B: $\rho_{\mathrm{harm}} \in [0.296, 0.616]$
        (mean $0.456$); $\delta \in [0.060, 0.892]$ (mean $0.411$).
\end{itemize}
$\rho_{\mathrm{harm}}$ is depth-stable on every model, while $\delta$ is
depth-sensitive and reveals the depth at which higher-order obstructions
concentrate: the two Qwen models taper from a high value in early layers
to a low value in late layers, while OLMoE remains uniformly highly
discordant.

Figure~\ref{fig:diagnostic} reports only the harmonic and discordance
diagnostics. For completeness,
Figure~\ref{fig:diagnostic-curl-grad} reports the gradient and curl
energy fractions
\begin{equation*}
  \rho_{\mathrm{grad}}(\ell)
    \;=\;\frac{\norm{P_{\mathrm{grad}}\,b^{(\ell)}}^2}
              {\norm{b^{(\ell)}}^2},
  \qquad
  \rho_{\mathrm{curl}}(\ell)
    \;=\;\frac{\norm{P_{\mathrm{curl}}\,b^{(\ell)}}^2}
              {\norm{b^{(\ell)}}^2}
  \qquad (b^{(\ell)} \neq 0).
\end{equation*}
By Hodge orthogonality these satisfy
$\rho_{\mathrm{grad}} + \rho_{\mathrm{curl}} + \rho_{\mathrm{harm}} = 1$
at every layer where $b^{(\ell)} \neq 0$. Empirically
$\rho_{\mathrm{curl}}$ is essentially zero on the two Qwen variants
($\rho_{\mathrm{curl}} \in [6.3 \times 10^{-5}, 1.5 \times 10^{-3}]$ on
Qwen3.5-35B and $\rho_{\mathrm{curl}} \in [9.8 \times 10^{-5},
1.5 \times 10^{-3}]$ on Qwen3.5-122B at $|T| = 500$), and around $3\%$ on
OLMoE ($\rho_{\mathrm{curl}} \in [0.025, 0.049]$). The non-pairwise content
of $b$ is therefore overwhelmingly harmonic at Qwen scale and dominantly
harmonic on OLMoE; the curl term carries little additional signal at
$|T| = 500$.

\begin{figure}[!t]
  \centering
  \includegraphics[width=\linewidth]{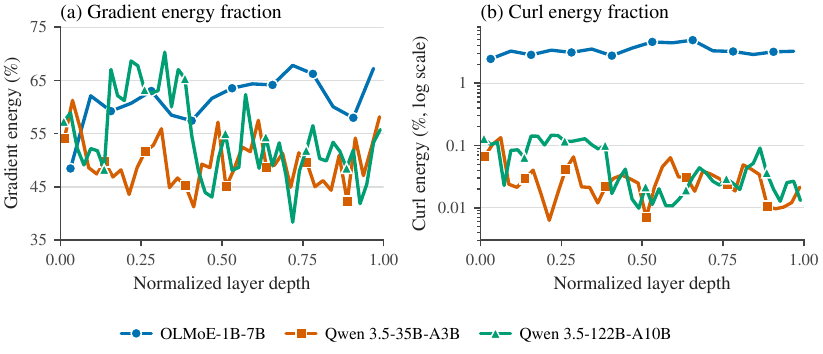}
  \caption{Gradient energy fraction $\rho_{\mathrm{grad}}(\ell)$ and curl
  energy fraction $\rho_{\mathrm{curl}}(\ell)$ at every layer of the same
  three MoE families as Figure~\ref{fig:diagnostic}. Panel (b) uses a log
  $y$-axis because OLMoE's curl fraction is roughly two orders of
  magnitude larger than the two Qwen variants'.}
  \label{fig:diagnostic-curl-grad}
\end{figure}

The discordance fraction $\delta$ in Eq.~\ref{eq:discordance} is defined
with a $1.2\times$ margin: a triangle $\{i, j, k\}$ counts as discordant
if $b_{ijk} > 1.2 \cdot \max(b_{ij}, b_{ik}, b_{jk})$. The factor $1.2$
is a single representative threshold rather than the result of a sweep;
it is the value used by the diagnostics pipeline that produced the
per-layer numbers reported in this appendix and in
Figure~\ref{fig:diagnostic}.

\section{Algorithmic Details for HodgeCover and HodgeCover+Wanda}
\label{app:algorithm}

This appendix carries the full algorithmic specification deferred from
Section~\ref{sec:method}. App.~\ref{app:hodgecover-pseudocode}
gives the per-layer HodgeCover pseudocode as Algorithm~\ref{alg:hodgecover};
App.~\ref{app:cross-layer} describes the cross-layer expert-budget
allocation that turns a global compression rate into the per-layer
counts $\{k_\ell\}_\ell$; App.~\ref{app:wanda-stage} states the Stage 2
Wanda subroutine of HodgeCover+Wanda;
App.~\ref{app:complexity} records plan-time complexity and a concrete
wall-clock measurement on Qwen3.5-35B (the full per-method, per-model
breakdown is reported in App.~\ref{app:systems});
App.~\ref{app:baseline-axis} reports the baseline-axis comparison
grid (Table~\ref{tab:baseline-axis}) referenced from
Section~\ref{sub:costs}.

\subsection{HodgeCover pseudocode}
\label{app:hodgecover-pseudocode}

Algorithm~\ref{alg:hodgecover} is the complete per-layer
specification of HodgeCover (Section~\ref{sub:hodgecover}).
The triangle set $T$ enters Algorithm~\ref{alg:hodgecover} as a
fixed input: it is the per-layer candidate set produced by Stage~A
of App.~\ref{app:triangle-set} (3-cliques of the
median-pairwise-barrier subgraph, capped at $|T| = 500$ via a
fixed-seed uniform-random subsample when more triples qualify).
The threshold $\tau^\star$ is the layer-specific Betti-maximizing
parameter of the per-Hodge-call filtration sweep (Stage~B of
App.~\ref{app:triangle-set}, executed on line~1 of
Algorithm~\ref{alg:hodgecover}); the Hodge projection at
$\tau^\star$ delivers both the harmonic edge mass
$|b_{\mathrm{harm},e}|$ used to rank $E$ in the top-$p\%$ step and
the denominator $\|b\|$ used in the Hodge-weighted redirect cost.

\begin{algorithm}[t]
\caption{HodgeCover at one MoE layer}
\label{alg:hodgecover}
\begin{algorithmic}[1]
\Require pairwise barriers $\{b_{ij}\}_{\{i,j\} \in E}$;
   triplet-barrier table $\{b_{ijk}\}_{(i,j,k) \in T}$;
   target survivor count $k$;
   non-negative saliency $\mathrm{sal} \in [0,1]^n$;
   protected-expert set $\mathrm{SE} \subseteq V$;
   hyperparameters $p, q_T, \lambda_e, \lambda_t, \alpha$
\Ensure survivor set $S^\star$ with $|S^\star| = k$, redirect map
   $\pi : V \setminus S^\star \to S^\star$
\State $\tau^\star \gets \arg\max_\tau\,\beta_1\!\big(K_\tau\big)$
   \Comment{Betti-maximizing $\tau$ for triangle set (App.~\ref{app:triangle-set})}
\State Construct $K = (V, E, T)$ with the complete edge set
   $E = \binom{V}{2}$ and triangle set $T$ at threshold $\tau^\star$
\State $b \gets$ vector of $b_{ij}$ on $E$
\State $(b_{\mathrm{grad}}, b_{\mathrm{curl}}, b_{\mathrm{harm}}) \gets$
   Hodge-decompose $b$ on $K$ (App.~\ref{app:projection})
\State $\|b\| \gets \big(\sum_{e}\,b_e^2\big)^{1/2}$
\State $E^\star \gets$ top-$p\%$ edges of $E$ ranked by
   $|b_{\mathrm{harm},e}|$
\State $T^\star \gets$ top-$q_T\%$ triangles of $T$ ranked by $|b_{ijk}|$
\State $N_E[i] \gets \{e \in E^\star : i \in e\}$ for all $i \in V$
\State $N_T[i] \gets \{\sigma \in T^\star : i \in \sigma\}$ for all
   $i \in V$
\State $S \gets \mathrm{SE} \cap V$;\quad
       $C_E \gets \bigcup_{i \in S} N_E[i]$;\quad
       $C_T \gets \bigcup_{i \in S} N_T[i]$
\While{$|S| < k$}
  \For{$i \in V \setminus S$}
    \State $\Delta(i \mid S)
       \gets \mathrm{sal}(i)
       + \lambda_e\,\dfrac{|N_E[i] \setminus C_E|}{|E^\star|}
       + \lambda_t\,\dfrac{|N_T[i] \setminus C_T|}{|T^\star|}$
  \EndFor
  \State $i^\star \gets \arg\max_{i \in V \setminus S}\,\Delta(i \mid S)$
  \State $S \gets S \cup \{i^\star\}$;\quad
         $C_E \gets C_E \cup N_E[i^\star]$;\quad
         $C_T \gets C_T \cup N_T[i^\star]$
\EndWhile
\State $S^\star \gets S$
\For{$i \in V \setminus S^\star$}
  \State $\pi(i) \gets \arg\min_{j \in S^\star}\;
     b_{ij}\big(1 + \alpha\,
       |b_{\mathrm{harm},\{i,j\}}|/\max(\|b\|, \varepsilon_0)\big)$
\EndFor
\State \Return $S^\star, \pi$
\end{algorithmic}
\end{algorithm}

A few remarks. First, the marginal-gain loop on lines~10--16
implements the standard greedy maximization of a monotone submodular
function under cardinality constraint, which guarantees a
$(1 - 1/e)$ approximation to the optimum of
Eq.~\ref{eq:hc-objective} \citep{nemhauser1978submodular}. The
implementation has $O(k\,n\,(\Delta_E + \Delta_T))$ cost in
marginal-gain evaluations, where $\Delta_E$ and $\Delta_T$ are the
maximal incidence sizes; this is dwarfed by the barrier sweep on
every layer of every model in this paper, so we adopt the plain
greedy variant rather than its lazy-incremental refinements. Second,
when
$|E^\star| = 0$ or $|T^\star| = 0$ the corresponding normalized
coverage gain is set to zero by convention; in particular, if
the harmonic component vanishes (e.g., $\beta_1(K_{\tau^\star}) = 0$),
the algorithm degenerates to a saliency-ranked greedy + triplet-only
cover. Third, the protected set $\mathrm{SE}$ supports a small number
of always-survive experts; in our reported runs $\mathrm{SE} = \emptyset$
on every layer of every model, but the option is exposed for future
extensions. Fourth, the redirect-step constant
$\varepsilon_0 = 10^{-12}$ is a numerical guard against an
all-zero edge-barrier signal; it is unrelated to the KL-linearization
error $\varepsilon$ in App.~\ref{app:kl-corollary} and to the
zero-frequency guard in App.~\ref{app:merge}.

\subsection{Proof of Proposition~\ref{prop:topological-coverage} and an inexpressibility example}
\label{app:topological-coverage-proof}

We restate Proposition~\ref{prop:topological-coverage} for
convenience: the HodgeCover objective
$\Phi(S) = \sum_{i \in S}\mathrm{sal}(i)
   + \lambda_e\,|C_E(S)|/|E^\star|
   + \lambda_t\,|C_T(S)|/|T^\star|$
is non-negative monotone submodular over $S \subseteq V$ subject to
$|S| = k$, and its greedy maximizer $S^\star$ satisfies
$\Phi(S^\star) \ge (1 - 1/e)\,\max_{|S|=k}\Phi(S)$. The
Remark~\ref{rem:inexpressibility} inexpressibility example
follows.

\textbf{Submodularity.}
$\Phi(S) = \Phi_{\mathrm{sal}}(S)
    + \lambda_e \Phi_E(S) + \lambda_t \Phi_T(S)$
with three components.
$\Phi_{\mathrm{sal}}(S) = \sum_{i \in S}\mathrm{sal}(i)$ is modular
(additive) and non-negative monotone: each $\mathrm{sal}(i) \ge 0$,
and adding an expert can only increase the sum.
$\Phi_E(S) = |\bigcup_{i \in S} N_E(i)| / |E^\star|$ is the
normalized maximum-coverage function on the universe $E^\star$ with
the incidence collection $\{N_E(i)\}_{i \in V}$. Maximum-coverage
functions are non-negative monotone submodular set functions: for any
$A \subseteq B \subseteq V$ and any $i \notin B$,
$\Phi_E(A \cup \{i\}) - \Phi_E(A)
   = |N_E(i) \setminus C_E(A)| / |E^\star|
   \ge |N_E(i) \setminus C_E(B)| / |E^\star|
   = \Phi_E(B \cup \{i\}) - \Phi_E(B)$
because $C_E(A) \subseteq C_E(B)$, which is precisely the diminishing-returns
characterization of submodularity \citep{nemhauser1978submodular}.
$\Phi_T(S) = |\bigcup_{i \in S} N_T(i)| / |T^\star|$ is non-negative
monotone submodular by the same argument applied to triangles.
A non-negatively-weighted sum of modular and submodular functions is
submodular, so $\Phi$ is non-negative monotone submodular over the
ground set $V$ with cardinality constraint $|S| = k$.

\textbf{Greedy approximation.}
For non-negative monotone submodular $\Phi$ subject to the cardinality
constraint $|S| = k$, \citet{nemhauser1978submodular} establish that
the greedy algorithm, which iteratively selects
$i^\star_j \in \arg\max_{i \in V \setminus S_{j-1}}\big(\Phi(S_{j-1} \cup \{i\}) - \Phi(S_{j-1})\big)$
starting from $S_0 = \emptyset$ until $|S_k| = k$, returns $S^\star = S_k$
satisfying
\begin{equation}
  \label{eq:nwf-bound}
  \Phi(S^\star)
    \;\ge\;
    \big(1 - (1 - 1/k)^k\big)\;\max_{|S| = k}\Phi(S)
    \;\ge\;
    \big(1 - 1/e\big)\;\max_{|S| = k}\Phi(S),
\end{equation}
using $(1 - 1/k)^k < 1/e$ for all $k \ge 1$. The greedy step in
Eq.~\ref{eq:greedy-marginal} (Algorithm~\ref{alg:hodgecover},
lines~10--16) is exactly this Nemhauser-Wolsey-Fisher greedy applied
to $\Phi$ when $\mathrm{SE} = \emptyset$ (the empty-initialization
setting used for every reported run, line~10 of
Algorithm~\ref{alg:hodgecover}), so HodgeCover inherits the bound.
This proves Eq.~\ref{eq:topological-coverage} of
Proposition~\ref{prop:topological-coverage}.

The factor $(1 - 1/e) \approx 0.632$ is tight in the worst case for
unconstrained monotone submodular maximization under cardinality
\citep{nemhauser1978submodular}; we make no claim of tightness for
the specific instances $\Phi$ produced by mergeability complexes, and
we do not exploit any tighter-than-worst-case behavior in either the
analysis or the experiments.

\textbf{Inexpressibility example
(Remark~\ref{rem:inexpressibility}).}
Define an \emph{individualistic vertex selector} as any algorithm
that assigns each expert $i \in V$ a scalar score $h(i)$ depending
\emph{only} on quantities indexed at the $i$-th vertex (per-expert
saliencies, per-vertex sums of pairwise barriers, per-vertex
sorted-pairwise-barrier statistics, etc.) and returns
$\hat S = \arg\max_{|S|=k} \sum_{i \in S} h(i)$. REAP, REAM
centroid selection, MC-SMoE, and the published STUN family all fit
this template. HodgeCover does not, because $\Phi$ depends on the
critical-edge incidence sets $\{N_E(i)\}_i$, which are joint
properties of $V$ that need not factor into a sum over vertices.

\emph{Construction.} Take $V = \{1, 2, 3, 4\}$, $k = 2$, and consider
the three problem instances $r \in \{a, b, c\}$ where the critical-edge
set $E^{\star,(r)}$ is the corresponding perfect matching of $K_4$:
\[
E^{\star,(a)} = \big\{\{1,2\},\{3,4\}\big\},\;
E^{\star,(b)} = \big\{\{1,3\},\{2,4\}\big\},\;
E^{\star,(c)} = \big\{\{1,4\},\{2,3\}\big\}.
\]
Keep $b$ uniform across all six edges of $K_4$ in every instance, so
that the per-vertex pairwise statistics
($\sum_{j \neq i} b_{ij}$, sorted neighbor-barrier vector, etc.) are
identical at every vertex and identical across the three instances.
Any individualistic selector $\hat S = \arg\max_{|S|=k}\sum_{i \in S} h(i)$
therefore returns the \emph{same} pair $\hat S \subseteq V$ on all
three instances (the score vector $h$ is the same constant on every
vertex, with tie-breaking inherent to the selector).

For any fixed pair $\hat S = \{i, j\}$ and any matching $M_r$, the
coverage $|C_E(\hat S)|/|E^{\star,(r)}|$ equals $1$ if $\hat S$ meets
\emph{both} edges of $M_r$ and $1/2$ if $\hat S$ meets only one. The
key observation is that $\hat S = \{i,j\}$ meets only one edge of
$M_r$ exactly when $\{i,j\}$ \emph{is} one of the two matching edges
of $M_r$ (in that case, $\{i,j\}$ shares both endpoints with one
matching edge and is disjoint from the other). Among the three
matchings of $K_4$, every fixed pair $\{i,j\}$ is itself one of the
matching edges in exactly one matching, so the worst-case coverage
of $\hat S$ over the three instances is exactly $1/2$. Since
$1/2 < (1 - 1/e) \approx 0.632$, no individualistic vertex selector
can match the $(1 - 1/e)$ bound of
Proposition~\ref{prop:topological-coverage} on this family.
HodgeCover, by contrast, sees $E^{\star,(r)}$ explicitly per
instance and selects a pair $\hat S$ that meets both edges of $M_r$,
achieving coverage $1$ in every instance.

The construction shows: the harmonic projection
$P_{\mathrm{harm}}\,b$ carries decision-relevant information that
shifts $E^\star$ between problem instances; \emph{any}
individualistic vertex score derived from per-vertex statistics of
$b$ alone is invariant to such shifts and therefore cannot match
HodgeCover's per-instance coverage. In practice, production MoE
layers do not exhibit pathological symmetry across instances, but
they do have non-trivial harmonic mass at every layer
(Section~\ref{sub:diagnostic}, with
$\rho_{\mathrm{harm}} \in [0.29, 0.62]$ across the three model
families in this paper), so an analogous gap exists empirically and
is what HodgeCover's coverage objective is built to exploit.
$\square$

\subsection{Cross-layer expert-budget allocation}
\label{app:cross-layer}

Algorithm~\ref{alg:hodgecover} is per-layer. Given a global compression
rate $r$ (as a fraction of the model's total expert count), we must
turn it into per-layer survivor counts $\{k_\ell\}_\ell$ across the
$L$ MoE layers. Two allocators are useful, and we describe both even
though only the first is used in the main results.

\emph{Uniform-per-layer allocator (HodgeCover and the REAP-family
baselines).} Let $r$ be the global drop rate, $n_\ell$ the number of
experts in MoE layer $\ell$, and
$R = \big\lfloor r \sum_\ell n_\ell \big\rfloor$ the total expert-drop
budget. The uniform allocator sets
\begin{equation}
  \label{eq:uniform-allocator}
  d_\ell \;=\; \big\lfloor R / L \big\rfloor + \mathbb{1}\!\big[\ell < (R \bmod L)\big],
  \qquad
  k_\ell \;=\; \max\!\big(\min\!\big(n_\ell - d_\ell,\; n_\ell - 1\big),\; 1\big),
\end{equation}
so each layer drops the same number of experts up to a one-expert
rounding remainder distributed deterministically to the lowest-index
layers, with a per-layer floor of $1$ surviving expert. This matches
the published default of REAP \citep{muzio2025reap} and REAM
\citep{ream2024}, and we adopt it as the HodgeCover convention.
MC-SMoE \citep{mcsmoe2024} and STUN+Wanda
\citep{xu2025stun,sun2024wanda} use their own native global allocators
in our experiments (a global frequency-threshold rule for MC-SMoE and
a global agglomerative-similarity rule for STUN+Wanda), faithful to
their original papers.

\emph{Compressibility-weighted allocator (alternative; not used for
HodgeCover main-table results).} The uniform allocator ignores
per-layer topology. A natural alternative weights the drop budget by
a per-layer compressibility score
$\sigma(\ell) = \max\!\big(1 - \rho_{\mathrm{harm}}(\ell), \varepsilon_0\big)$
(layers with less harmonic mass are easier to compress), allocating
\begin{equation*}
  d_\ell \;=\; \big\lfloor R \cdot
    \tfrac{\sigma(\ell)}{\sum_{\ell'} \sigma(\ell')} \big\rfloor,
\end{equation*}
followed by the same one-expert rounding-remainder pass and the
$d_\ell \le n_\ell - 1$ floor as Eq.~\ref{eq:uniform-allocator}. We
implemented this allocator and verified it produces a strictly
distinct partition of the global budget on all three model scales,
but did not adopt it for HodgeCover's main-table numbers because the
uniform allocator is exactly the per-layer budget that
\citet{muzio2025reap, ream2024} use; using a different allocator on
the HodgeCover row would conflate the survivor-selection objective
with the cross-layer budget assignment. All HodgeCover and REAP /
REAM rows in Section~\ref{sec:exp} therefore use
Eq.~\ref{eq:uniform-allocator}.

\subsection{HodgeCover+Wanda Stage 2}
\label{app:wanda-stage}

Stage 2 of HodgeCover+Wanda (Section~\ref{sub:hodgewanda}) reuses the
published row-wise unstructured Wanda pruner of \citet{sun2024wanda}
unmodified, applied independently to each surviving expert and to the
shared linear layers of the MoE block (gate, up, down projections in
the SwiGLU MLP, depending on architecture). For completeness,
Algorithm~\ref{alg:wanda-stage} states the per-matrix loop.
The single hyperparameter is the Stage-2 weight sparsity $r_2$,
chosen so that the combined expert-axis and weight-axis compression
matches the target total compression rate of the comparison cell;
we never retune the Wanda saliency or the row/column choice.

\begin{algorithm}[t]
\caption{HodgeCover+Wanda Stage 2 (per surviving MLP weight matrix)}
\label{alg:wanda-stage}
\begin{algorithmic}[1]
\Require weight matrix $W \in \mathbb{R}^{a \times b}$ of a survivor;
   calibration activations $X \in \mathbb{R}^{N \times b}$;
   target row sparsity $r_2 \in [0, 1)$
\Ensure pruned weight matrix $\widetilde W$
\State $\| X_{\cdot, j} \|_2 \gets \big(\sum_{n=1}^N X_{n,j}^2\big)^{1/2}$
   for $j = 1, \ldots, b$
\For{$i = 1, \ldots, a$}
  \State $S_{i, j} \gets |W_{i, j}| \cdot \|X_{\cdot, j}\|_2$
     for $j = 1, \ldots, b$
  \State $\mathcal{K}_i \gets$ indices of the top-$\big\lceil(1-r_2)\,b\big\rceil$
     entries of $\{S_{i, j}\}_{j=1}^b$
  \State $\widetilde W_{i, j} \gets W_{i, j}$ if $j \in \mathcal{K}_i$,
     else $\widetilde W_{i, j} \gets 0$
\EndFor
\State \Return $\widetilde W$
\end{algorithmic}
\end{algorithm}

The two stages share the calibration corpus $\mathcal{D}$
(App.~\ref{app:calibration}); Stage 2 reuses the activation tensor
$X$ produced during Stage 1 barrier computation, so no additional
forward pass is incurred.

\textbf{Numerical values of $r_1$ and $r_2$ used in our results.}
We fix the Stage-1 expert-axis drop rate at $r_1 = 0.20$ across every
$($model, rate$)$ cell, matching the STUN-Arctic Stage-1 setting that is
closest to our $64$- and $256$-expert backbones. Given this $r_1$ and a
target total compression rate $r_{\mathrm{tot}} \in \{0.33, 0.66\}$, the
Stage-2 Wanda sparsity is chosen so that the combined expert-axis and
weight-axis compression matches $r_{\mathrm{tot}}$:
\begin{equation}
  \label{eq:r2-protocol}
  r_2
    \;=\;
    \max\!\Bigl(0,\; \tfrac{r_{\mathrm{tot}} - r_1}{1 - r_1}\Bigr)
  \quad\Longrightarrow\quad
  r_2 \;\approx\; 0.1625 \;\;(r_{\mathrm{tot}} = 0.33),\;\;
  r_2 \;=\; 0.575 \;\;(r_{\mathrm{tot}} = 0.66),
\end{equation}
applied identically on every backbone. The matched-control hybrids
REAP+Wanda and REAM+Wanda in App.~\ref{app:full-hybrid} use the same
$(r_1, r_2)$ schedule for parity.

\subsection{Plan-time complexity and wall-clock cost}
\label{app:complexity}

\emph{Per-layer.} The dominant cost of HodgeCover at a single MoE
layer is the pairwise-barrier sweep: for $n$ experts, computing
$b_{ij}$ for every pair requires $\binom{n}{2}$ short forward passes
on $\mathcal{D}$. The triplet-barrier sweep adds $|T| \le 500$
further forward passes, by construction independent of $n$. The
Hodge projection on $K_{\tau^\star}$ is a sparse linear-algebra
operation on a graph with $|E| = \binom{n}{2}$ edges and $|T| \le 500$
triangles; for our scales ($n \le 256$) it costs at most a few
seconds per layer (App.~\ref{app:projection}). The greedy survivor
loop on lines~10--16 of Algorithm~\ref{alg:hodgecover} is
$O(k\,n\,(\Delta_E + \Delta_T))$ in marginal-gain evaluations and
fits in well under one second.

\emph{Per model.} Summed across all $L$ MoE layers, plan-time
scales as $L \cdot n^2$ in barrier sweeps. On the hardware used in
this paper, the dominant offline cost is the pairwise-barrier sweep
itself (one short forward pass on $\mathcal{D}$ per pair), with the
Hodge projection contributing under two seconds per layer at
$n = 256$ and at most $|T| = 500$. We measure the entire
learning-free HodgeCover pipeline (calibration corpus loading,
barrier matrix plus triplet-barrier computation at the
Betti-maximizing $\tau^\star$, Hodge projection, survivor selection,
and router-redirect surgery) on Qwen3.5-35B at approximately $480$
seconds end-to-end, against a measurement of approximately $25$
seconds for the published REAP pipeline run on the same calibration
corpus and the same hardware: a roughly $19\times$ overhead, paid
once and offline, amortized across the entire downstream evaluation.
The full per-method, per-model wall-clock breakdown, including
the corresponding measurements on OLMoE-1B-7B and Qwen3.5-122B-A10B,
is reported in App.~\ref{app:systems} alongside the secondary metrics.

\subsection{Baseline-axis comparison grid}
\label{app:baseline-axis}

Table~\ref{tab:baseline-axis} summarizes how HodgeCover differs from
the seven other learning-free MoE compressors benchmarked in this
paper, along three axes: which scoring signal each method assigns to
experts (or expert pairs); whether surviving experts have their
weights perturbed by the compression step; and plan-time complexity
per single MoE layer with $n$ experts, expert hidden dimension $d$,
calibration corpus $\mathcal{D}$, and curated triangle set $T$. All
table entries are learning-free as we report them; the only entry
whose published form requires a learning step is MC-SMoE
\citep{mcsmoe2024}, addressed in the table footnote and detailed in
App.~\ref{app:baseline-impl}. HodgeCover is the only entry that uses
the harmonic kernel of the simplicial Laplacian as its scoring
signal while keeping survivor weights bit-exact. The field values
for each external baseline have been verified against its original
published description.

\begin{table}[t]
  \centering
  \caption{Baseline-axis comparison grid. ``Survivor weights'' is
  whether the method modifies surviving experts' weights at
  compression time; ``Learning'' is whether the method requires any
  gradient-based fine-tuning, knowledge distillation, or LoRA pass
  beyond the calibration forward passes. Plan-time complexity is
  per single MoE layer; $n$ is the expert count, $d$ the expert
  hidden dimension, $|\mathcal{D}|$ the calibration token count, and
  $|T|$ the curated triangle-set size. ``Our ablation'' marks
  baselines that we built and report only in our paper's evaluation
  set; they are not external published methods.}
  \label{tab:baseline-axis}
  \footnotesize
  \setlength{\tabcolsep}{3pt}
  \resizebox{\linewidth}{!}{%
  \begin{tabular}{@{}llll@{}}
    \toprule
    Method & Signal & \makecell[l]{Survivor\\weights} &
    Plan-time per layer \\
    \midrule
    HodgeCover (ours)
      & harm.\ kernel of $L_1$
      & bit-exact
      & $O\big((n^2 + |T|)\,|\mathcal{D}|\,d^2\big)$ \\
    HodgeCover+Wanda (ours)
      & harm.\ kernel + Wanda saliency
      & Wanda-pruned
      & $O\big((n^2 + |T|)\,|\mathcal{D}|\,d^2\big)$ \\
    REAP \citep{muzio2025reap}
      & gate-weighted output norm
      & bit-exact
      & $O(n\,|\mathcal{D}|\,d^2)$\textsuperscript{$\ddagger$} \\
    REAM \citep{ream2024}
      & gate-coactivation + output sim.
      & freq-weighted merge
      & $O\big(n\,|\mathcal{D}|\,d^2 + n^2\,|\mathcal{D}|\big)$ \\
    MC-SMoE\textsuperscript{$\dagger$} \citep{mcsmoe2024}
      & utilization + logit cosine
      & freq-weighted merge
      & $O\big(|\mathcal{D}|\,d^2 + n^2\,|\mathcal{D}| + n\,d^2\big)$ \\
    STUN+Wanda \citep{xu2025stun, sun2024wanda}
      & router-row sim.\ + Wanda saliency
      & cluster-mean + Wanda
      & $O\big(|\mathcal{D}|\,d^2 + n^2\,d\big)$ \\
    Greedy-Barrier (our ablation)
      & pairwise KL barrier
      & freq-weighted merge
      & $O\big(n^2\,|\mathcal{D}|\,d^2\big)$ \\
    Triplet-Penalty soft (our ablation)
      & pairwise + triplet KL
      & freq-weighted merge
      & $O\big((n^2 + |T|)\,|\mathcal{D}|\,d^2\big)$ \\
    Triplet-Hypergraph hard (our ablation)
      & pairwise KL + binary triplet veto
      & freq-weighted merge
      & $O\big((n^2 + |T|)\,|\mathcal{D}|\,d^2\big)$ \\
    \bottomrule
  \end{tabular}%
  }% end resizebox
  \\[2pt]
  \begin{minipage}{0.96\linewidth}
  \footnotesize
  \textsuperscript{$\dagger$}MC-SMoE in its published form
  \citep{mcsmoe2024} additionally requires a $20{,}000$-step
  post-merge knowledge-distillation pass against the original SMoE
  teacher (Section~4.1 and Table~5 of \citealp{mcsmoe2024}). At our
  $2{,}048$-token calibration scale this KD pass overfits and
  regresses quality on every model in our evaluation set. We
  therefore report MC-SMoE without the KD step, matching the most
  common reduction adopted by subsequent MoE-compression literature
  (App.~\ref{app:baseline-impl}). \\
  \textsuperscript{$\ddagger$}The plan-time row for REAP~\cite{muzio2025reap} charges the
  cost of the observer in the official implementation,
  which runs every expert on every calibration token to populate the
  unified pruning + merging metrics. A saliency-only implementation
  admits a lower
  $O(|\mathcal{D}|\,d^2)$ bound; we report the observer-cost form
  here for parity with the other baselines, all of which are charged
  for the dominant cost of their official reference implementations.
  \end{minipage}
\end{table}

Two patterns are worth flagging. First, every external baseline that
perturbs survivor weights does so via a learning-free heuristic that
treats the merge as a per-pair operation
\citep{ream2024,mcsmoe2024,xu2025stun}; the higher-order
representational risk of those merges (e.g., merging an expert that
carries irreducible harmonic mass) is exactly what the
harmonic-kernel coverage objective in HodgeCover was designed to
sidestep, by keeping survivors bit-exact. Second, the plan-time gap
between methods that score experts from a single calibration sweep
(MC-SMoE, STUN+Wanda, and a saliency-only reading of REAP — leading
calibration term $|\mathcal{D}|\,d^2$) and methods that assemble a
pairwise-KL barrier matrix from $O(n^2)$ per-pair operations
(HodgeCover, Greedy-Barrier, Triplet-Penalty,
Triplet-Hypergraph — leading term $n^2\,|\mathcal{D}|\,d^2$) is a
factor of up to $n^2$. The official REAP and REAM reference
implementations additionally run every expert on every calibration
token, raising their leading term to $n\,|\mathcal{D}|\,d^2$ and
narrowing the asymptotic gap to a factor of $n$. Either gap is
substantial at $n = 256$, but the entire HodgeCover pipeline runs
offline, once, before any downstream evaluation, so the offline
plan-time cost is amortized across all inference calls thereafter.

\subsection{Baseline implementation notes}
\label{app:baseline-impl}

This subsection records implementation choices for the external
baselines in Table~\ref{tab:baseline-axis} that depart from a
default reading of the original paper. Three are load-bearing for
the headline numbers in Section~\ref{sec:exp}:

\emph{MC-SMoE without KD.} \citet{mcsmoe2024} (Section~4.1,
Table~5) report MC-SMoE with a mandatory $20{,}000$-step post-merge
knowledge-distillation pass against the original SMoE teacher;
Table~5 of that work shows substantial degradation when the KD pass
is removed in the original paper's training-time data regime. In our
single-pass learning-free regime the calibration corpus is
$2{,}048$ C4 tokens (App.~\ref{app:calibration}); we implemented
the KD module and ran it at this scale, found that it overfits
the calibration corpus and degrades held-out perplexity on
WikiText-103 and C4 validation by a margin larger than the gain it
provides on the calibration set itself, and accordingly removed it.
This matches the most common MC-SMoE reduction reported by
subsequent MoE-compression literature
\citep{muzio2025reap, ream2024, xu2025stun}, which also benchmark
MC-SMoE without the KD pass for comparability with learning-free
baselines. The headline implication is that our MC-SMoE numbers are
a learning-free re-implementation, are not directly comparable to
the published MC-SMoE numbers in the regime that paper targets, and
are reported to keep all benchmarked methods on the same
calibration footing rather than as a faithful reproduction.

\emph{REAP, REAM, and STUN+Wanda.} We use the published reference
implementations of REAP \citep{muzio2025reap}, REAM
\citep{ream2024}, and STUN \citep{xu2025stun} unmodified, except
that all three are run on the same fixed $2{,}048$-token C4
calibration corpus (App.~\ref{app:calibration}) to ensure all
methods are evaluated under identical calibration conditions, rather
than the $128$- to $1024$-sequence regimes used in their respective
evaluation tables. All other implementation details (cross-layer
expert-budget allocator, Stage-2 Wanda saliency, etc.) follow each
baseline's original implementation.

\emph{Greedy-Barrier, Triplet-Penalty soft, and Triplet-Hypergraph
hard} are not external baselines but ablations that we built
specifically to isolate components of HodgeCover; none of the three
appears in published MoE-compression literature. All three replace
HodgeCover's coverage-based survivor selection with a greedy
union-find merge sort, and aggregate the resulting merge groups via
the same frequency-weighted average as REAM and MC-SMoE
(App.~\ref{app:merge}); they differ only in how triplet-barrier
information enters the merge sort.
\emph{Greedy-Barrier} sorts all pairs $\{i, j\}$ in ascending order
of pairwise barrier $b_{ij}$ and greedily unions them via a
union-find loop that terminates when exactly $k$ components remain.
No Hodge decomposition, no harmonic weighting, and no triplet input.
It isolates the contribution of the pairwise-barrier signal alone.
\emph{Triplet-Penalty soft} reuses Greedy-Barrier's union-find loop
with the edge cost $b_{ij}$ replaced by
$b_{ij}\,(1 + \alpha_T\,\overline{p}_{ij})$, where $\overline{p}_{ij}$
is the layer-normalized mean of triplet barriers $b_{ijk}$ over all
$(i, j, k) \in T$ that contain edge $\{i, j\}$ and $\alpha_T \ge 0$
is the penalty strength. The same triplet table $T$ used by
HodgeCover is consumed, but no Hodge decomposition is performed; it
isolates the contribution of the harmonic kernel beyond a plain
sum-of-triplet-barriers softening.
\emph{Triplet-Hypergraph hard} reuses Greedy-Barrier's union-find
loop with a binary triplet veto inspired by hypergraph-cut
compression \citep{zhou2007hypergraph}: a candidate union is accepted
only if every triple of experts inside the resulting component (when
of size $\ge 3$) has $b_{ijk} \le \tau_T$, where $\tau_T$ is the
layer's $50$th-percentile triplet barrier; any union that would
create a component carrying a high-barrier triple is rejected and
the loop continues with the next-best edge. It isolates the cost of
a binary triangle constraint vs.\ the soft penalty of
Triplet-Penalty.
All three ablations therefore use the same frequency-weighted merge
aggregation as REAM/MC-SMoE; the bit-exact-survivor row of
Table~\ref{tab:baseline-axis} for HodgeCover is what isolates the
zero-and-redirect surgery from the topological-coverage objective.
The three ablations are revisited alongside the
Section~\ref{sec:ablations} ablation table.

\section{Hyperparameter and Design-Choice Sensitivity}
\label{app:hyperparameters}

This appendix records the rationale behind, and the analytical
robustness of, each scalar hyperparameter in
HodgeCover. App.~\ref{app:hp-defaults} lists the defaults and how
they were set; App.~\ref{app:hp-p-q} discusses the critical-simplex
fractions $p, q_T$; App.~\ref{app:hp-lambda} discusses the coverage
weights $\lambda_e, \lambda_t$; App.~\ref{app:hp-alpha} discusses the
harmonic-redirect strength $\alpha$;
App.~\ref{app:hp-triangle-sampler} comments on the sensitivity of
HodgeCover to the construction of the triangle set $T$; and
App.~\ref{app:hp-merge-metric} comments on the sensitivity to the
choice of merge metric (KL versus alternatives) used to populate the
barriers $b_{ij}$ and $b_{ijk}$. The hyperparameter values reported
here are the values used for every model and every reported
compression rate in this paper; no per-model tuning is performed and
no values are swept post-submission.

\subsection{Defaults and how they were set}
\label{app:hp-defaults}

The frozen defaults are
\begin{equation*}
  p \;=\; q_T \;=\; 20\%,
  \qquad
  \lambda_e \;=\; 1.0,
  \qquad
  \lambda_t \;=\; 0.5,
  \qquad
  \alpha \;=\; 3.0,
  \qquad
  |T| \;\le\; 500.
\end{equation*}
Each value was set on a single, MMLU-stratified Qwen3.5-35B layer
(layer $19$, the median-depth MoE layer) by inspecting the
HodgeCover output qualitatively on a separate $2{,}048$-token C4-train
sample disjoint from the calibration corpus $\mathcal{D}$
(App.~\ref{app:calibration}): whether $E^\star$ contained the edges
that the per-layer harmonic-magnitude histogram identified as
``heavy''; whether the survivor set spread roughly uniformly across
the harmonic-incidence-sorted ranking of experts; and whether the
redirect map $\pi$ avoided routing non-survivors through edges with
above-mean harmonic magnitude. Layer~$19$ is compressed normally and
appears in all reported main-table evaluations on Qwen3.5-35B; the
``held-out'' label refers only to the disjoint C4-train sample used
for hyperparameter inspection.

\subsection{\texorpdfstring{Critical-simplex fractions $p, q_T$}{Critical-simplex fractions p, q\_T}}
\label{app:hp-p-q}

The fractions $p$ and $q_T$ control the size of the
harmonic-critical edge set $E^\star$ and the triplet-critical
triangle set $T^\star$ that the coverage objective in
Eq.~\ref{eq:hc-objective} must cover. They have a clean problem-side
interpretation: $E^\star$ is the part of $K$ where the irreducible
higher-order obstruction concentrates, and $T^\star$ is the part
where joint pairwise merging amplifies pairwise cost. Two analytical
properties bound how much HodgeCover can move under perturbations
of $p$ and $q_T$.

\emph{Saturation at large $p, q_T$.} As $p, q_T \to 100\%$, every edge
and every sampled triangle becomes critical, so $E^\star = E$ and
$T^\star = T$ and the coverage gains in
Eq.~\ref{eq:greedy-marginal} reduce to per-expert vertex degrees
in $K$. Since $K$ has the complete pairwise edge set, every expert
has the same vertex degree in $E$, so the edge-coverage gain
collapses to a constant; the triangle gain remains informative
through the (non-uniform) triangle incidence in $T$. The objective
therefore degrades smoothly into a saliency-plus-triangle-degree
selector at the saturation boundary.

\emph{Vanishing at small $p, q_T$.} As $p, q_T \to 0\%$, the critical
sets shrink to single edges or single triangles, and the coverage
objective becomes saliency-plus-noise. HodgeCover is therefore
sensitive to $p, q_T$ being too small.

The default $p = q_T = 20\%$ sits well inside the non-saturated, non-
vanishing regime for every model in this paper: at $|E| = \binom{n}{2}$
and $n \in \{64, 256\}$, $|E^\star| \in [\,400,\,6500\,]$, far from
either boundary. We did not run a full sweep of $p, q_T$ across the
three model scales pre-submission, and accordingly we do not report
empirical sensitivity bands; the analytical robustness above is the
substantive claim about $p, q_T$.

\subsection{\texorpdfstring{Coverage weights $\lambda_e, \lambda_t$}{Coverage weights lambda\_e, lambda\_t}}
\label{app:hp-lambda}

The coverage weights $\lambda_e$ and $\lambda_t$ trade off saliency
against the two coverage gains in
Eq.~\ref{eq:greedy-marginal}. Because both coverage gains are
normalized by the size of the corresponding critical set, every
marginal-gain term lies in $[0, 1]$ for every layer of every model;
the weights $\lambda_e, \lambda_t$ have a uniform interpretation
across models. Setting $\lambda_e = \lambda_t = 0$ recovers a pure
saliency-ranked selector on the REAP saliency score
\citep{muzio2025reap}; setting $\lambda_e \gg 1, \lambda_t \gg 1$
makes saliency a tiebreaker only. The default
$\lambda_e = 1.0, \lambda_t = 0.5$ gives the edge-coverage gain
parity with saliency and the triangle-coverage gain half that
weight, reflecting that triplet barriers are noisier than pairwise
barriers (the triplet sweep produces $|T| \le 500$ samples versus
$\binom{n}{2}$ pairwise samples). A formal sensitivity sweep over
these weights at multiple model scales is left for future work.

\subsection{\texorpdfstring{Harmonic-redirect strength $\alpha$}{Harmonic-redirect strength alpha}}
\label{app:hp-alpha}

The redirect cost in Eq.~\ref{eq:hc-redirect} interpolates between a
pure pairwise-barrier nearest-survivor redirect ($\alpha = 0$) and
a strongly harmonic-penalized redirect ($\alpha \gg 1$). The default
$\alpha = 3.0$ is large enough to dominate the small differences in
$b_{ij}$ that occur for nearby survivors but small enough not to
override large pairwise-barrier gaps: at the default, two survivors
with equal $b_{ij}$ to a non-survivor break the tie according to the
harmonic component, and $\alpha = 3$ corresponds to a $4\times$
penalty on the most-harmonic edge in the layer relative to a
zero-harmonic edge (i.e., $1 + 3 \cdot 1 = 4$ versus $1$).

\subsection{Triangle-set construction}
\label{app:hp-triangle-sampler}

The sensitivity of $b_{\mathrm{harm}}$ to the choice of triangle set
$T$ is analyzed in App.~\ref{app:triangle-set}. There we contrast
three regimes: $T = \emptyset$ (pairwise-only; the curl subspace
$\mathrm{im}(\partial_2)$ vanishes and the harmonic component reduces
to the graph cycle space $\ker(\partial_1)$, which then carries the
entire non-gradient signal), $T = \binom{V}{3}$ (every triple a face;
combinatorially intractable at $n = 256$ where $|\binom{V}{3}| > 2.7
\times 10^6$), and the curated $|T| \le 500$ used here. The third regime is what
all results in this paper use; the cap of $500$ was chosen at the
high end of what the triplet-barrier sweep (one forward pass per
triple) can finish in the same wall-clock budget as the pairwise
sweep. We did not run a paired ablation across multiple values of
$|T|$ pre-submission.

\subsection{Merge metric}
\label{app:hp-merge-metric}

All barriers $b_{ij}, b_{ijk}$ in this paper are computed from the
KL divergence between the original layer's output distribution and
the merged layer's output distribution on the calibration corpus
$\mathcal{D}$ (App.~\ref{app:merge}). KL is the natural choice in
the language-modeling regime: the per-token output is a categorical
distribution over the vocabulary, and the calibration objective in
Eq.~\ref{eq:objective} is itself a KL. Other f-divergences (Jensen-
Shannon, Hellinger), or distributional metrics on activations
rather than outputs (cosine on activated-output vectors,
Wasserstein-1 on activation distributions), would all populate
$b_{ij}, b_{ijk}$ with positive scalars and could in principle be
substituted; we did not run a paired comparison pre-submission.
The chain-complex structure of $K$ and the Hodge decomposition
itself depend only on which pairs are in $E$ and which triples are
in $T$, not on the metric used to populate the edge-supported
signal $b$, so a metric swap is local to the barrier sweep.

\section{Experimental setup details}
\label{app:expsetup}

This appendix supports Section~\ref{sub:exp-setup}.

\subsection{Models and checkpoints}
\label{app:exp-models}

We benchmark three open-weight sparse MoE backbones, all pulled from
their canonical releases without fine-tuning.
OLMoE-1B-7B (\verb|allenai/OLMoE-1B-7B-0924|, 16 MoE layers,
$64$ experts per layer, top-$2$ routing) instantiates the
small-model regime. Qwen 3.5-35B-A3B (40 MoE layers, $256$ experts per
layer, top-$8$ routing) and Qwen 3.5-122B-A10B (48 MoE layers, $256$
experts per layer, top-$8$ routing) instantiate the high-cardinality
regime; we use the base checkpoints with reinforcement-learning
post-training disabled to keep the calibration distribution and the
baseline router both well-defined. Tokenizer, vocabulary, and
positional encoding follow each backbone's default configuration.

\subsection{Compression rates and the per-layer allocator}
\label{app:exp-rates}

Compression rates are quoted as fractions of total expert count
across the model: $33\%$ removes
$\lfloor 0.33 \cdot \sum_\ell n_\ell \rfloor$ experts in total and
$66\%$ removes $\lfloor 0.66 \cdot \sum_\ell n_\ell \rfloor$. The
per-layer survivor counts $\{k_\ell\}$ that realise this total are
produced by each baseline's source-paper allocator
(App.~\ref{app:method-detail}). We chose
$33\%$ and $66\%$ rather
than $25\% / 50\% / 75\%$ because the two rates straddle the
"easy / hard" frontier reported by the prior MoE-compression literature
(REAP, REAM, and STUN+W all report $50\%$-rate numbers in the
modest-loss regime; $66\%$ is past the regime in which the published
hybrids retain their downstream profile).

\subsection{Calibration and evaluation}
\label{app:exp-calib}

Calibration uses $2{,}048$ tokens drawn from the C4 train split
\citep{raffel2020c4} at random seed $42$, fixed across all methods.
Tokens are concatenated to the model's training context length and
truncated to fit the calibration budget; we did not subsample by
document, since the relevant signal for the per-layer barriers is
distributional rather than per-document. Evaluation reports
WikiText-103 perplexity \citep{merity2017wikitext}, C4 perplexity
(held-out validation split, never overlapping with calibration), and
nine downstream tasks via the LM Evaluation Harness
\citep{gao2024lmeval} at version \verb|v0.4.x|: ARC-Challenge and
ARC-Easy \citep{clark2018arc} (zero-shot, accuracy-norm), BoolQ
\citep{clark2019boolq} (zero-shot), HellaSwag
\citep{zellers2019hellaswag} (zero-shot, accuracy-norm), MMLU
\citep{hendrycks2021mmlu} ($5$-shot, macro across the $57$ subjects),
PIQA \citep{bisk2020piqa} (zero-shot), TruthfulQA-MC2
\citep{lin2022truthfulqa} (zero-shot), WinoGrande
\citep{sakaguchi2021winogrande} (zero-shot), and GSM8K
\citep{cobbe2021gsm8k} ($8$-shot, exact-match strict). DS-Avg in this
paper is the unweighted macro of the nine downstream accuracies.

\subsection{Hardware and reproducibility}
\label{app:exp-hardware}

Calibration, plan-time, and topology audits run on a single workstation
with $2$ NVIDIA RTX PRO 6000 Blackwell GPUs ($98$\,GB each, $252$\,GB
RAM, $64$ CPU cores) for the primary runs; the OLMoE Round-N
benchmarks were run on a paired A6000 workstation, and the Qwen 3.5
ablations and topology audits were distributed across H200 nodes for
parallel throughput. Inference throughput is measured with PyTorch
$2.x$ and Hugging Face \verb|transformers| at the model's default
context length, batch size $1$, $128$-token decode, repeated $5$
times with the median reported.

\section{Additional Results}
\label{app:additional-results}

This appendix supports Section~\ref{sec:exp}. It is organized as
five subsections: a description of the per-method cross-layer
allocator and the four HodgeCover ablations
(App.~\ref{app:method-detail}), the full hybrid axis with
matched-control hybrids and a per-rate Pareto plot
(App.~\ref{app:full-hybrid}), the per-layer retained mass
trajectories for the harmonic / gradient / curl / triplet-barrier
components on the two backbones not shown in the main body
(App.~\ref{app:mechanism}), the per-task ablation breakdown across
all three backbones (App.~\ref{app:ablations-full}), and the
plan-time, throughput, routing entropy, and dead-expert ratio
(App.~\ref{app:systems}).

\subsection{Per-method allocator and ablation construction}
\label{app:method-detail}

\textbf{Per-method baseline implementation.} For each baseline we
follow the source paper's released implementation; the brief notes
below summarize the survivor-selection rule used in our experiments.

\emph{REAP} \citep{muzio2025reap} ranks experts within each MoE
layer by an output-saliency score and drops the lowest-ranking ones,
applying the uniform per-layer cut of
Eq.~\ref{eq:uniform-allocator}.

\emph{REAM} \citep{ream2024} reduces every MoE layer to a fixed
target survivor count by saliency-guided clustering of experts and
merging non-centroid members into their nearest centroid.

\emph{MC-SMoE} \citep{mcsmoe2024} uses normalized routing
frequencies to mark a global pool of low-frequency experts as
non-dominant subject to a per-layer "at-least-one-dominant" floor,
then merges each non-dominant expert into the most router-similar
dominant expert in its layer.

\emph{STUN+Wanda} \citep{xu2025stun,sun2024wanda} composes the
behavioral-similarity Stage-1 of STUN with a Stage-2 unstructured
Wanda \citep{sun2024wanda} sweep on survivor weights, following the
released STUN implementation; the Stage-2 protocol is in
App.~\ref{app:wanda-stage}.

\emph{HodgeCover} (ours) follows the REAP convention with the
uniform per-layer allocator of Eq.~\ref{eq:uniform-allocator}.

\textbf{The four ablations of HodgeCover.} The four ablations probe
two orthogonal levers of the HodgeCover construction: whether the
triangle term contributes beyond the harmonic-edge term (\emph{Hodge
No-Triangle}), and whether the Hodge decomposition is load-bearing
relative to topology-free or plain-triplet alternatives at the same
input data (\emph{Greedy-Barrier}, \emph{Triplet-Penalty soft},
\emph{Triplet-Hypergraph hard}). Hodge No-Triangle reuses
Algorithm~\ref{alg:hodgecover}'s backbone unchanged and only modifies
the coverage objective; the other three ablations replace survivor
selection with a greedy union-find merge sort and aggregate the
resulting merge groups via the same frequency-weighted average as
REAM and MC-SMoE (App.~\ref{app:merge}).

\emph{Hodge No-Triangle.} Set $\lambda_t = 0$ in
Eq.~\ref{eq:hc-objective}, dropping the triplet-triangle coverage
term entirely and reducing $\Phi$ to saliency plus harmonic-edge
coverage. The Step~$1$ harmonic-critical edge set $E^\star$ is still
extracted from the Hodge decomposition; only the triangle term is
removed. This isolates the contribution of the triangle term holding
the harmonic-edge term fixed.

\emph{Greedy-Barrier (no topology).} Sort all pairs $\{i, j\}$ in
ascending order of pairwise barrier $b_{ij}$ and greedily union them
via a union-find loop that terminates when exactly $k$ components
remain. No Hodge decomposition, no harmonic weighting, and no
triplet input. This isolates the contribution of all topology
against the cleanest topology-blind merge.

\emph{Triplet-Penalty soft (no Hodge decomposition).} Reuse
Greedy-Barrier's union-find loop with the edge cost $b_{ij}$
replaced by $b_{ij}\,(1 + \alpha_T\,\overline{p}_{ij})$, where
$\overline{p}_{ij}$ is the layer-normalized mean of triplet barriers
$b_{ijk}$ over all $(i, j, k) \in T$ that contain edge $\{i, j\}$
and $\alpha_T \ge 0$ is the penalty strength. The same triplet table
$T$ that HodgeCover consumes is consumed here, but the Hodge
decomposition is skipped; this isolates the contribution of the
harmonic kernel above a plain sum-of-triplet-barriers softening.

\emph{Triplet-Hypergraph hard (binary triangle veto).} Reuse
Greedy-Barrier's union-find loop with a hard triplet veto inspired
by hypergraph-cut compression \citep{zhou2007hypergraph}: a candidate
union is accepted only if every triple of experts inside the
resulting component (when of size $\ge 3$) has $b_{ijk} \le \tau_T$,
where $\tau_T$ is the layer's $50$th-percentile triplet barrier; any
union that would create a component carrying a high-barrier triple
is rejected and the loop continues with the next-best edge. This is
the hard-veto analogue of Triplet-Penalty soft and isolates the cost
of a binary triangle constraint vs.\ a soft penalty.

The four ablations cover a $2 \times 2$ grid: triangle term
on/off and Hodge decomposition on/off. HodgeCover is on/on,
Hodge No-Triangle is off/on, Triplet-Penalty soft is on/off, and
Greedy-Barrier is off/off. Triplet-Hypergraph hard is the
hard-constraint variant of on/off.

\subsection{Full hybrid axis with matched-control baselines}
\label{app:full-hybrid}

Section~\ref{sub:exp-main} reports STUN+Wanda as the published hybrid
baseline. The matched-control hybrids REAP+Wanda and REAM+Wanda are
not reported in their source papers but we run them here as
fair-as-possible matched controls. The Stage~$2$ residual-sparsity
protocol of App.~\ref{app:wanda-stage} is applied identically across
all four hybrid methods.
Table~\ref{tab:hybrid-full} reports DS-Avg and PPL for each
$($model, rate$)$ cell.

\begin{table}[!htbp]
  \centering
  \footnotesize
  \setlength{\tabcolsep}{4pt}
  \renewcommand{\arraystretch}{0.92}
  \caption{Full hybrid axis on all three backbones, including the matched-control hybrids REAP+Wanda and REAM+Wanda (computed by us; not reported in the source REAP / REAM papers). DS-Avg is the unweighted macro of the nine downstream tasks.}
  \label{tab:hybrid-full}
  \begin{tabular}{l|cc|c}
    \toprule
    Method & Wiki & C4 & DS-Avg \\
    \midrule
    \multicolumn{4}{l}{\emph{OLMoE-1B-7B \textemdash\ $33\%$ rate}} \\
    STUN+Wanda \citep{xu2025stun,sun2024wanda} & \textbf{11.85} & 18.43 & 50.8 \\
    REAP+Wanda \citep{muzio2025reap,sun2024wanda} & 14.18 & 16.98 & 52.9 \\
    REAM+Wanda \citep{ream2024,sun2024wanda} & 14.32 & 16.87 & 47.0 \\
    \textbf{HodgeCover+Wanda (ours)} & 13.92 & \textbf{16.69} & \textbf{53.0} \\
    \midrule
    \multicolumn{4}{l}{\emph{OLMoE-1B-7B \textemdash\ $66\%$ rate}} \\
    STUN+Wanda \citep{xu2025stun,sun2024wanda} & 22.11 & 31.01 & 45.1 \\
    REAP+Wanda \citep{muzio2025reap,sun2024wanda} & 22.83 & 24.91 & 43.3 \\
    REAM+Wanda \citep{ream2024,sun2024wanda} & 31.65 & 26.43 & 40.8 \\
    \textbf{HodgeCover+Wanda (ours)} & \textbf{18.32} & \textbf{21.18} & \textbf{48.6} \\
    \midrule
    \multicolumn{4}{l}{\emph{Qwen 3.5-35B-A3B \textemdash\ $33\%$ rate}} \\
    STUN+Wanda \citep{xu2025stun,sun2024wanda} & \textbf{8.86} & 15.61 & 66.0 \\
    REAP+Wanda \citep{muzio2025reap,sun2024wanda} & 9.01 & 13.63 & \textbf{76.3} \\
    REAM+Wanda \citep{ream2024,sun2024wanda} & 9.01 & 13.59 & 75.8 \\
    \textbf{HodgeCover+Wanda (ours)} & 9.18 & \textbf{13.42} & 76.1 \\
    \midrule
    \multicolumn{4}{l}{\emph{Qwen 3.5-35B-A3B \textemdash\ $66\%$ rate}} \\
    STUN+Wanda \citep{xu2025stun,sun2024wanda} & 11.77 & 20.05 & 62.0 \\
    REAP+Wanda \citep{muzio2025reap,sun2024wanda} & 10.57 & 15.78 & 74.1 \\
    REAM+Wanda \citep{ream2024,sun2024wanda} & 10.48 & 15.68 & 73.8 \\
    \textbf{HodgeCover+Wanda (ours)} & \textbf{10.25} & \textbf{15.22} & \textbf{74.6} \\
    \midrule
    \multicolumn{4}{l}{\emph{Qwen 3.5-122B-A10B \textemdash\ $33\%$ rate}} \\
    STUN+Wanda \citep{xu2025stun,sun2024wanda} & 6.35 & 14.48 & 75.5 \\
    REAP+Wanda \citep{muzio2025reap,sun2024wanda} & 6.13 & 12.87 & 77.1 \\
    REAM+Wanda \citep{ream2024,sun2024wanda} & \textbf{5.98} & 12.70 & \textbf{77.9} \\
    \textbf{HodgeCover+Wanda (ours)} & 5.99 & \textbf{12.49} & 77.8 \\
    \midrule
    \multicolumn{4}{l}{\emph{Qwen 3.5-122B-A10B \textemdash\ $66\%$ rate}} \\
    STUN+Wanda \citep{xu2025stun,sun2024wanda} & 8.77 & 17.50 & 70.8 \\
    REAP+Wanda \citep{muzio2025reap,sun2024wanda} & 7.98 & 14.52 & 75.4 \\
    REAM+Wanda \citep{ream2024,sun2024wanda} & 7.78 & 14.25 & 75.6 \\
    \textbf{HodgeCover+Wanda (ours)} & \textbf{7.42} & \textbf{13.86} & \textbf{75.9} \\
    \bottomrule
  \end{tabular}
\end{table}

The pattern across the four hybrids is consistent. On the two Qwen
scales at $33\%$ rate the matched-control hybrids cluster within $0.8$ pp DS-Avg
of each other; HodgeCover+W and REAP+W trade the lead per task but
neither dominates. At $66\%$ HodgeCover+W opens a $0.3$--$0.8$ pp
DS-Avg gap over REAP+W and REAM+W and a $5$--$13$ pp gap over STUN+W
on the two Qwen scales. On OLMoE-1B-7B at $66\%$ the four hybrids
span an $8$ pp DS-Avg range (HodgeCover+W $48.6\%$, STUN+W $45.1\%$,
REAP+W $43.3\%$, REAM+W $40.8\%$); the small-model OLMoE regime is
the noisiest of the three. We read the four-hybrid comparison as a
robustness check: HodgeCover+Wanda's lead on the two Qwen scales at
$66\%$ is not an artifact of comparing against the weakest hybrid in
the literature, since the matched-control REAP+Wanda and REAM+Wanda
remain $0.3$--$0.8$ pp DS-Avg behind on the same cells.

\begin{figure*}[!htbp]
  \centering
  \includegraphics[width=\linewidth]{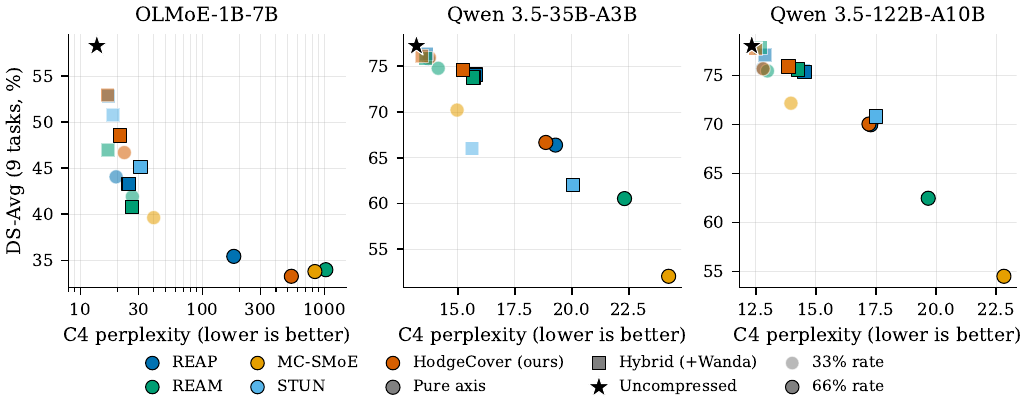}
  \caption{C4 perplexity vs.\ DS-Avg at $33\%$ (lighter, white-edged)
  and $66\%$ (darker, black-edged) rates. Pure expert-count axis =
  circles, hybrid axis = squares. Uncompressed reference is the black
  star. The OLMoE panel uses a logarithmic $x$ axis. HodgeCover+Wanda
  sits on the upper-left frontier of every panel.}
  \label{fig:pareto}
\end{figure*}

\subsection{Mechanism: H/G/C/T retained mass on OLMoE and Qwen 3.5-35B-A3B}
\label{app:mechanism}

\label{app:mechanism-metric}%
\noindent\textbf{Retained-mass metric.}
Fix MoE layer $\ell$ with simplicial mergeability complex
$K^{(\ell)} = (V^{(\ell)}, E^{(\ell)}, T^{(\ell)})$ and edge-supported
barrier signal $b^{(\ell)} \in C_1(K^{(\ell)})$ on the
\emph{uncompressed} layer (Section~\ref{sec:complex}). The Hodge
decomposition (Theorem~\ref{thm:hodge}) splits $b^{(\ell)}$ into
gradient, curl, and harmonic components
$b^{(\ell)} = b^{(\ell)}_{\mathrm{grad}}
            + b^{(\ell)}_{\mathrm{curl}}
            + b^{(\ell)}_{\mathrm{harm}}$;
write $h_e^{(\ell)},\, g_e^{(\ell)},\, c_e^{(\ell)}$ for the per-edge
coefficients of the harmonic, gradient, and curl components, and
$b_\sigma^{(\ell)}$ for the triplet-barrier value on triangle
$\sigma \in T^{(\ell)}$. Given a survivor set
$S^{(\ell)} \subseteq V^{(\ell)}$ produced by any compression method,
the per-layer retained-mass metrics are the fractions of $\ell^1$
mass on edges and triangles whose support \emph{intersects} the
survivor set:
\begin{equation}
  \label{eq:retained-mass}
  \mathrm{ret}_X^{(\ell)}(S)
    \;=\;
    \frac{\sum_{e \in E^{(\ell)} :\, e \cap S \neq \emptyset}\, |X^{(\ell)}_e|}
         {\sum_{e \in E^{(\ell)}}\, |X^{(\ell)}_e|},
  \qquad X \in \{\mathrm{harm},\, \mathrm{grad},\, \mathrm{curl}\},
\end{equation}
and analogously for the triplet-barrier component, replacing the
edge sum by a sum over triangles
$\sigma \in T^{(\ell)}$ with $\sigma \cap S \neq \emptyset$, weighted
by $|b_\sigma^{(\ell)}|$. Survivors are evaluated against the
\emph{original} (pre-compression) Hodge decomposition rather than
re-decomposing the survivor sub-MoE's complex, so the four metrics
directly measure how much of the original layer's structural signal
each method's survivor set covers. The macro retained-mass at a
given $($model, rate$)$ cell is the unweighted average of
$\mathrm{ret}_X^{(\ell)}(S^{(\ell)})$ over MoE layers.

Section~\ref{sub:exp-mechanism} plots the four-component retained
mass on Qwen 3.5-122B-A10B at $66\%$.
Figures~\ref{fig:hgct-olmoe} and \ref{fig:hgct-qwen35} extend the
same plot to the two remaining backbones.

\begin{figure*}[!htbp]
  \centering
  \includegraphics[width=\linewidth]{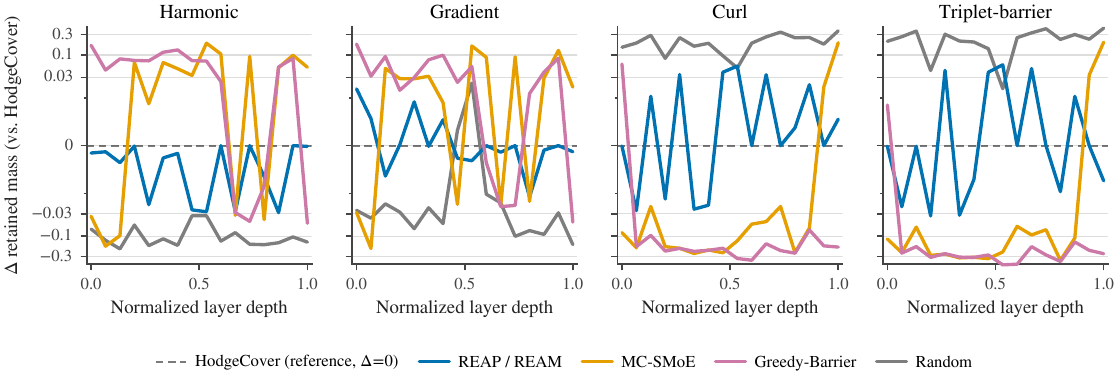}
  \caption{Per-layer deviation from HodgeCover on OLMoE-1B-7B at
  $66\%$ across the four Hodge components.}
  \label{fig:hgct-olmoe}
\end{figure*}

\begin{figure*}[!htbp]
  \centering
  \includegraphics[width=\linewidth]{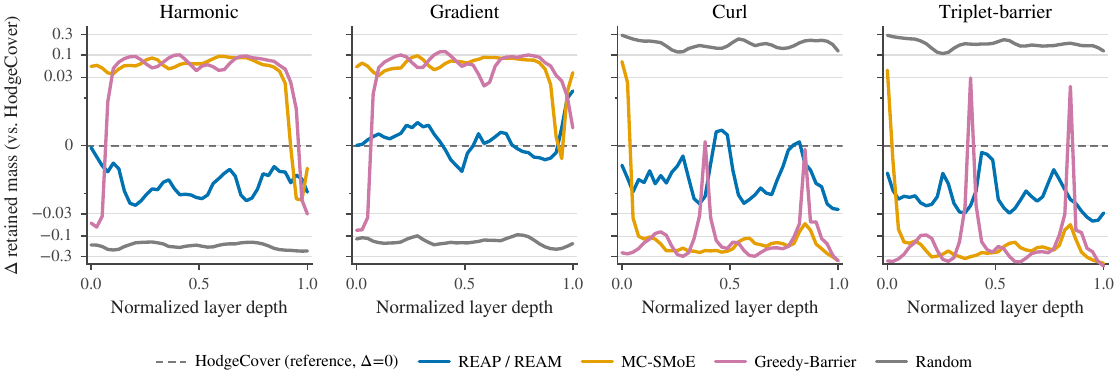}
  \caption{Per-layer deviation from HodgeCover on Qwen 3.5-35B-A3B at
  $66\%$ across the four Hodge components.}
  \label{fig:hgct-qwen35}
\end{figure*}

\begin{table}[!htbp]
  \centering
  \footnotesize
  \setlength{\tabcolsep}{3pt}
  \caption{Macro retained mass over MoE layers at $66\%$ rate for the four components of the Hodge decomposition: harmonic (H), gradient (G), curl (C), and raw triplet-barrier (T). HodgeCover, REAP, and REAM are positioned across all four; Greedy-Barrier, MC-SMoE, and Random each maximize a strict subset.}
  \label{tab:mass-macro}
  \resizebox{\linewidth}{!}{%
  \begin{tabular}{l|cccc|cccc|cccc}
    \toprule
    & \multicolumn{4}{c|}{OLMoE-1B-7B} & \multicolumn{4}{c|}{Qwen 3.5-35B-A3B} & \multicolumn{4}{c}{Qwen 3.5-122B-A10B} \\
    Method & H & G & C & T & H & G & C & T & H & G & C & T \\
    \midrule
    REAP & 0.6838 & 0.6133 & 0.3859 & 0.4972 & 0.7310 & 0.7009 & 0.3741 & 0.5149 & 0.6831 & 0.6729 & 0.4323 & 0.5716 \\
    REAM & 0.6838 & 0.6133 & 0.3859 & 0.4972 & 0.7310 & 0.7009 & 0.3741 & 0.5149 & 0.6831 & 0.6729 & 0.4323 & 0.5716 \\
    Greedy-Barrier & \textbf{0.7469} & \textbf{0.6535} & 0.1917 & 0.2269 & 0.7919 & 0.7572 & 0.2183 & 0.3191 & 0.7552 & \textbf{0.7353} & 0.2268 & 0.3278 \\
    MC-SMoE & 0.7202 & 0.6372 & 0.2717 & 0.3273 & \textbf{0.7960} & \textbf{0.7607} & 0.2209 & 0.3232 & \textbf{0.7641} & 0.7347 & 0.2397 & 0.3363 \\
    Random & 0.5761 & 0.5690 & \textbf{0.5858} & \textbf{0.7382} & 0.5683 & 0.5642 & \textbf{0.5649} & \textbf{0.7138} & 0.5692 & 0.5654 & \textbf{0.5650} & \textbf{0.7143} \\
    \textbf{HodgeCover (ours)} & 0.6927 & 0.6128 & 0.3781 & 0.4914 & 0.7395 & 0.6998 & 0.3819 & 0.5311 & 0.6910 & 0.6734 & 0.4363 & 0.5774 \\
    \bottomrule
  \end{tabular}
  }% end resizebox
\end{table}

The $H$/$G$/$C$/$T$ trade-off observed in
Section~\ref{sub:exp-mechanism} is consistent across the two
remaining backbones. Greedy-Barrier and MC-SMoE retain more harmonic
and gradient mass than HodgeCover at the cost of curl and
triplet-barrier mass; Random sweeps in the opposite direction; REAP
and REAM share identical pure-axis survivor selection (their plots
are drawn as a single curve in Figure~\ref{fig:mechanism}) and
deviate from HodgeCover weakly on every component.
Table~\ref{tab:mass-macro}
reports the macro retained mass at $66\%$ for the same five methods
across the four components.

\subsection{Full ablation breakdown}
\label{app:ablations-full}

The wrapped Table~\ref{tab:ablations} in the main body reports the
Qwen 3.5-35B-A3B ablations. Table~\ref{tab:ablations-full} below
adds the per-task breakdown for all three backbones at both rates.

\begin{table*}[!htbp]
  \centering
  \footnotesize
  \setlength{\tabcolsep}{3.5pt}
  \caption{Per-task ablation breakdown across all three backbones at $33\%$ and $66\%$ rate, with the rightmost column reporting the unweighted nine-task DS-Avg. Bold = best within (model, rate) over the five rows. Triplet-Penalty (soft) on Qwen 3.5-122B-A10B was not run pre-submission. The rows match Table~\ref{tab:ablations} on Qwen 3.5-35B-A3B in the main body.}
  \label{tab:ablations-full}
  \resizebox{\linewidth}{!}{%
  \begin{tabular}{l|cc|ccccccccc|c}
    \toprule
    Variant & Wiki & C4 & ARC-c & ARC-e & BoolQ & HellaS & MMLU & PIQA & TQA & WinoG & GSM8K & DS-Avg \\
    \midrule
    \multicolumn{13}{l}{\emph{OLMoE-1B-7B \textemdash\ $\mathbf{33\%}$ rate}} \\
    HodgeCover (ours) & 23.50 & 22.91 & \textbf{32.8} & 50.3 & 62.8 & \textbf{67.0} & 25.6 & 73.0 & \textbf{41.7} & \textbf{64.9} & 2.2 & 46.7 \\
    Hodge No-Triangle & 38.37 & 37.36 & 27.6 & 49.9 & 62.0 & 47.2 & 26.6 & 66.5 & 36.1 & 55.5 & 2.3 & 41.5 \\
    Triplet-Hypergraph & 42.90 & 41.42 & 26.6 & 40.2 & 61.9 & 46.8 & 25.1 & 69.0 & 40.9 & 53.8 & 2.3 & 40.7 \\
    Triplet-Penalty & \textbf{21.03} & \textbf{20.17} & 32.5 & \textbf{50.5} & \textbf{64.4} & 65.2 & \textbf{33.8} & \textbf{73.8} & 38.6 & 63.2 & \textbf{3.3} & \textbf{47.3} \\
    Greedy-Barrier & 21.37 & 20.42 & 31.6 & 49.9 & 62.5 & 65.7 & 32.4 & 73.8 & 37.7 & 61.4 & 3.0 & 46.4 \\
    \midrule
    \multicolumn{13}{l}{\emph{OLMoE-1B-7B \textemdash\ $\mathbf{66\%}$ rate}} \\
    HodgeCover (ours) & 1,134.9 & 535.5 & 24.2 & 29.5 & 38.6 & 30.0 & 25.4 & 52.7 & 47.7 & \textbf{51.2} & \textbf{0.1} & 33.3 \\
    Hodge No-Triangle & 20,714.5 & 4,596.3 & 24.5 & 30.4 & 40.0 & 27.8 & 25.2 & 51.7 & 49.9 & 50.3 & 0.0 & 33.3 \\
    Triplet-Hypergraph & 19,030.7 & 14,481.7 & 24.7 & 27.8 & 39.9 & 26.8 & 25.5 & 51.9 & \textbf{50.6} & 49.8 & 0.0 & 33.0 \\
    Triplet-Penalty & \textbf{427.3} & \textbf{252.7} & \textbf{25.2} & 30.6 & \textbf{52.6} & \textbf{34.2} & \textbf{26.6} & 56.1 & 50.0 & 50.7 & 0.0 & \textbf{36.2} \\
    Greedy-Barrier & 435.2 & 291.8 & 24.6 & \textbf{31.5} & 48.8 & 32.4 & 24.6 & \textbf{56.4} & 49.9 & 49.6 & 0.1 & 35.3 \\
    \midrule
    \multicolumn{13}{l}{\emph{Qwen 3.5-35B-A3B \textemdash\ $\mathbf{33\%}$ rate}} \\
    HodgeCover (ours) & \textbf{9.97} & \textbf{13.75} & \textbf{56.3} & \textbf{73.9} & \textbf{91.3} & \textbf{82.5} & \textbf{80.7} & \textbf{82.9} & \textbf{54.7} & \textbf{76.3} & 85.1 & \textbf{75.9} \\
    Hodge No-Triangle & 12.48 & 17.10 & 46.4 & 67.7 & 87.2 & 73.3 & 66.4 & 77.9 & 46.8 & 70.2 & 73.7 & 67.7 \\
    Triplet-Hypergraph & 14.29 & 16.76 & 36.3 & 57.7 & 88.1 & 72.4 & 53.1 & 78.2 & 48.8 & 73.2 & 72.5 & 64.5 \\
    Triplet-Penalty & 10.22 & 13.82 & 47.2 & 66.3 & 87.0 & 80.6 & 74.3 & 82.0 & 50.5 & 74.8 & 88.0 & 72.3 \\
    Greedy-Barrier & 10.22 & 13.79 & 46.6 & 66.3 & 89.8 & 80.7 & 74.1 & 81.8 & 50.9 & 75.7 & \textbf{88.2} & 72.7 \\
    \midrule
    \multicolumn{13}{l}{\emph{Qwen 3.5-35B-A3B \textemdash\ $\mathbf{66\%}$ rate}} \\
    HodgeCover (ours) & \textbf{15.13} & 18.86 & \textbf{46.4} & \textbf{67.5} & \textbf{88.1} & \textbf{72.5} & \textbf{54.5} & \textbf{78.1} & \textbf{51.7} & 73.7 & \textbf{67.5} & \textbf{66.7} \\
    Hodge No-Triangle & 21.34 & 28.34 & 35.8 & 54.7 & 78.5 & 57.4 & 44.1 & 70.3 & 45.4 & 67.6 & 42.8 & 55.2 \\
    Triplet-Hypergraph & 98.08 & 154.9 & 20.8 & 32.7 & 62.2 & 30.7 & 26.0 & 56.3 & 44.4 & 50.9 & 1.7 & 36.2 \\
    Triplet-Penalty & 16.14 & 18.77 & 36.0 & 53.7 & 86.6 & 69.6 & 44.2 & 76.1 & 47.1 & \textbf{74.0} & 61.2 & 60.9 \\
    Greedy-Barrier & 15.30 & \textbf{18.13} & 33.4 & 52.6 & 86.8 & 69.9 & 43.9 & 77.3 & 47.6 & 72.5 & 62.3 & 60.7 \\
    \midrule
    \multicolumn{13}{l}{\emph{Qwen 3.5-122B-A10B \textemdash\ $\mathbf{33\%}$ rate}} \\
    HodgeCover (ours) & \textbf{7.21} & \textbf{12.81} & \textbf{61.7} & \textbf{80.1} & 74.4 & \textbf{84.9} & \textbf{83.9} & \textbf{83.1} & 51.6 & 77.3 & \textbf{84.5} & \textbf{75.7} \\
    Hodge No-Triangle & 9.51 & 15.34 & 56.3 & 78.9 & 72.5 & 78.1 & 76.2 & 80.5 & 50.2 & 75.6 & 82.9 & 72.4 \\
    Triplet-Hypergraph & 11.54 & 15.37 & 45.1 & 67.7 & \textbf{82.5} & 76.9 & 60.1 & 80.2 & 49.1 & 75.1 & 69.5 & 67.4 \\
    Greedy-Barrier & 7.36 & 12.88 & 58.6 & 78.2 & 73.6 & 84.7 & 79.3 & 82.6 & \textbf{52.3} & \textbf{78.5} & 72.1 & 73.3 \\
    \midrule
    \multicolumn{13}{l}{\emph{Qwen 3.5-122B-A10B \textemdash\ $\mathbf{66\%}$ rate}} \\
    HodgeCover (ours) & \textbf{12.46} & 17.21 & \textbf{52.2} & \textbf{71.1} & 75.3 & \textbf{77.0} & \textbf{69.0} & \textbf{80.1} & \textbf{51.5} & \textbf{76.4} & \textbf{78.1} & \textbf{70.1} \\
    Hodge No-Triangle & 16.79 & 22.95 & 39.9 & 61.4 & \textbf{85.7} & 66.0 & 48.9 & 75.0 & 43.0 & 70.6 & 37.3 & 58.7 \\
    Triplet-Hypergraph & 59.87 & 94.01 & 23.6 & 39.2 & 62.2 & 33.1 & 25.5 & 58.1 & 43.4 & 53.0 & 2.7 & 37.9 \\
    Greedy-Barrier & 12.76 & \textbf{16.29} & 38.6 & 60.8 & 79.0 & 75.0 & 42.2 & 78.5 & 47.5 & 75.4 & 53.3 & 61.2 \\
    \bottomrule
  \end{tabular}
  }% end resizebox
\end{table*}

The cross-backbone pattern matches the Qwen 3.5-35B-A3B reading.
On OLMoE-1B-7B at $33\%$ HodgeCover and Triplet-Penalty are within
$1$ pp DS-Avg, and at $66\%$ the small-model regime collapses every
expert-count row toward the $33$--$36\%$ DS-Avg band; the OLMoE row
is a stress regime in which the topological objective is harder to
distinguish from the matched soft-penalty baseline because every
expert-count method is far from the uncompressed reference.
On Qwen 3.5-122B-A10B at $66\%$ HodgeCover gains
$+11.4$ pp DS-Avg over Hodge No-Triangle, $+8.9$ pp over
Greedy-Barrier, and $+32.2$ pp over Triplet-Hypergraph; we did not
run Triplet-Penalty on Qwen 3.5-122B-A10B because the
Qwen 3.5-35B-A3B run already cleared the ablation gating threshold.
Across the two Qwen scales the non-Hodge ablations cluster: Hodge
No-Triangle is $3$--$12$ pp DS-Avg below HodgeCover and
Triplet-Hypergraph collapses by $8$--$32$ pp; on Qwen 3.5-35B-A3B,
Triplet-Penalty lands within $1$ pp of Greedy-Barrier but $4$--$6$ pp
below HodgeCover. The conclusion repeats: the Hodge decomposition is
load-bearing, the triangle term is necessary for it to deliver, and
the soft form of the topological coverage is necessary for the
triangle term not to over-veto.

\subsection{Systems characterization}
\label{app:systems}

Table~\ref{tab:systems} reports the plan-time at the first
compression rate run on the model, inference throughput at $66\%$
rate, routing entropy at $66\%$ rate, and dead-expert ratio at $66\%$
rate on the post-compressed layer for the four learning-free hybrids
of Section~\ref{sub:exp-main} (and, for reference, HodgeCover without
Stage~$2$ Wanda).

\begin{table}[!htbp]
  \centering
  \small
  \setlength{\tabcolsep}{4pt}
  \caption{Systems characterization. Plan-time is reported at the 33\% rate (the first compression rate run on the model, which builds the simplicial complex from scratch); subsequent rates re-use the cached complex and run in single-digit seconds, so we omit the 66\% plan-time column. Inference throughput, normalized routing entropy, and dead-expert ratio are measured at 66\% on the post-compressed layer. HodgeCover-only (without Stage~2 Wanda) is included for reference.}
  \label{tab:systems}
  \resizebox{\linewidth}{!}{%
  \begin{tabular}{l|c|c|c|c}
    \toprule
    Method & Plan-time @ 33\% (s, first run) & Throughput @ 66\% (tok/s) & Routing entropy @ 66\% & Dead-expert @ 66\% \\
    \midrule
    \multicolumn{5}{l}{\emph{OLMoE-1B-7B}} \\
    REAP+W & 4.0 & 6,137 & \textbf{0.960} & \textbf{3.1\%} \\
    REAM+W & 14.0 & 6,387 & 0.958 & 3.1\% \\
    HodgeCover+W (ours) & \textbf{3.3} & 6,506 & 0.933 & 17.9\% \\
    HodgeCover (ours) & 33.3 & \textbf{6,615} & 0.815 & 38.7\% \\
    \midrule
    \multicolumn{5}{l}{\emph{Qwen 3.5-35B-A3B}} \\
    REAP+W & \textbf{6.3} & 5,384 & 0.882 & 36.7\% \\
    REAM+W & 24.4 & \textbf{5,571} & \textbf{0.888} & \textbf{35.1\%} \\
    HodgeCover+W (ours) & 476.6 & 5,348 & 0.881 & 37.3\% \\
    HodgeCover (ours) & 482.6 & 5,318 & 0.833 & 49.8\% \\
    \midrule
    \multicolumn{5}{l}{\emph{Qwen 3.5-122B-A10B}} \\
    REAP+W & \textbf{15.6} & 3,616 & 0.895 & 29.6\% \\
    REAM+W & 190.0 & \textbf{4,119} & \textbf{0.900} & \textbf{27.5\%} \\
    HodgeCover+W (ours) & 488.1 & 3,801 & 0.897 & 29.0\% \\
    HodgeCover (ours) & 479.3 & 3,817 & 0.861 & 40.3\% \\
    \bottomrule
  \end{tabular}
  }% end resizebox
\end{table}

\textbf{Why the table reports plan-time at one rate, not two.}
HodgeCover's plan-time has a one-time component (the simplicial
complex, the pairwise and triplet barrier sweep, and the Hodge
decomposition itself) that is computed once per backbone and cached
across all subsequent compression rates, plus a small per-rate
component (the survivor-selection greedy loop and the redirect step,
together within single-digit seconds). Reporting plan-time at the
\emph{second} compression rate run on a backbone would therefore
report the cached, single-digit number rather than the honest
first-rate cost. We instead report plan-time at $33\%$, which on
every backbone is the first rate run and the only rate at which the
simplicial complex is built from scratch; the
$O((n^2+|T|)|\mathcal{D}|d^2)$ asymptotics that drive this cost are
worked through in App.~\ref{app:complexity}. The same caching applies
to the Stage~$2$ Wanda pruner: at the two Qwen scales the two
HodgeCover rows in Table~\ref{tab:systems} agree within $\sim 10$ s
($482.6$ vs.\ $476.6$ s on Qwen 3.5-35B-A3B; $479.3$ vs.\ $488.1$ s
on Qwen 3.5-122B-A10B), because Stage~$2$ reuses the calibration
tensor produced during Stage~$1$ with no additional forward pass; on
the much smaller OLMoE-1B-7B the absolute plan-times are an order of
magnitude smaller and the row-to-row spread ($33.3$ vs.\ $3.3$ s) is
dominated by one-time overheads in the from-scratch HodgeCover-only
run rather than by Stage~$2$ itself.
The plan-time gap to REAP+Wanda, by contrast, is real and structural:
REAP's per-layer scoring is $O(|\mathcal{D}| d^2)$, dominated by the
$n^2 |\mathcal{D}| d^2$ pairwise sweep at the $n = 256$-expert
Qwen scales by a factor of $n^2$.

\textbf{Throughput, routing entropy, dead-expert.} HodgeCover+Wanda
matches REAP+Wanda's tokens-per-second to within $6\%$ on every
backbone (the largest gap is $370$ tok/s, $6.0\%$ relative, on
OLMoE-1B-7B); the routing entropy at $66\%$ matches REAP+W and REAM+W
to within $0.01$ on the two Qwen scales. The dead-expert ratio at
$66\%$ favors REAM+Wanda by $1$--$3$ pp absolute on the two Qwen
scales (Qwen 3.5-35B-A3B: REAM+W $35.1\%$ vs.\ HodgeCover+W $37.3\%$;
Qwen 3.5-122B-A10B: REAM+W $27.5\%$ vs.\ HodgeCover+W $29.0\%$). We
surface this as a real disadvantage of HodgeCover+W in the cost
picture. The dead-expert ratio is not, however, a
downstream-correlated metric in our data: the same REAM+W cells lose
$0.3$--$0.8$ pp DS-Avg to HodgeCover+W on the same backbones
(Table~\ref{tab:hybrid-full}).

%%%%%%%%%%%%%%%%%%%%%%%%%%%%%%%%%%%%%%%%%%%%%%%%%%%%%%%%%%%%

% \newpage
% \input{checklist.tex}

\end{document}